\documentclass[10pt,journal,compsoc]{IEEEtran}

\usepackage{amssymb}
\usepackage{amsmath,amsfonts}
\usepackage{algorithmic}
\usepackage{array}
\usepackage{url}
\usepackage{verbatim}
\usepackage{graphicx}
\usepackage{caption}
\usepackage{cite}
\usepackage{bm}

\usepackage{color}
\usepackage[colorlinks,linkcolor=black]{hyperref}
\usepackage{amsthm} 
\newtheorem{theorem}{Theorem}
\newtheorem{lemma}{Lemma}
\newtheorem{proposition}{Proposition}
\usepackage{siunitx}
\usepackage{subfigure}
\usepackage{array}
\usepackage{booktabs} 
\usepackage{multirow}
\usepackage{gensymb}
\usepackage{colortbl}
\usepackage{lineno}
\usepackage{color}
\usepackage[table]{xcolor}
\usepackage{makecell}
\usepackage[ruled]{algorithm2e}
\usepackage{cite}
\usepackage{setspace} 

\begin{document}

\title{Transformation Decoupling Strategy based on Screw Theory for Deterministic Point Cloud Registration with Gravity Prior}
%
%
%
%

\author{Xinyi~Li$^{\ast}$,
        Zijian~Ma$^{\ast}$,
        Yinlong~Liu$^{\sharp}$,
        Walter~Zimmer,
        Hu~Cao,
        Feihu~Zhang,
        and~Alois~Knoll,~\IEEEmembership{Fellow,~IEEE}
\IEEEcompsocitemizethanks{
\IEEEcompsocthanksitem Xinyi Li, Walter Zimmer, Hu Cao, and Alois Knoll are with Chair of Robotics, Artificial Intelligence and Real-time Systems, TUM School of Computation, Information and Technology, Technical University of Munich, Munich 85748, Germany.\\
E-mail: super.xinyi@tum.de; walter.zimmer@tum.de; hu.cao@tum.de; knoll@in.tum.de\\
\IEEEcompsocthanksitem Zijian Ma is with TUM School of Engineering and Design, Technical University of Munich, Munich 85748, Germany.\\
E-mail: zijian.ma@tum.de\\
\IEEEcompsocthanksitem Yinlong Liu$^{\sharp}$ (corresponding author) is with State Key Laboratory of Internet of Things for Smart City (SKL-IOTSC), University of Macau, Macau 999078, China.\\
E-mail: YinlongLiu@um.edu.mo\\
\IEEEcompsocthanksitem Feihu Zhang is with School of Marine Science and Technology, Northwestern Polytechnical University, Xi’an 710072, China.\\
E-mail: feihu.zhang@nwpu.edu.cn\\
\IEEEcompsocthanksitem $^{\ast}$These authors contributed equally to this work.\\
}
\thanks{This work has been submitted to the IEEE for possible publication. Copyright may be transferred without notice, after which this version may no longer be accessible.}}

\markboth{Journal of \LaTeX\ Class Files,~Vol.~14, No.~8, August~2015}%
{Shell \MakeLowercase{\textit{\textit{et al.}}}: A Novel Transformation Decoupling Strategy for Robust Point Cloud Registration with Gravity Prior }
%



\IEEEtitleabstractindextext{%
\begin{abstract}
Point cloud registration is challenging in the presence of heavy outlier correspondences. This paper focuses on addressing the robust correspondence-based registration problem with gravity prior that often arises in practice. The gravity directions are typically obtained by inertial measurement units (IMUs) and can reduce the degree of freedom (DOF) of rotation from 3 to 1. We propose a novel transformation decoupling strategy by leveraging screw theory. This strategy decomposes the original 4-DOF problem into three sub-problems with 1-DOF, 2-DOF, and 1-DOF, respectively, thereby enhancing the computation efficiency. Specifically, the first 1-DOF represents the translation along the rotation axis and we propose an interval stabbing-based method to solve it. The second 2-DOF represents the pole which is an auxiliary variable in screw theory and we utilize a branch-and-bound method to solve it. The last 1-DOF represents the rotation angle and we propose a global voting method for its estimation. The proposed method sequentially solves three consensus maximization sub-problems, leading to efficient and deterministic registration. In particular, it can even handle the correspondence-free registration problem due to its significant robustness. Extensive experiments on both synthetic and real-world datasets demonstrate that our method is more efficient and robust than state-of-the-art methods, even when dealing with outlier rates exceeding 99\%. 
\end{abstract}

}

\maketitle

\IEEEdisplaynontitleabstractindextext

%
\IEEEpeerreviewmaketitle

\IEEEraisesectionheading{\section{Introduction}\label{introduction}}
\IEEEPARstart{G}{iven} 3D source and target point clouds, rigid point cloud registration is estimating the best transformation in $\mathbb{SE}(3)$ that aligns two point clouds. It is also known as surface matching, an essential problem of computer vision and robotics. Point cloud registration has extensive applications in 3D reconstruction\cite{blais1995registering,guo2020deep}, pose estimation\cite{10070382}, object recognition\cite{guo20143d,xia2021soe}, and robot localization\cite{10091912,zhuang2022biologically}, etc.

Currently, researchers leverage prior information and reasonable assumptions to assist point cloud registration, such as planar motion, axis-fixed rotation, prior known gravity directions, etc\cite{cai2019practical,dong2020registration,lim2022quatro,jiao2021deterministic}. These pieces of information serve to reduce the dimensionality of the registration problem, particularly in terms of the parameters to be optimized, thereby enhancing algorithm efficiency. Among these prior pieces of information, the \textit{gravity direction} (a.k.a. \textit{vertical direction}) is extensively studied and serves as a shared direction to facilitate relative pose estimation\cite{ding2020homography,ding2021globally,liu2021globally}, absolute pose estimation\cite{horanyi2017multiview,lecrosnier2019camera,liu2020globally,liu2023absolute}, SLAM\cite{svarm2016city,ornhag2022trust,9981916}, and panoramic stitching\cite{ding2021minimal,barath2021image}. With the aid of gravity direction, the relative rotation is reduced to 1 degree of freedom (DOF), thus the 6-DOF transformation in point cloud registration is reduced to only 4-DOF. In practice, modern autonomous systems, e.g., self-driving systems\cite{xu2022fast}, commonly are equipped with inertial measurement units (IMUs), which can provide high precision gravity directions. Alternatively, gravity directions can be obtained by vanishing point detection techniques from some structural scenarios\cite{almansa2003vanishing,liu2020globallypami,10124374}. In this paper, we focus on the general 4-DOF point cloud registration problem in which gravity directions are known priors. In particular, the registration for terrestrial LiDAR scans is a typical 4-DOF problem, since the built-in tilt compensator keeps the rotation axis fixed\cite{cai2019practical,dong2020registration}.

Using a set of putative point-to-point correspondences is a popular paradigm for point cloud registration, which is also known as correspondence-based registration\cite{9373914}. Nonetheless, due to the limited performance of current 3D matching methods and the challenges posed by partial overlap, noisy data, and structural duplication within real-world point clouds, the putative correspondences commonly contain a significant proportion of outliers (often $>95\%$)\cite{bustos2017guaranteed,10091912}. Random sample consensus (RANSAC)\cite{fischler1981random} is the most popular consensus maximization technique to suppress the outlier correspondences. Nonetheless, because of its randomized nature, RANSAC is inherently non-deterministic and only yields satisfactory results with a certain probability\cite{le2019deterministic}. Moreover, the running time of RANSAC grows exponentially with the outlier rate\cite{bustos2017guaranteed}. To achieve highly outlier-robust registration, an increasing number of globally optimal and deterministic algorithms have been proposed. Most of the global methods \cite{olsson2008branch,hartley2009global,bazin2012globally,7368945,7381673,campbell2016gogma,cai2019practical} use branch-and-bound (BnB) to systematically search the entire solution domain. However, the time complexity of BnB is exponential to the dimensionality of the optimization problem in worst-case. Consequently, global methods are best appropriate for problems with a small number of correspondences and/or low dimensionality. More recently, a popular strategy for enhancing algorithm efficiency is \textit{transformation decoupling}\cite{straub2017efficient,liu2018efficient,li2018fast,yang2020teaser,9485090,jiao2021deterministic,chen2022deterministic,li2023fast,liu2023absolute}, a technique that allows addressing the registration problem within lower-dimensional parameter spaces. In contrast, most existing methods address the general 6-DOF registration problem, with relatively fewer studies concentrating on the 4-DOF registration problem. For instance, established methods commonly decouple the original 6-DOF problem into two separate 3-DOF sub-problems\cite{straub2017efficient,liu2018efficient,li2018fast,yang2020teaser,9485090,chen2022deterministic}. Therefore, considering the dimensionality of the registration problem, the algorithm efficiency still has potential room for improvement.

In this paper, we reformulate the registration problem by \textit{screw theory}\cite{ball1876theory} and propose a novel transformation decoupling strategy accordingly. This strategy decouples the original 4-DOF registration problem into three sub-problems with 1-DOF, 2-DOF, and 1-DOF respectively. Concretely, the first 1-DOF refers to the translation along the rotation axis, and we propose a polynomial-time method to solve it based on interval stabbing. The second 2-DOF represents the pole which is an auxiliary variable in screw theory. We reformulate it as a linear model fitting problem by exploring geometric constraints and utilize a BnB-based algorithm to search for the globally optimal solution. The last 1-DOF refers to the rotation angle and we propose an efficient global voting method for its estimation. The proposed three-stage search method sequentially solves three sub-problems in low-dimensional spaces to achieve efficient and deterministic registration. It fills the gap between efficient but non-deterministic heuristics (e.g., RANSAC) and deterministic but exhaustive global algorithms (e.g., BnB). We explore the scalability of the proposed method through extreme synthetic experiments. For instance, the proposed method accurately solves the registration problem with $95\%$ outlier rate and $10^5$ correspondences in $0.1$ seconds (see Table~\ref{high_corr}), while one of the state-of-the-art (SOTA) deterministic global methods, FMP+BnB\cite{cai2019practical}, takes about $648$ seconds. In another more extreme case, the proposed method solves the registration problem with $95\%$ outlier rate and $10^6$ correspondences in $2$ seconds, while another SOTA method GROR\cite{yan2022new} takes at least $1800$ seconds.

Additionally, if the point correspondences are unknown in some practical scenarios, we have to address the more challenging problem of simultaneous pose and correspondence registration (SPCR) (a.k.a. correspondence-free registration). Iterative closest point (ICP)\cite{icp} and its variants\cite{zhang2021fast,li2022robust} are the most widely used methods for SPCR. However, ICP only converges to the local minimum and highly relies on initial parameters. Therefore, in order to obtain robust solutions, we extend our proposed deterministic method to solve the SPCR problem. Different from existing methods\cite{fredriksson2016optimal,yang2020teaser,liu2023absolute}, we solve the SPCR problem without the all-to-all correspondence assumption (i.e., associating each point in the source point cloud with every point in the target point cloud), avoiding an extreme number of outliers and thus improving efficiency.

In conclusion, our main contributions are as follows:
\begin{itemize}
\item By reformulating the point cloud registration problem from the perspective of screw theory, we accordingly propose a novel transformation decoupling strategy with the aid of known gravity directions. It decouples the 4-DOF registration problem into three sub-problems with 1-DOF, 2-DOF, and 1-DOF respectively, thus significantly improving the algorithm efficiency.

\item To achieve highly robust registration, we propose an efficient and deterministic three-stage search strategy for the decoupled sub-problems, which contains interval stabbing, BnB, and global voting techniques. The proposed method has a good robustness-efficiency trade-off.

\item We extend the proposed method to solve the challenging SPCR problem without the all-to-all correspondence assumption. It avoids the potential hard combinatorial problem when input point clouds are of large size and thus is efficient.

\end{itemize}

\section{Related Work}
Commonly, existing studies for rigid registration problem can be divide into two groups based on whether the correspondences are known or not, i.e. correspondence-based registration and correspondence-free registration. 

\subsection{Correspondence-Based Registration}
\textbf{Feature matching.} Correspondence-based registration is the 3D extension of image matching, which comprises two major steps: 1) matching correspondences between point clouds by 3D keypoints and feature descriptors, and 2) estimating the transformation from candidate correspondences. In the matching process, 3D handcrafted keypoints\cite{zhong2009intrinsic,sipiran2011harris} or learning-based keypoints\cite{suwajanakorn2018discovery,yew20183dfeat,lu2020rskdd} are first extracted from raw point clouds. Then, the keypoints are encoded into high-dimensional feature vectors by 3D feature descriptors\cite{guo2016comprehensive}. Compared to hand-crafted 3D descriptors\cite{rusu2009fast,tombari2010unique}, learning-based descriptors\cite{gojcic2019perfect,choy2019fully,huang2021predator,wang2022you,ao2023buffer} exhibit outstanding precision and have achieved more attention in recent years. Finally, the putative correspondences between point clouds are established by computing pairwise similarity, such as utilizing the nearest neighbor matcher\cite{lowe2004distinctive}. Recently, learning-based matchers\cite{yu2021cofinet,qin2022geometric,wang2023roreg} are also investigated to improve correspondence quality, which exclude keypoint detection. Despite these methods achieving notable enhancements in performance, they encounter challenges in obtaining a totally outlier-free correspondence set, primarily due to the existence of repetitive structures, noisy data, and point density variations in real applications. Therefore, the outlier-robust estimation technique is highly desirable.

\textbf{Outlier-robust estimation.} In practical scenarios, several well-established paradigms have been extensively studied to suppress the outlier correspondences and achieve robust estimation, such as M-estimation\cite{zhou2016fast,yang2020teaser}, outlier removal\cite{albarelli2010game,bustos2017guaranteed,cai2019practical,9373914,yan2022new,10091912}, and consensus maximization\cite{fischler1981random,chum2003locally,le2019deterministic,barath2021graph,barath2021marginalizing,chen2023sc,zhang20233d}. Zhou \textit{et al.}\cite{zhou2016fast} proposed a fast global registration (FGR) algorithm, which is a representative M-estimation method. It formulates the registration problem as minimization of the Geman-McClure objective function and adapts graduated non-convexity (GNC) to solve the optimization problem. Though efficient, it is easy to generate incorrect solutions at a high outlier rate. The pioneering work of outlier removal is~\cite{albarelli2010game}, which proposed an inlier selection method based on the game-theoretic framework. However, this method is stochastic and has no global optimality guarantees. Parra and Chin\cite{bustos2017guaranteed} first proposed a guaranteed outlier removal (GORE) method, which leverages geometrical bounds to prune outliers and guarantees that eliminated correspondences are not inliers. In order to improve the efficiency, Cai \textit{et al.}\cite{cai2019practical} presented an extension of GORE for the 4-DOF terrestrial LiDAR registration problem, where the assumption of axis-fixed rotation is used. Motivated by GORE, Li \textit{et al.}\cite{9373914,10091912} proposed different variants by seeking new bounds, which could also reduce the computational cost. More recently, Yan \textit{et al.}\cite{yan2022new} presented an outlier removal strategy based on the reliability of the correspondence graph to increase efficiency. 

RANSAC and its variants\cite{fischler1981random,chum2003locally,lebeda2012fixing,barath2021graph,barath2021marginalizing} are widely used consensus maximization methods for robust estimation. Nonetheless, RANSAC-based methods are non-deterministic and generate a correct solution only with a certain probability due to the essence of random sampling\cite{le2019deterministic}. Therefore, many deterministic methods have been proposed, most of which are based on the globally optimal BnB framework\cite{bazin2012globally,bustos2017guaranteed,cai2019practical,chen2022deterministic,li2023fast}. BnB\cite{morrison2016branch} is a powerful global optimization algorithm framework for solving non-convex and NP-hard problems. However, the computational complexity of BnB is exponential to the dimension of the parameter space in worst-case. Existing methods\cite{bazin2012globally,bustos2017guaranteed,cai2019practical} commonly are computationally demanding since they jointly search the globally optimal solution in a high-dimensional parameter space. Hence, there is potential room for efficiency enhancement by considering the dimensionality of the registration problem.

\textbf{Transformation decoupling.} At present, there has been an increasing trend to apply transformation decoupling technique\cite{horn1987closed,arun1987least,yang2020teaser,chen2022deterministic,li2023fast} for accelerating correspondence-based registration. It can divide the original registration problem into multiple sub-problems, effectively reducing its dimensionality. In earlier research, transformation decoupling was integrated into closed-form solvers\cite{horn1987closed,arun1987least} to separately estimate rotation and translation. Recently, Yang \textit{et al.}\cite{yang2020teaser} proposed translation invariant measurements (TIM) to decouple the initial 6-DOF transformation into a 3-DOF rotation and 3-DOF translation. They presented an adaptive voting to address the translation estimation and employed GNC to solve the rotation sub-problem. Chen \textit{et al.}\cite{chen2022deterministic} introduced a decomposition strategy that also decouples the 6-DOF problem into two 3-DOF sub-problems. However, in the first sub-problem, it tackles the 2-DOF rotation and the 1-DOF translation, with the remaining DOFs determined in the second sub-problem. These two 3-DOF sub-problems are sequentially addressed by BnB. In our previous work\cite{li2023fast}, we presented a decoupling strategy for the registration problem with gravity prior. Leveraging known gravity directions, we decouple the joint transformation into a separate 3-DOF translation and 1-DOF rotation. The translation estimation is solved through BnB, while rotation estimation is accomplished via a global voting method. Decoupling-based methods are generally faster than those without decoupling. Motivated by this observation, we aim to develop a novel transformation decoupling scheme that further reduces the parameter space, particularly from the perspective of screw theory. 

In addition to the aforementioned geometric-based paradigms, recent studies have also utilized deep learning techniques for correspondence-based registration, such as 3DRegNet\cite{pais20203dregnet}, DGR\cite{choy2020deep}, PointDSC\cite{bai2021pointdsc}, DetarNet\cite{chen2022detarnet}, and VBReg\cite{jiang2023robust}. Nonetheless, a large amount of training data is necessary for them, which may not always be readily accessible in real-world applications. Therefore, this paper places its emphasis on geometric-based methods.

\subsection{Simultaneous Pose and Correspondence Registration}
Unlike correspondence-based registration, correspondence-free registration utilizes raw point cloud data as input instead of features and jointly estimates both the correspondences and transformation.

\textbf{Local registration.} ICP \cite{icp} is the most widely used SPCR method, which iteratively alternates between identifying correspondences and refining the transformation. Many variants of ICP have been proposed to enlarge the convergence basin and improve the robustness to outliers, such as EM-ICP\cite{granger2002multi}, LM-ICP\cite{fitzgibbon2003robust}, and trimmed-ICP\cite{chetverikov2005robust}. However, ICP and its variants are highly sensitive to the initialization and are prone to converge to local minima. By representing the point cloud with Gaussian mixture models (GMMs)\cite{jian2005robust}, some other approaches\cite{jian,CPD} convert the registration problem to a probability distribution alignment problem. Although improved performance has been achieved by GMM-based methods, they still converge only to local minima.

\textbf{Global registration.} To avoid local minima, an increasing number of global methods have been investigated, most of which are based on BnB, such as\cite{li20073d,7368945,7381673,campbell2016gogma,straub2017efficient,liu2018efficient,campbell2018globally,li2018fast,9485090}. Within these approaches, Go-ICP~\cite{7368945} can provide the globally optimal solution to the 6-DOF SPCR problem by nested BnB. A trimming strategy in GO-ICP was introduced to alleviate the effects of occlusion and partial overlap. Parra \textit{et al.}~\cite{7381673} presented a tighter bound function based on stereographic projections for fast rotation search and extended the proposed method to 6-DOF by nested BnB. Campbell \textit{et al.}~\cite{campbell2016gogma} introduced a GMM-based robust objective function for SPCR and employed BnB-based approach to find the globally optimal solution. Nonetheless, these methods are computationally expensive since they search for the optimal solution in the 6-dimensional solution domain. 

\textbf{Transformation decoupling.} To enhance efficiency, many studies also introduce the transformation decoupling strategy\cite{straub2017efficient,li2018fast,liu2018efficient,9485090} into the correspondence-free registration problem. For instance, Straub \textit{et al.}\cite{straub2017efficient} utilized the translation-invariant surface normal distributions for transformation decoupling, introducing a two-stage BnB algorithm for the sequential estimation of rotation and translation. Similarly, Li \textit{et al.}\cite{li2018fast} utilized translation invariant vectors (TIV) to decompose the 6-DOF transformation search into two sub-problems, first searching for rotation and then for translation. Liu \textit{et al.}\cite{liu2018efficient} introduced rotation invariant features (RIF) to decouple the joint 6-DOF search into two sequential search sub-problems for translation and rotation. However, the use of these pairwise features leads to a quadratic increase in input data, constraining efficiency gains. In contrast, our study introduces a novel transformation decoupling strategy without inflating the input data size. 

Instead of combining feature matching with robust estimation, some learning-based end-to-end registration methods also have been investigated, such as PointNetLK\cite{aoki2019pointnetlk}, DCP\cite{wang2019deep}, RPM-Net\cite{yew2020rpm}, DeepGMR\cite{yuan2020deepgmr}, and RegTR\cite{yew2022regtr}. Notably, learning-based methods require additional training procedures and their generalization capabilities are not always reliable\cite{yang2023mutual}.

\section{Problem Formulation}\label{Problem Formulation}
\subsection{Inlier Set Maximization with Screw Theory}
Typically, outlier-robust rigid point cloud registration can be formulated as a consensus maximization (a.k.a. \textit{inlier set maximization}) problem, where the cardinality of the inlier set is to be maximized. We assume that the set of putative correspondences $\mathcal{C}=\big\{(\bm{p}_i,\bm{q}_i)\big\}_{i=1}^N$ is extracted between the source point cloud \(\mathcal{P}\) and the target point cloud \(\mathcal{Q}\), where \(\bm{p}_i,\bm{q}_i\in\mathbb{R}^3\). Considering the existence of noise, we introduce a threshold $\epsilon$ for the identification of inliers. Thus the inlier set maximization problem for rigid registration is commonly organized as below,
\begin{equation}
    \begin{aligned}
    \underset {\bm{R},\bm{t},\mathcal{I}\subseteq\mathcal{H}} {\max} \ &\left|\mathcal{I}\right| \\
    \text{s.t.}\quad &\left\|\bm{Rp}_i+\bm{t}-\bm{q}_i\right\|\leq\epsilon,\ \forall i\in\mathcal{I}, \label{ori}
    \end{aligned}
\end{equation}
where \(\bm{R}\in \mathbb{SO}(3)\) is the rotation matrix, \(\bm{t}\in\mathbb{R}^3\) is the translation vector, $\mathcal{H}=\{1,\dots,N\}$ is the set of indices for $\mathcal{C}$, $\mathcal{I}$ represents the inlier set, $\left|\cdot\right|$ is the cardinality of a set, and $\left\|\cdot\right\|$ is the $L_2$-norm. This optimization problem aims to find the optimal $\bm{R}^*$ and $\bm{t}^*$ to maximize the cardinality of the inlier set. This problem is inherently a 6-DOF optimization problem since both $\bm{R}$ and $\bm{t}$ minimally require three parameters to be defined.

\textit{Screw theory}\cite{ball1876theory} is a widely used important tool in robot mechanics\cite{lynch2017modern} and computational geometry\cite{de1997computational}. It includes a fundamental theorem, which is known as \textit{Chasles' Theorem}\cite{hunt1978kinematic,bottema1990theoretical}, as shown in Theorem~\ref{theorem1}. In this paper, we will reformulate the rigid point cloud registration problem from the perspective of screw theory.
\begin{figure}
\centering
\includegraphics[width=0.7\columnwidth]{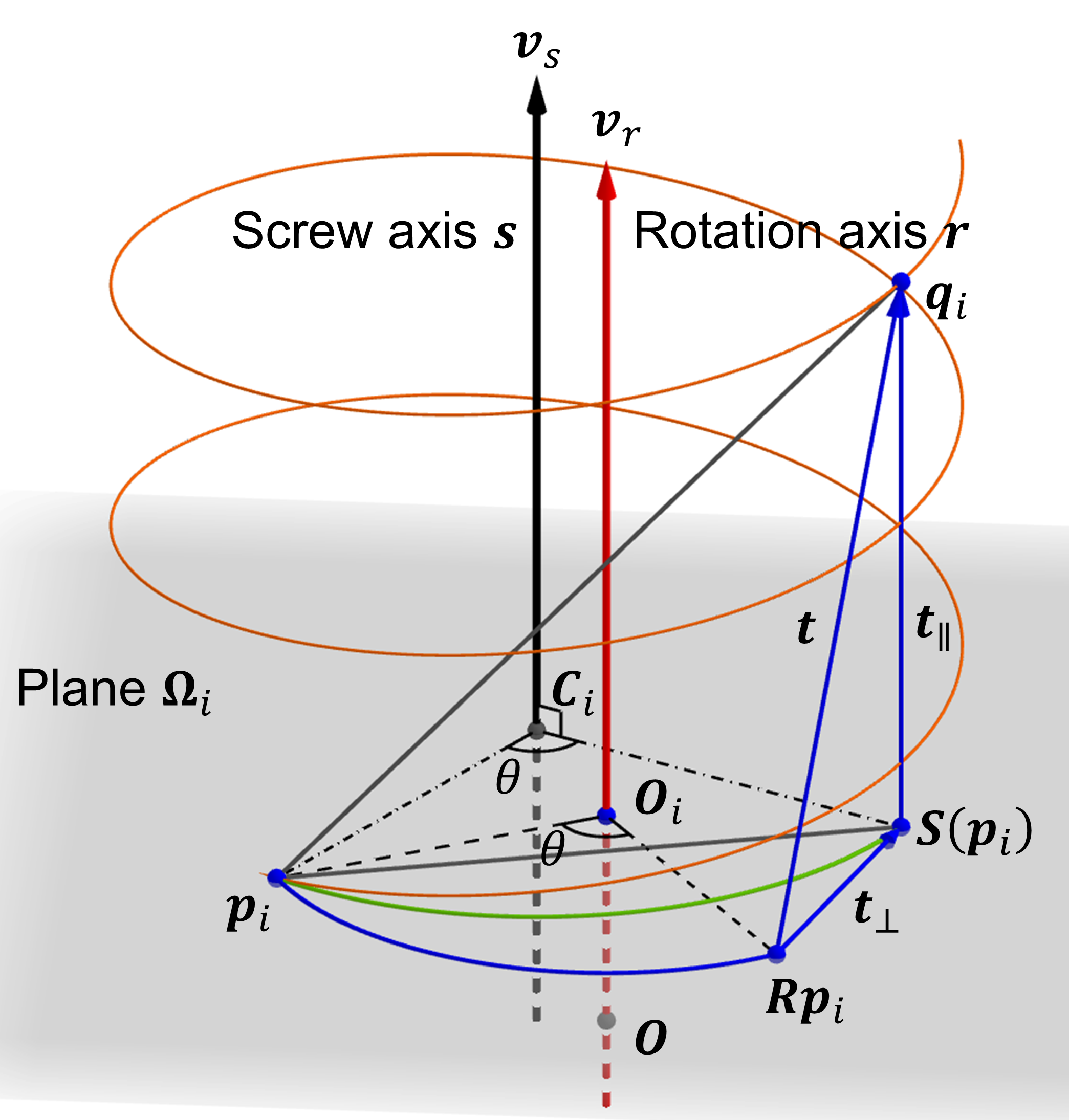}
\caption{Schematic of Chasles' Theorem and screw transformation. Specifically, the Euclidean transformation of point $\bm{p}_i$, i.e., $\bm{Rp}_i+\bm{t}$ can be represented by a screw rotation $\bm{S}(\bm{p}_i)$ and a screw translation $\bm{t}_{\parallel}$. The 
screw rotation $\bm{S}(\bm{p}_i)$ is to rotate $\bm{p}_i$ about a screw axis $\bm{s}$, which is through the rotation center point $\bm{C}_i$ and is parallel to the rotation axis $\bm{r}$. Notably, $\bm{p}_i$, $\bm{S}(\bm{p}_i)$ and $\bm{C}_i$ are in the same plane $\bm{\Omega}_i$. }
\label{fig_1}
\end{figure}

\noindent\fbox{\parbox{0.98\linewidth}{
\begin{theorem}[Chasles' Theorem\cite{hunt1978kinematic,bottema1990theoretical}]\label{theorem1}
Each Euclidean transformation in three-dimensional space has a screw axis, and the transformation can be decomposed into a rotation about and a translation along this screw axis.
\end{theorem}
}}

Chasles' Theorem indicates that, the six parameters of a Euclidean transformation contain the four independent components that define the \textit{screw axis}, together with the rotation angle about and translation along this screw axis\cite{ball1876theory}. Typically, the six parameters defining the Euclidean transformation can also be given by three Euler angles of rotation and three translation components. In screw theory, the Euclidean transformation is also known as the \textit{screw transformation}, the rotation about the screw axis is known as the \textit{screw rotation}, and the translation along the screw axis is known as the \textit{screw translation}. Mathematically, according to Chasles' Theorem, the screw transformation of $\bm{p}_i$ can be represented by
\begin{subequations}
\begin{align}
\bm{T}(\bm{p}_i)&=\bm{Rp}_i+\bm{t}\\
&=(\bm{R}\bm{p}_i+\bm{t}_{\perp})+\bm{t}_{\parallel}\\
&=\bm{S}(\bm{p}_i)+\bm{t}_{\parallel}\label{screw_},
\end{align}
\end{subequations}
where \(\bm{T}\in \mathbb{SE}(3)\) denotes the Euclidean transformation, $\bm{S}(\cdot)$ denotes the screw rotation, and $\bm{t}_{\perp}$ and $\bm{t}_{\parallel}$ are the translation components perpendicular and parallel to the screw axis, respectively. Thus the screw rotation $\bm{S}(\cdot)$ contains the rotation $\bm{R}$ and the translation $\bm{t}_{\perp}$. The screw translation is the translation component $\bm{t}_{\parallel}$. 

The geometrical interpretation of Chasles' Theorem and screw transformation is shown in Fig.~\ref{fig_1}. We define the unit-norm constrained orientation vectors of the screw and rotation axis as $\bm{v}_s$ and $\bm{v}_r$. Specifically, the screw axis $\bm{s}$ can be defined by the orientation vector $\bm{v}_r$ and a unique point $\bm{C}_i\in\mathbb{R}^3$ in the \textit{rotation plane} $\bm{\Omega}_i$ that is perpendicular to $\bm{v}_r$ and through the point $\bm{p}_i$. Thus the screw axis is naturally parallel to the rotation axis, i.e., $\bm{v}_s$ is parallel to $\bm{v}_r$. In three-dimensional space, the screw rotation $\bm{S}(\cdot)$ can be defined as 
\begin{equation}
\bm{S}(\bm{p}_i)=\bm{R}(\bm{p}_i-\bm{C}_i)+\bm{C}_i.
\end{equation} 
For this unique point $\bm{C}_i$, we can obtain
\begin{equation}
\bm{S}(\bm{C}_i)=\bm{C}_i.
\end{equation}
Therefore, according to the definition of rotation, the rigid motion $\bm{S}(\cdot)$, which can keep a point fixed (i.e., $\bm{C}_i$), is a rotational motion, called \textit{screw rotation}. Both the orientation of the screw axis and the rotation angle of the screw rotation are identical to that of the original rotation $\bm{R}$. The only difference is that the rotation axis $\bm{r}$ is through the origin $\bm{O}$ by default, while the screw axis $\bm{s}$ is through the point $\bm{C}_i$. 

In general, the 6-DOF rigid registration problem (\ref{ori}) under screw theory can be rewritten as
\begin{equation}
    \begin{aligned}
    \underset {\bm{R},\bm{C}_i,\bm{t}_{\parallel},\mathcal{I}\subseteq\mathcal{H}} {\max} \ &\left|\mathcal{I}\right| \\
    \text{s.t.}\quad &\left\|\bm{S}(\bm{p}_i)+\bm{t}_{\parallel}-\bm{q}_i\right\|\leq\epsilon,\ \forall i\in\mathcal{I}.\label{screw}
    \end{aligned}
\end{equation}
Solving this optimization problem thus amounts to searching for $\bm{R}$ (a 3-DOF problem), $\bm{C}_i$ (a 2-DOF problem), and $\bm{t}_{\parallel}$ (a 1-DOF problem).

On the other hand, we can also obtain the following theorem about the planar specialization of Chasles'Theorem.

\noindent\fbox{\parbox{0.98\linewidth}{
\begin{theorem}[Planar Chasles' Theorem\cite{hunt1978kinematic,bottema1990theoretical}]\label{theorem2}
When a Euclidean transformation specializes to a planar transformation in two-dimensional space, the screw axis becomes a pole, and the planar transformation can be represented by a rotation about this pole.
\end{theorem}
}}

Theorem~\ref{theorem2} indicates that, a 2D rigid transformation (e.g., transform point $\bm{p}$ to $\bm{q}$) in the plane can be converted to a pure rotational motion around a unique \textit{pole} $\bm{C}_0$. This pure rotational motion is a 2D \textit{screw rotation}, denoted by $\bm{S}_2(\cdot)$, and we can obtain
\begin{equation}
\bm{S}_2(\bm{p})=\bm{R}_2(\bm{p}-\bm{C}_0)+\bm{C}_0=\bm{q},
\end{equation}
where \(\bm{p},\bm{q},\bm{C}_0\in\mathbb{R}^2\), and $\bm{R}_2$ is a 2D rotation matrix. In this study, Theorem~\ref{theorem2} is also one of the fundamental parts of our proposed transformation decoupling strategy.

\begin{figure*}[!t]
 \centering
 \includegraphics[width=1\textwidth]{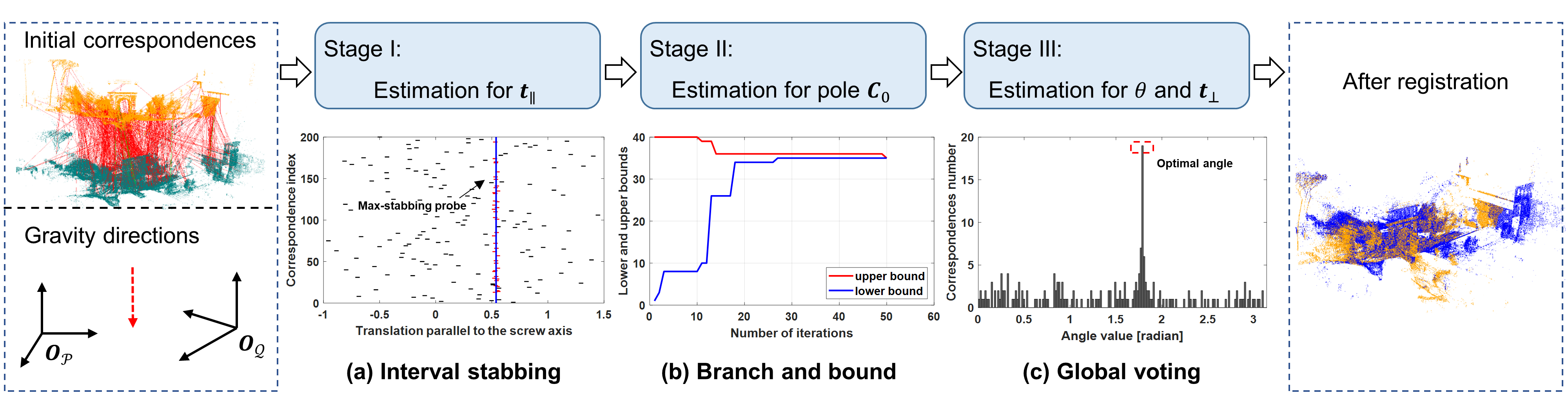}
 \caption{Calculation flow of the proposed three-stage method for correspondence-based registration. The source, target, and aligned point clouds are green, orange, and blue, respectively. In the initial correspondences set, green and red line segments represent inliers and outliers respectively (only a subset of correspondences is shown for visual clarity). For the interval stabbing part, the candidate intervals are represented by the black line segments, while the intervals crossed by the max-stabbing probe are depicted as the red line segments.}
 \label{fig_2}
\end{figure*}

\subsection{Rigid Registration with Gravity Prior}

In the case that the gravity directions are given, we assume they are denoted by unit-norm constrained vectors $\bm{v}_p$, $\bm{v}_q$ for the source and target point clouds $\mathcal{P}$, $\mathcal{Q}$, respectively. The constraint of gravity directions is given by
\begin{equation}
    \bm{v}_q = \bm{R}\bm{v}_p.
\end{equation}
The solution of this equation is given by \cite{bustos2017guaranteed,liu2023absolute,li2023fast}
\begin{equation}
    \bm{R}=\bm{R}(\theta,\bm{v}_q)\bm{R}_{\bm{v}_p}^{\bm{v}_q},
    \label{to_screw_axis}
\end{equation}
where $\bm{R}(\theta,\bm{v}_q)$ is the rotation that rotates $\theta$ about axis $\bm{v}_q$, and $\bm{R}_{\bm{v}_p}^{\bm{v}_q}$ is the rotation that rotates $\bm{v}_p$ to $\bm{v}_q$ with the minimum geodesic motion. Eq.~(\ref{to_screw_axis}) indicates that the rotation $\bm{R}$ is only dependent on the rotation angle $\theta\in [-\pi,\pi]$ when the gravity directions are prior known. Without loss of generality, we can align the Z-axis of the source point cloud $\mathcal{P}$ to the Z-axis of the target point cloud $\mathcal{Q}$ by the following rotation,
\begin{equation}
    \bm{p}'_i=\bm{R}_{\bm{v}_p}^{\bm{v}_q}\bm{p}_i.
\end{equation}
After this rotation, the 6-DOF rigid registration problem (\ref{screw}) can be reduced to a 4-DOF problem, i.e.,
\begin{equation}
    \begin{aligned}
    \underset {\theta,\bm{C},\bm{t}_{\parallel},\mathcal{I}\subseteq\mathcal{H}} {\max} \ &\left|\mathcal{I}\right| \\
    \text{s.t.}\quad &\left\|\bm{S}'(\bm{p}'_i)+\bm{t}_{\parallel}-\bm{q}_i\right\|\leq\epsilon,\ \forall i\in\mathcal{I},\label{4dof}
    \end{aligned}
\end{equation}
where $\bm{S}'(\bm{p}'_i)=\bm{R}'(\bm{p}'_i-\bm{C})+\bm{C}$, and we set $\bm{R}'\triangleq\bm{R}(\theta,\bm{v}_q)$ for convenience. Please note that in problem (\ref{4dof}), the orientation of the screw axis is parallel to $\bm{v}_q$ rather than the original rotation axis $\bm{v}_r$, i.e., $\bm{v}_q$ is the current rotation axis. Accordingly, $\bm{t}_{\parallel}$ in problem (\ref{4dof}) is the translation component parallel to $\bm{v}_q$, and $\bm{t}_{\perp}$ is the component perpendicular to $\bm{v}_q$.

\section{Three-stage Consensus Maximization Registration}
Due to the high-dimensional parameter space, solving the 4-DOF problem (\ref{4dof}) jointly is relatively time-consuming\cite{cai2019practical}. To accelerate the registration, we first reduce the 4-DOF original problem to a 1-DOF sub-problem with the aid of known gravity directions in Section~\ref{stage1}. The 1-DOF sub-problem is estimating the translation parallel to the screw axis. Then we decouple the remaining 3-DOF sub-problem into a 2-DOF and a 1-DOF sub-problem by screw theory, which is presented in Section~\ref{stage2} and Section~\ref{stage3} respectively. The 2-DOF sub-problem is searching for the \textit{pole}, an auxiliary variable in screw theory. The last 1-DOF sub-problem is estimating the rotation angle. After acquiring the rotation angle, we can readily calculate the translation orthogonal to the rotation axis and thereby obtain the final optimal solution. Same to the original problem, we formulate all three sub-problems as consensus maximization problems. The pipeline of the proposed method is given in Fig.~\ref{fig_2}.

\subsection{Stage I: Estimation for the Translation Parallel to the Screw Axis}\label{stage1}

Using the known gravity directions, we can reduce the original problem into a 1-DOF sub-problem that solely involves the translation parallel to the screw axis. Specifically, for an ideal inlier correspondence $(\bm{p}_i,\bm{q}_i)$, we have the following derivation about the original constraint in Eq.~(\ref{4dof}),
\begin{subequations}
\begin{align}
&\bm{q}_i=\bm{S}'(\bm{p}'_i)+\bm{t}_{\parallel}
\\
\Leftrightarrow&\bm{q}_i-\bm{p}'_i=\bm{S}'(\bm{p}'_i)-\bm{p}'_i+\bm{t}_{\parallel}
\\
\Rightarrow&\bm{v}_q^\mathrm{T}\left(\bm{q}_i-\bm{p}'_i\right)=\bm{v}_q^\mathrm{T}\left[ \bm{S}'(\bm{p}'_i)-\bm{p}'_i+\bm{t}_{\parallel}\right]\label{c}
\\
\Leftrightarrow&\bm{v}_q^\mathrm{T}\left(\bm{q}_i-\bm{p}'_i\right)=\left\|\bm{t}_{\parallel}\right\|\label{d}
\end{align}
\end{subequations}
where Eq.~(\ref{d}) is from the fact that vector $\bm{S}'(\bm{p}'_i)-\bm{p}'_i$ is perpendicular to $\bm{v}_q$ (see Fig.~\ref{fig_1}), and $\bm{t}_{\parallel}$ is parallel to $\bm{v}_q$. Considering the noise, we have the following new inlier constraint,
\begin{equation}
  \left|\left\|\bm{t}_{\parallel}\right\|+\bm{v}_q^\mathrm{T}\left(\bm{p}'_i-\bm{q}_i\right)\right|\leq\delta,\label{e}
\end{equation}
where $\delta$ is the inlier threshold. Since the gravity direction $\bm{v}_q$ is given, the constraint in Eq.~(\ref{e}) is only dependent on $\left\|\bm{t}_{\parallel}\right\|$. Accordingly, we can rewrite Eq.~(\ref{e}) in the form of an interval for $l\triangleq \left\|\bm{t}_{\parallel}\right\|$, i.e.,
\begin{equation}
    -\delta-\bm{v}_q^\mathrm{T}\left(\bm{p}'_i-\bm{q}_i\right)\leq l \leq\delta-\bm{v}_q^\mathrm{T}\left(\bm{p}'_i-\bm{q}_i\right).\label{interval}
\end{equation}

Given the set of initial correspondences and gravity directions, the first sub-problem aims at estimating the optimal $l^*$ to maximize the cardinality of the inlier set, which is defined by
\begin{equation}
    \begin{aligned}
    \underset {l,\mathcal{I}_1\subseteq\mathcal{H}} {\max} \ &\left|\mathcal{I}_1\right| \\
    \text{s.t.}\quad &l\in\left[ l_i^-,l_i^+\right],\ \forall i\in\mathcal{I}_1,\\ 
    & l_i^-=-\delta-\bm{v}_q^\mathrm{T}\left(\bm{p}'_i-\bm{q}_i\right),\\
    & l_i^+=\delta-\bm{v}_q^\mathrm{T}\left(\bm{p}'_i-\bm{q}_i\right),\\\label{problem1}
    \end{aligned}
\end{equation}
where $\mathcal{I}_1$ denotes the inlier set that only satisfies the constraint in Eq.~(\ref{problem1}). From the perspective of computational geometry, solving the consensus maximization problem (\ref{problem1}) is a typical \textit{interval stabbing} problem\cite{de1997computational}. Recently, the interval stabbing technology is widely used to solve different geometric optimization problems\cite{bustos2017guaranteed,cai2019practical,peng2022arcs,yan2022new}.

As depicted in Fig.~\ref{fig_2}, the interval stabbing problem is concerned with finding a probe (represented by the blue line segment) that stabs the maximum number of intervals. The interval stabbing problem can be efficiently solved with a time complexity of $\mathcal{O}(N\log N)$. On the basis of existing algorithms, the proposed interval stabbing algorithm (shown in appendix A) enhances the precision by returning the midpoint of the maximum overlapping interval instead of the common left endpoint\cite{cai2019practical}. The returned \textit{max-stabbing number} is the maximized cardinality of the inlier set. The returned \textit{max-stabbing position} is the value of optimal $l^*$. Besides, since $l^*$ is a scalar, we can obtain the optimal translation parallel to $\bm{v}_q$ by
\begin{equation}
    \bm{t}_\parallel^* = l^* \cdot \bm{v}_q
    \label{result of t_parallel}
\end{equation}

\subsection{Stage II: Searching for the Pole}\label{stage2}

\begin{figure}
\centering
\includegraphics[width=1\columnwidth]{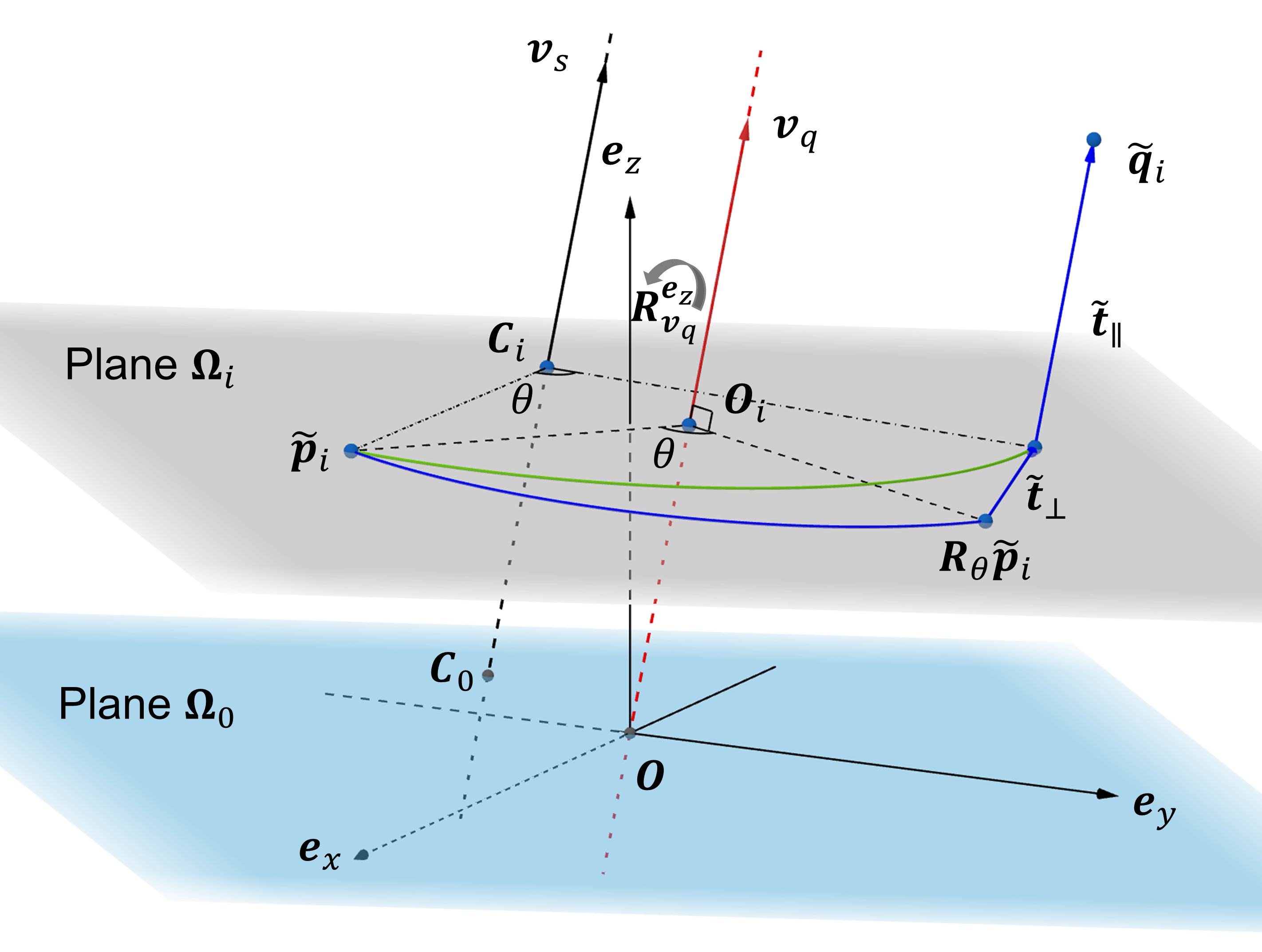}
\caption{The geometrical interpretation of the coordinate conversion by $\bm{R}_{\bm{v}_q}^{\bm{e}_z}$ and the plane projection to $\bm{\Omega}_0$.}
\label{fig_3}
\end{figure}

In this section, we first reduce the original 4-DOF problem (Eq.~(\ref{4dof})) into a 2D-2D rigid registration problem with 3-DOF using \textit{plane projection}. Subsequently, we decouple this 3-DOF registration problem into a 2-DOF sub-problem and a 1-DOF sub-problem by screw theory. This section focuses on solving the 2-DOF sub-problem, concretely searching for the pole.

\subsubsection{Plane Projection and Pole}
For an ideal inlier correspondence $(\bm{p}_i,\bm{q}_i)$, the constraint in Eq.~(\ref{4dof}) can be rewritten as,
\begin{equation}
    \bm{q}_i=\bm{R}'\bm{p}'_i+\bm{t}_{\perp}+\bm{t}_{\parallel}.
    \label{ideal}
\end{equation}
It can be observed that only the translation $\bm{t}_{\parallel}$ is along the screw axis, while the rotation $\bm{R}'$ and the translation $\bm{t}_{\perp}$ maintain the point $\bm{p}'_i$ within the rotation plane $\bm{\Omega}_i$. This plane is perpendicular to $\bm{v}_q$ and passes through $\bm{p}'_i$. Therefore, given the gravity direction $\bm{v}_q$, we can project the set of correspondence $\mathcal{C}$ from the three-dimensional space onto a two-dimensional plane $\bm{\Omega}_0$ that is perpendicular to $\bm{v}_q$ and through the origin $\bm{O}$ (i.e., $\bm{\Omega}_0$ is parallel to each $\bm{\Omega}_i$). After plane projection, the original 3D rigid transformation can be reduced to a 2D transformation with 3-DOF, which solely contains the 1-DOF rotation angle $\theta$ and the 2-DOF translation $\bm{t}_{\perp}$. Interestingly, the 2D rigid transformation can be represented by a 2D screw rotation (Theorem~\ref{theorem2}). In the next step, we first conduct the following coordinate conversion operation to achieve plane projection, and the geometrical illustration is given in Fig.~\ref{fig_3}.

Firstly, we introduce a new coordinate system, in which the rotation axis $\bm{v}_q$ corresponds to the Z-axis (denoted by $\bm{e}_z=[0,0,1]^{\mathrm{T}}$). The transformation from the original coordinate system to the new one is given by
\begin{equation}
    \bm{e}_z=\bm{R}_{\bm{v}_q}^{\bm{e}_z}\bm{v}_q.
\end{equation}
We can transform each vector in Eq.~(\ref{ideal}) to the new coordinate system by
\begin{equation}
    \begin{aligned}
        \tilde{\bm{p}}_i=\bm{R}_{\bm{v}_q}^{\bm{e}_z}\bm{p}'_i,\ \tilde{\bm{q}}_i=\bm{R}_{\bm{v}_q}^{\bm{e}_z}\bm{q}_i,\
        \tilde{\bm{t}}_{\perp}=\bm{R}_{\bm{v}_q}^{\bm{e}_z}\bm{t}_{\perp},\ \tilde{\bm{t}}_{\parallel}=\bm{R}_{\bm{v}_q}^{\bm{e}_z}\bm{t}_{\parallel}.
    \end{aligned}
    \label{new coordinates}
\end{equation}
Then we have the following coordinate conversion, 
\begin{subequations}
\begin{align}
&\bm{q}_i=\bm{R}'\bm{p}'_i+\bm{t}_{\perp}+\bm{t}_{\parallel}
\\
\Leftrightarrow&\bm{R}_{\bm{v}_q}^{\bm{e}_z}\bm{q}_i=\bm{R}_{\bm{v}_q}^{\bm{e}_z}\bm{R}'({\bm{R}_{\bm{v}_q}^{\bm{e}_z}})^{\mathrm{T}}\bm{R}_{\bm{v}_q}^{\bm{e}_z}\bm{p}'_i+\bm{R}_{\bm{v}_q}^{\bm{e}_z}\bm{t}_{\perp}+\bm{R}_{\bm{v}_q}^{\bm{e}_z}\bm{t}_{\parallel}
\\
\Leftrightarrow&\tilde{\bm{q}}_i=\bm{R}_{\theta}\tilde{\bm{p}}_i+\tilde{\bm{t}}_{\perp}+\tilde{\bm{t}}_{\parallel},\label{18c}
\end{align}
\end{subequations}
where $\bm{R}_{\theta}\triangleq \bm{R}_{\bm{v}_q}^{\bm{e}_z}\bm{R}'({\bm{R}_{\bm{v}_q}^{\bm{e}_z}})^{\mathrm{T}}$ is the rotation matrix in the new coordinate system, as shown in Fig.~\ref{fig_3}. The rotation axis in this coordinate system is now the Z-axis, thus $\bm{R}_{\theta}$ can be denoted as 
\begin{equation}
    \bm{R}_{\theta} \triangleq \begin{bmatrix}
    \cos\theta & -\sin\theta & 0 \\
    \sin\theta & \cos\theta & 0 \\
    0 & 0 & 1
    \end{bmatrix}.\label{20}
\end{equation}
Furthermore, $\tilde{\bm{t}}_{\perp}$ and $\tilde{\bm{t}}_{\parallel}$ have the following form
\begin{equation}
    \tilde{\bm{t}}_{\perp} \triangleq \left[
        t_a, t_b, 0
    \right]^\mathrm{T},\ 
    \tilde{\bm{t}}_{\parallel} \triangleq \left[
        0, 0, l
    \right]^\mathrm{T}.\label{21}
\end{equation}

After projecting the 3D rigid transformation in Eq.~(\ref{18c}) onto the 2D plane $\bm{\Omega}_0$, we can obtain the 2D rigid transformation as shown below
\begin{equation}
    \hat{\bm{q}}_i=\hat{\bm{R}}_{\theta}\hat{\bm{p}}_i+\hat{\bm{t}}_{\perp},\
    \hat{\bm{R}}_{\theta}={\begin{bmatrix} \cos\theta & -\sin\theta \\ \sin\theta &  \cos\theta \end{bmatrix}},\label{2d_equation}
\end{equation}
where $\hat{\bm{t}}_{\perp}=\left[t_a, t_b\right]^\mathrm{T}$, $\hat{\bm{p}}_i\triangleq[x_p, y_p]^\mathrm{T}$, and $\hat{\bm{q}}_i\triangleq[x_q, y_q]^\mathrm{T}$. The geometrical illustration for plane projection is given in Fig.~\ref{fig_4}. Estimating the optimal $\hat{\bm{R}}_{\theta}$ and $\hat{\bm{t}}_{\perp}$ is a typical 2D-2D rigid registration problem.
\begin{figure}
\centering
\includegraphics[width=0.75\columnwidth]{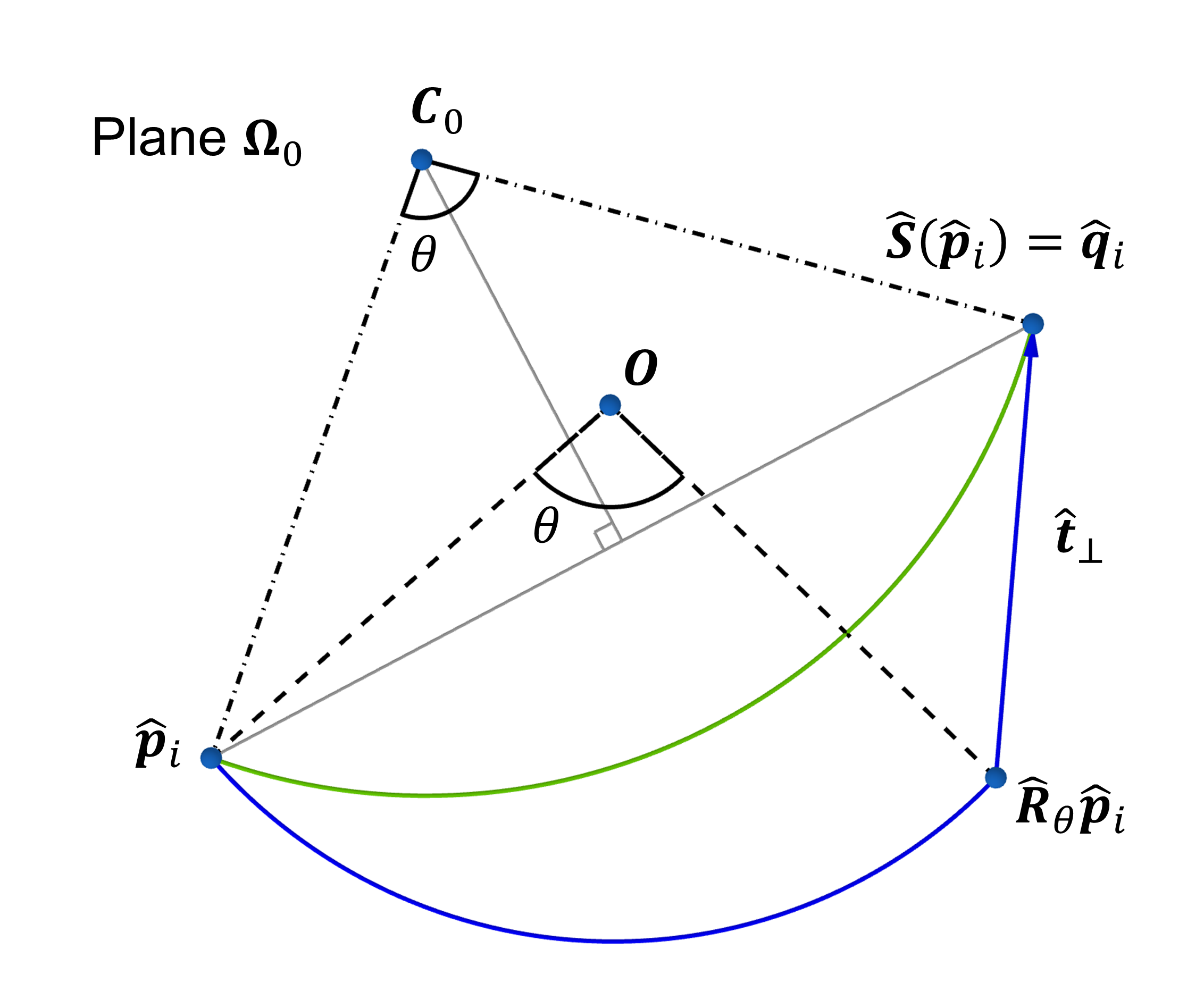}
\caption{The geometrical interpretation of the 2D rigid transformation and the 2D screw rotation after plane projection.}
\label{fig_4}
\end{figure}
 
In our formulation based on screw theory, the 2D rigid transformation can be reformulated by a 2D screw rotation, as presented in Theorem~\ref{theorem2}. The 2D screw rotation is a pure rotational motion around \textit{pole} $\bm{C}_0\in\mathbb{R}^2$, which is the projection of $\bm{C}_i\in\mathbb{R}^3$. We can rewrite Eq.~(\ref{2d_equation}) as,
\begin{equation}
    \hat{\bm{q}}_i=\bm{\hat R}_{\theta}(\hat{\bm{p}}_i-\bm{C}_0)+\bm{C}_0=\bm{\hat S}(\hat{\bm{p}}_i). 
    \label{2D_pure_rotation}
\end{equation}
Accordingly, under screw theory the 2D-2D rigid registration problem consists of estimating the 2-DOF rotation center $\bm{C}_0$ and the 1-DOF rotation angle $\theta$, as shown in Fig.~\ref{fig_4}. An important property of this rotation center (pole) is that it must stand on the vertical bisector of each line segment $\hat{\bm{p}}_i\hat{\bm{q}}_i$, as presented in Proposition~\ref{proposition1}.

\begin{proposition}\label{proposition1}
When a point rotates around a rotation center on a plane, the rotation center must fall on the perpendicular bisector of the line segment connecting the two corresponding points.
\end{proposition}

We first define the perpendicular bisector as a linear function of the form $a_ix+b_iy+c_i=0$. For an ideal inlier correspondence $(\hat{\bm{p}}_i,\hat{\bm{q}}_i)$, we can obtain
\begin{equation}
    \begin{aligned}
    a_i &= x_q-x_p, \\
    b_i &= y_q-y_p, \\
    c_i &= -(\frac{x_p+x_q}{2}a_i+\frac{y_p+y_q}{2}b_i).\label{line}
    \end{aligned}
\end{equation}
Pole $\bm{C}_0\triangleq\left[C_x,C_y\right]^\mathrm{T}$ should stand on this line. However, when the rotation angle tends to zero, $\bm{C}_0$ stands in an infinite far position, which can not be expressed in a normal way. To avoid this limitation, we use the homogeneous coordinate $\hat{\bm{C}}\triangleq\left[c_x,c_y,w\right]^\mathrm{T}$ to represent pole $\bm{C}_0$, and the constraint of such coordinate is $\|\hat{\bm{C}}\|=1$. We also use the vector expressions for every perpendicular bisector, i.e., $\bm{n}_i\triangleq\left[a_i,b_i,c_i\right]^\mathrm{T}$. Therefore, the constraint of the pole can be expressed as 
\begin{equation}
    a_ic_x+b_ic_y+c_iw=\bm{n}_i^\mathrm{T}\hat{\bm{C}}=0.
\end{equation}
Considering the noise, estimating pole $\bm{C}_0$ can be formulated as a \textit{linear model fitting problem}\cite{wang2021practical,liu2022globally}. We can solve this problem by addressing the following consensus maximization problem, in which we seek an optimal unit-norm constrained $\hat{\bm{C}}$ with the largest number of candidate inliers.
\begin{equation}
 \begin{aligned}
     \underset {\hat{\bm{C}},\mathcal{I}_2\subseteq\mathcal{I}_1^*} {\max} \ &\left|\mathcal{I}_2\right| \\
     \text{s.t.}\quad &|\bm{n}_i^\mathrm{T} \hat{\bm{C}}|\leq \tau,\ \forall i\in\mathcal{I}_2,\\
     &\|\hat{\bm{C}}\|=1,
     \label{problem2}
 \end{aligned}    
\end{equation}
where $\mathcal{I}_2$ is the inlier set extracted from the cardinality-maximized inlier set $\mathcal{I}_1^*$, and $\tau$ is the inlier threshold. 

\subsubsection{Branch-and-Bound}
Since searching for the exact solution of consensus maximization problem (\ref{problem2}) is nonconvex and NP-hard\cite{tat2020robust}, we utilize the globally optimal and deterministic branch-and-bound (BnB) algorithm\cite{morrison2016branch} to solve it. Specifically, the BnB algorithm systematically explores the entire parameter space by iteratively dividing it into smaller sub-branches and calculating the upper and lower bound for the objective function on each sub-branch. It discards those sub-branches where larger objective function values than the current optimal value are impossible. As the solution domain progressively narrows down, the gap between the upper and lower bounds gradually diminishes until zero, and then the BnB algorithm achieves the optimal solution. An example of the convergence curve of the proposed BnB algorithm is given in Fig.~\ref{fig_2}.

The first step in constructing the BnB algorithm is the parameterization of the solution domain. Geometrically, the unit-norm constrained vector $\hat{\bm{C}}$ lies on the surface of a \textit{unit sphere}. We denote the unit sphere as $\mathbb{S}^{2}$. Since $\hat{\bm{C}}$ and $-\hat{\bm{C}}$ have the same inlier set, we can set the solution domain of $\hat{\bm{C}}$ as a hemisphere denoted by $\mathbb{S}^{2+}=\{\bm{h}=[h_1,h_2,h_3]^\mathrm{T} \big| \|\bm{h}\|=1,h_3\geq 0\}$. We then use a compact representation method, \textit{exponential mapping}\cite{hartley2009global,liu2020globallypami}, to map the 3D unit hemisphere to a 2D disk. Concretely, the vector $\bm{h}\in\mathbb{S}^{2+}$ corresponds to a unique point $\bm{\varphi}\in\mathbb{R}^2$ in the disk,
\begin{subequations}
\begin{align}
    \bm{h}&=\left[\sin\omega\cdot\hat{\bm{\varphi}}^\mathrm{T},\cos\omega \right]^\mathrm{T}\\
    \bm{\varphi}&=\hat{\bm{\varphi}}\cdot\omega
\end{align}
\end{subequations}
where $\omega\in[0,\pi/2]$, and $\hat{\bm{\varphi}}$ is a unit vector in $\mathbb{R}^{2}$. Therefore, to facilitate manipulation, the solution domain of $\hat{\bm{C}}$ is defined as a circumscribed square (with a radius of $\pi/2$) of the disk domain, as shown in Fig.~\ref{fig_5}.

The next step is estimating the upper and lower bound for the sub-branch. We first introduce the following lemma.

\begin{lemma}
\label{exp.mapping}
Given a square-shaped sub-branch $\mathbb{B}$ in the exponential mapping plane, its center is $\bm{\varphi}_c\in \mathbb{R}^2$ and half-side length is $\gamma$. For $\forall\bm{\varphi}\in\mathbb{B}$, we have
\begin{equation}
\angle\left(\bm{h},\bm{h}_c\right)\leq \left\| \bm{\varphi}-\bm{\varphi}_c \right\| \leq \sqrt{2}\gamma,\label{28}
\end{equation}
where $\bm{\varphi}$ and $\bm{\varphi}_c$ correspond to $\bm{h}\in\mathbb{S}^{2+}$ and $\bm{h}_c\in\mathbb{S}^{2+}$, respectively.
\end{lemma}
\begin{proof}
    The completed proof is given in \cite{liu2020globallypami}.
\end{proof}

Based on Lemma~\ref{exp.mapping}, the upper and lower bound of the proposed BnB algorithm for problem (\ref{problem2}) can be set as  
\begin{lemma} 
\label{bounds}
Given a square-shaped sub-branch $\mathbb{B}$, whose center is $\bm{\varphi}_c\in \mathbb{R}^2$ (corresponds to $\bm{h}_c\in\mathbb{S}^{2+}$ by exponential mapping) and half-side length is $\gamma$, the upper bound $U(\mathbb{B})$ and lower bound 
$L(\mathbb{B})$ can be set as

\begin{subequations}
\begin{align}
&U(\mathbb{B})=\sum\limits_{i=1}^{N'}\mathbb{I}\big(|\bm{n}_i^\mathrm{T}\bm{h}_c|\leq \Psi_i\big),\\
&L(\mathbb{B})=\sum\limits_{i=1}^{N'}\mathbb{I}\big(|\bm{n}_i^\mathrm{T}\bm{h}_c|\leq \tau\big),\\
&{\Psi_i}=\begin{cases}     \left\|\bm{n_i}\right\|\sin(\sqrt{2}\gamma+\xi_i), &\sqrt{2}\gamma+\xi_i<\pi/2 \\
 \left\|\bm{n_i}\right\|, &\sqrt{2}\gamma+\xi_i\geq\pi/2
\end{cases}
\end{align}
\end{subequations}
where $\xi_i\triangleq\arcsin (\tau/\left\|\bm{n_i}\right\|)$, $\tau$ is the inlier threshold, and $N'$ is the cardinality of the inlier set $\mathcal{I}_1^*$.
\end{lemma}

\begin{proof}
    The completed proof is given in \cite{liu2022globally}.
\end{proof}

Based on Lemma~\ref{bounds}, the proposed BnB algorithm for the 2-DOF pole search sub-problem is given in appendix A. To improve the algorithm efficiency, the input data is set as the candidate inlier correspondence set of problem (\ref{problem1}).
After we obtain optimal solution $\hat{\bm{C}}^*=[c_x^*,c_y^*,w^*]^{\mathrm{T}}$ by the proposed algorithm, we can transform it back to the 2D coordinate by 
\begin{equation}
    \bm{C}_0^*=\left[\frac{c_x^*}{w^*},\frac{c_y^*}{w^*}\right]^\mathrm{T}
\end{equation}

\begin{figure}
\centering
\includegraphics[width=0.7\columnwidth]{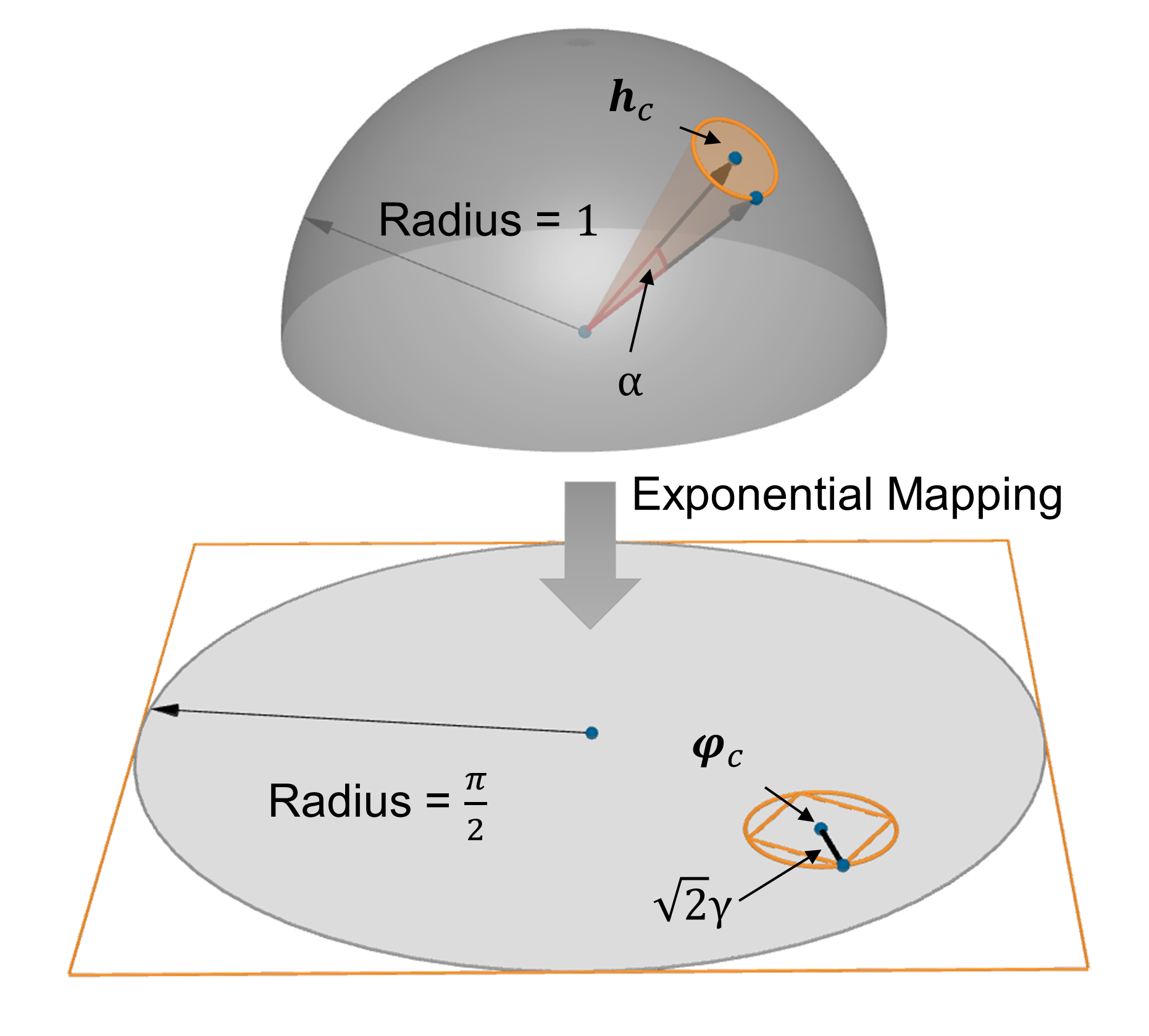}
\caption{Exponential mapping of the solution domain and the geometrical interpretation of Lemma~\ref{exp.mapping}.}
\label{fig_5}
\end{figure}

\subsection{Stage III: Voting for the Rotation Angle}\label{stage3}

\begin{figure*}[!t]
 \centering
 \includegraphics[width=0.8\textwidth]{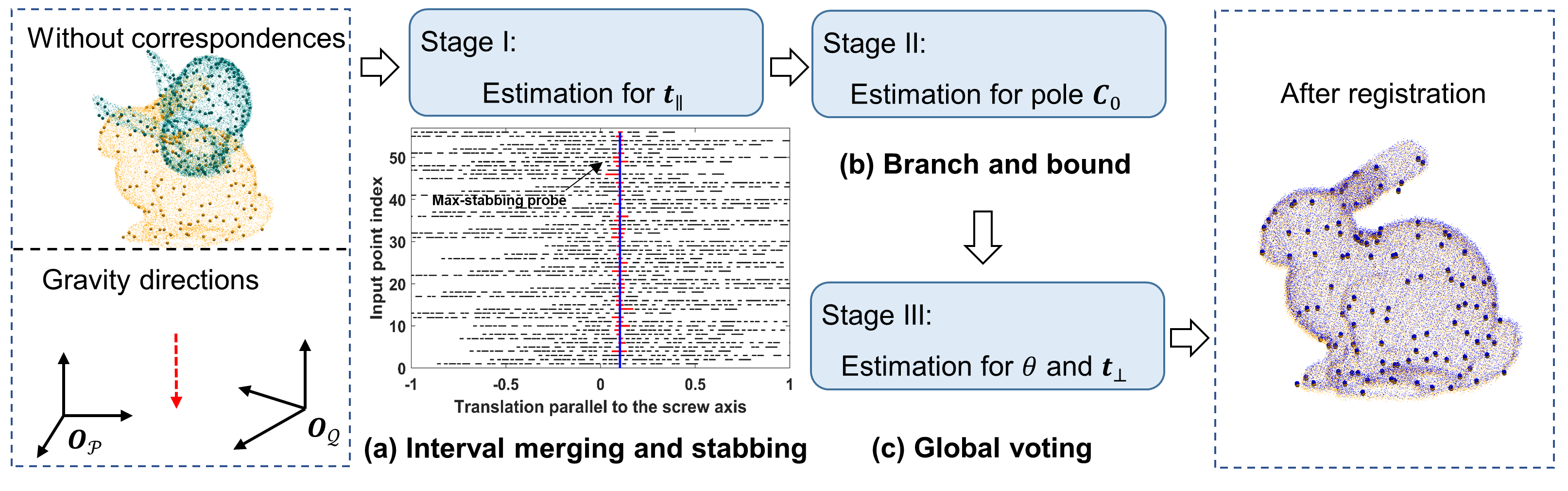}
 \caption{Calculation flow of the proposed three-stage method for correspondence-free registration. The bolded points in the point cloud represent the points after downsampling. In the first stage, the merged intervals are represented by the black line segments, while the intervals crossed by the max-stabbing probe are depicted as the red line segments.}
 \label{merge_and_stabbing}
\end{figure*}
In the last stage, we aim to solve the remaining 1-DOF rotation angle estimation sub-problem. According to Eq.~(\ref{2D_pure_rotation}), we can observe that the rotation angle is the angle between two vectors $(\hat{\bm{p}}_i-\bm{C}_0)$ and $(\hat{\bm{q}}_i-\bm{C}_0)$ in 2D plane $\bm{\Omega}_0$. Given optimal pole $\bm{C}_0^*$, the rotation angle for the $i$-th candidate inlier correspondence can be computed as 
\begin{equation}
    \theta_i = \angle(\hat{\bm{p}}_i-\bm{C}_0^*,\hat{\bm{q}}_i-\bm{C}_0^*).
    \label{theta definition}
\end{equation}
Inspired by the \textit{uniform grid approach}\cite{nesterov2018lectures}, we take $s$ equally spaced 1D grids on $[0,2\pi]$ whose centers are $\{\theta_k=(2k-1)\pi/s\}_{k=1}^s$ as the complete solution domain of rotation angle. Therefore, given the set of candidate inlier correspondences from problem (\ref{problem2}), the third sub-problem aims at estimating the optimal $\theta$ to maximize the cardinality of the inlier set, defined by
\begin{equation}
 \begin{aligned}
     \underset {\theta\in \{\theta_k\}_{k=1}^s,\mathcal{I}_3\subseteq\mathcal{I}_2^*} {\max} \ &\left|\mathcal{I}_3\right| \\
     \text{s.t.}\quad &\left|\theta - \angle(\hat{\bm{p}}_i-\bm{C}_0^*,\hat{\bm{q}}_i-\bm{C}_0^*)\right| \leq \zeta,\ \forall i\in\mathcal{I}_3  
     \label{problem3}
 \end{aligned}    
\end{equation}
where $\mathcal{I}_3$ is the inlier set extracted from the cardinality-maximized inlier set $\mathcal{I}_2^*$, and $\zeta\triangleq\pi/s$ is the inlier threshold. We can apply \textit{voting-based method}\cite{glent2014search,scaramuzza20111,yang2020teaser,yang2023mutual} to efficiently address this consensus maximization problem. For $N''\triangleq\left|\mathcal{I}_2^*\right|$ candidate inlier correspondences, we can vote $\theta^*$ with the largest consensus set $\mathcal{I}_3^*$ from $\{\theta_i\}_{i=1}^{N''}$. Due to its logical simplicity, we omit the global voting algorithm in this paper. Considering the trade-off between accuracy and efficiency, we set $s=360$ as an invariant parameter. An visualized illustration of the proposed global voting method is presented in Fig.~\ref{fig_2}. 

After obtaining optimal $\theta^*$, according to Eq.~(\ref{2d_equation}) and Eq.~(\ref{2D_pure_rotation}), we can derive that
\begin{equation}
    \hat{\bm{t}}_{\perp}^*=(\bm{I}-\hat{\bm{R}}_{\theta}^*)\bm{C}_0^*
\end{equation}
where $\bm{I}$ is an identity matrix. Then, according to Eq.~(\ref{20}) and Eq.~(\ref{21}), we have
\begin{equation}
    \bm{R}_{\theta}^*={
    \begin{bmatrix}
    \cos\theta^* & -\sin\theta^* &0\\
    \sin\theta^* & \cos\theta^* &0 \\
    0 &0 &1 \\
    \end{bmatrix}},\ 
    \tilde{\bm{t}}_{\perp}^* = {
    \begin{bmatrix}
    (\bm{I}-\hat{\bm{R}}_{\theta}^*)\bm{C}_0^* \\
    0 \\        
    \end{bmatrix}}
\end{equation}
Finally, according to Eq.~(\ref{to_screw_axis}) and Eq.~(\ref{new coordinates}), we can obtain the optimal rotation and translation by
\begin{align}
    \bm{R}^*&=(\bm{R}_{\bm{v}_q}^{\bm{e}_z})^{\mathrm{T}}\bm{R}_{\theta}^*\bm{R}_{\bm{v}_q}^{\bm{e}_z} \bm{R}_{\bm{v}_p}^{\bm{v}_q}\\
    \bm{t}^*&=\bm{t}_{\perp}^*+\bm{t}_{\parallel}^*=(\bm{R}_{\bm{v}_q}^{\bm{e}_z})^{\mathrm{T}}\tilde{\bm{t}}_{\perp}^*+\bm{t}_{\parallel}^*
\end{align}

\section{Simultaneous Pose and Correspondence Registration}
Since the performance limitation of current 3D feature matching methods, we have to face the situation that the correspondences are unknown in some practical applications, such as when the point clouds are not sampled densely from smooth surfaces\cite{7368945}. Therefore, we explore the feasibility and potential of our proposed approach in solving the more challenging SPCR problem in this section.

We assume that point clouds $\mathcal{P}=\{\bm{p}_i\}_{i=1}^{M}$ and $\mathcal{Q}=\{\bm{q}_j\}_{j=1}^{N}$ are the source and target point clouds, respectively. Following \cite{7381673,campbell2018globally}, and assuming known gravity directions, we have the following consensus maximization problem for correspondence-free registration
\begin{equation}
	\begin{aligned}
		\underset {\theta,\bm{C},\bm{t}_{\parallel},\mathcal{I}\subseteq\mathcal{S}} {\max} \ &\left|\mathcal{I}\right| \\
		\text{s.t.}\quad &\exists j \in\mathcal{K},\ \left\|\bm{S}(\bm{p}'_i)+\bm{t}_{\parallel}-\bm{q}_j\right\|\leq\epsilon,\ \forall i\in\mathcal{I},\label{SPCR}
	\end{aligned}
\end{equation}
where $\bm{p}'_i=\bm{R}_{\bm{v}_p}^{\bm{v}_q}\bm{p}_i$ and $\mathcal{S}=\left\{1,\dots,M\right\}$, $\mathcal{K}=\left\{1,\dots,N\right\}$ are the sets of indices for point clouds $\mathcal{P}$, $\mathcal{Q}$, respectively. According to the derivation in Section \ref{stage1}, we can obtain a new consensus maximization sub-problem for estimating the translation parallel to the screw axis without correspondences, which is defined by
\begin{equation}
	\begin{aligned}
		\underset {l,\mathcal{I}_1\subseteq\mathcal{S}} {\max} \ &\left|\mathcal{I}_1\right| \\
		\text{s.t.}\quad &\exists j \in\mathcal{K},\ l\in\left[ l_{ij}^-,l_{ij}^+ \right],\ \forall i\in\mathcal{I}_1,\\ 
		& l_{ij}^-=-\delta-\bm{v}_q^\mathrm{T}\left(\bm{p}'_i-\bm{q}_j\right),\\
		& l_{ij}^+=\delta-\bm{v}_q^\mathrm{T}\left(\bm{p}'_i-\bm{q}_j\right).\\\label{problem1_spcr}
	\end{aligned}
\end{equation}

For sub-problem (\ref{problem1_spcr}), we introduce an \textit{interval merging} algorithm\cite{de1997computational} (shown in appendix A) to effectively reduce the number of input intervals before interval stabbing. This operation can effectively avoid the all-to-all correspondence assumption\cite{yang2020teaser,liu2023absolute,li2023fast}, in which the assumed correspondence-based registration problem is extremely outlier-contaminated (with high input number and high outlier rate), thereby affecting the efficiency. After interval merging, the max-stabbing probe can only cross through at most one interval for each point $\bm{p}'_i$. From the perspective of inlier set cardinality, each point $\bm{p}'_i$ contributes a maximum of 1 to the cardinality of the inlier set. Therefore, the interval merging algorithm is executed for each point $\bm{p}'_i$, followed by the execution of the proposed interval stabbing algorithm for all merged intervals. The subplot of Fig.~\ref{merge_and_stabbing} illustrates an example of the visualization results about interval merging and stabbing.

After solving the correspondence-free sub-problem (\ref{problem1_spcr}), we can obtain candidate inlier correspondences that only satisfy the 1-DOF constraint. While it may not be possible to identify outliers that coincidentally satisfy this constraint, we can readily employ the proposed correspondence-based methods to tackle  the second and third sub-problems (Eq.~(\ref{problem2}) and Eq.~(\ref{problem3})), enabling the estimation of the remaining translation and the rotation angle. The three-stage calculation flow for the SPCR problem is shown in Fig.~\ref{merge_and_stabbing}.

\section{Experiments}

\begin{figure*}
 \centering
 \begin{tabular}{c c c}
    \multicolumn{3}{c}{
    \begin{minipage}{0.5\textwidth}
    \centering
    \raisebox{-.1\height}{\includegraphics[width=\linewidth]{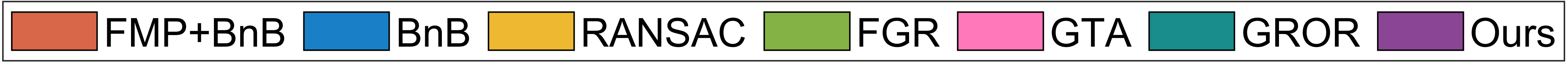}}
    \end{minipage} }
 \\ 
    \begin{minipage}[b]{0.3\textwidth}
    \centering
    \raisebox{-.1\height}{\includegraphics[width=\linewidth]{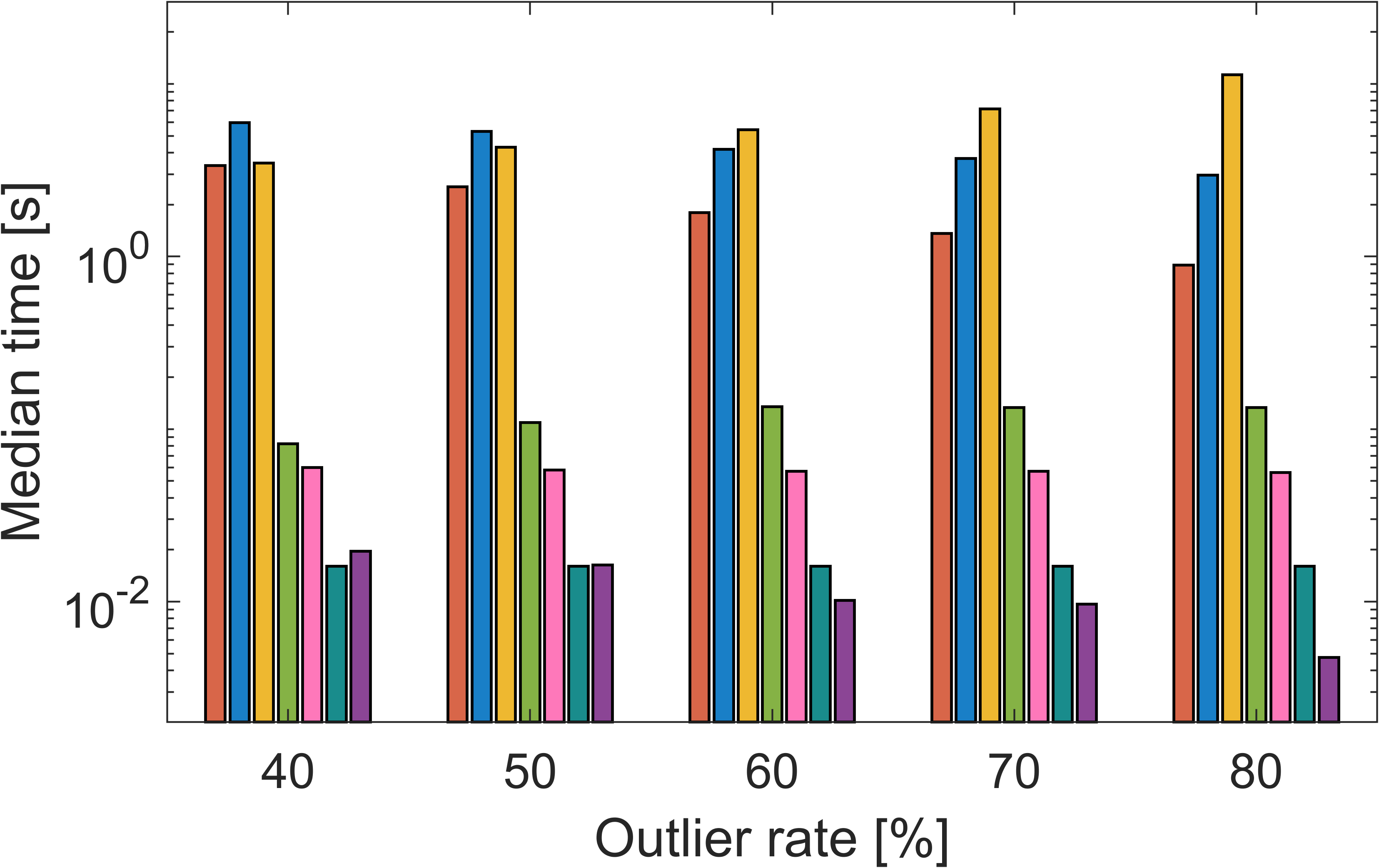}}
    \end{minipage}
  & \begin{minipage}[b]{0.3\textwidth}
    \centering
    \raisebox{-.1\height}{\includegraphics[width=\linewidth]{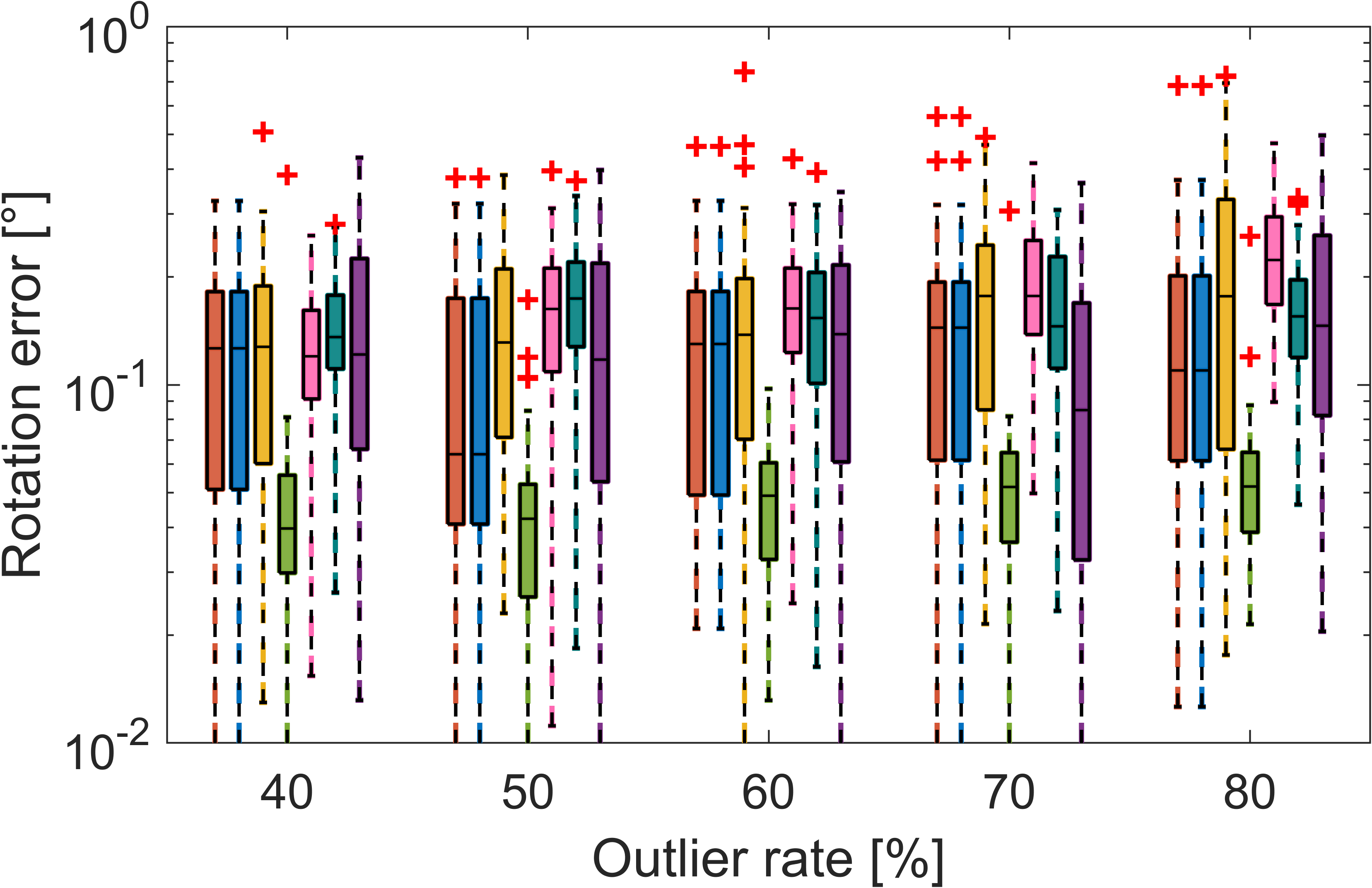}}
    \end{minipage}
  & \begin{minipage}[b]{0.3\textwidth}
    \centering
    \raisebox{-.1\height}{\includegraphics[width=\linewidth]{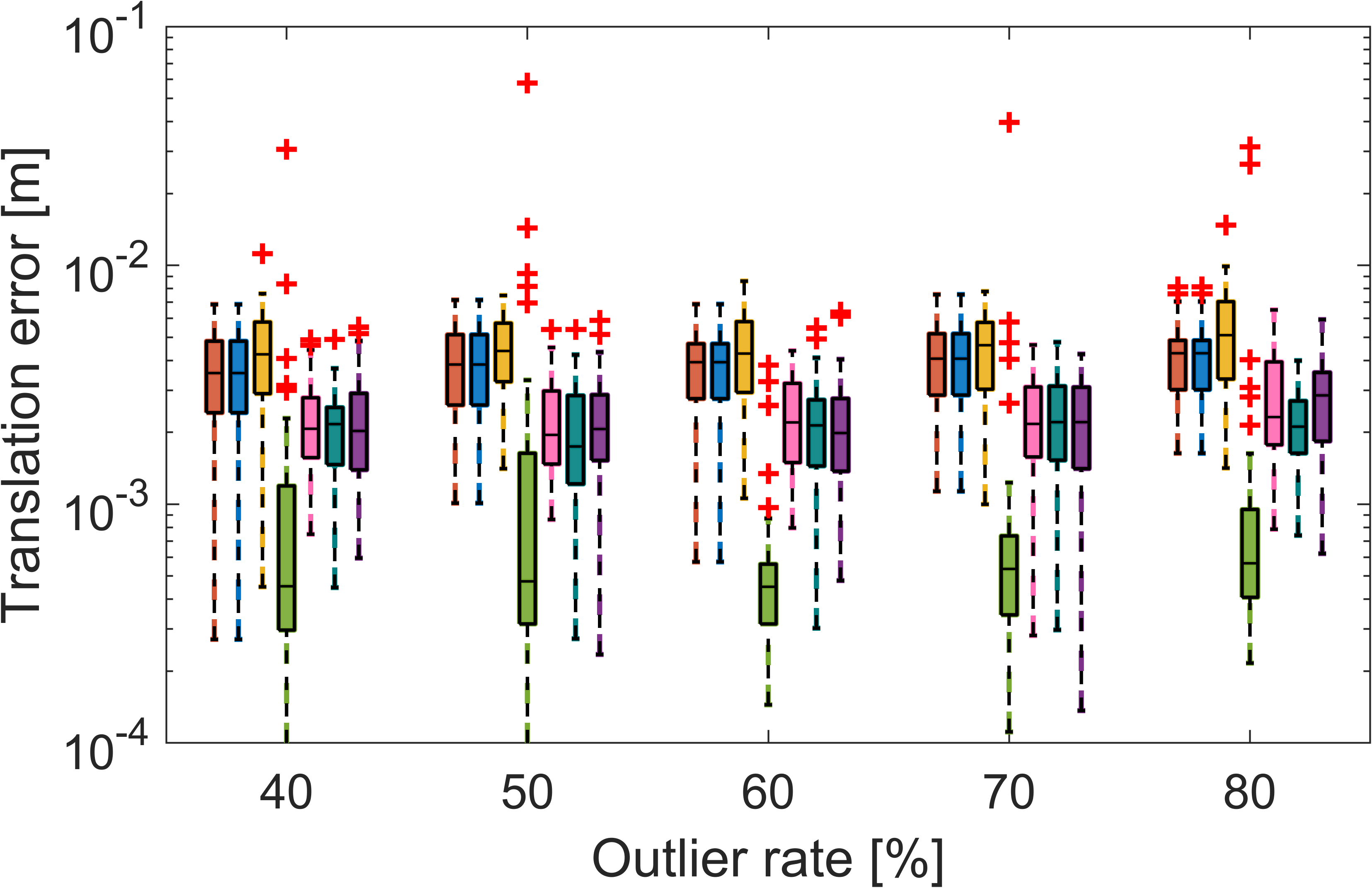}}
    \end{minipage}
 \\ 
    \begin{minipage}[b]{0.3\textwidth}
    \centering
    \raisebox{-.1\height}{\includegraphics[width=\linewidth]{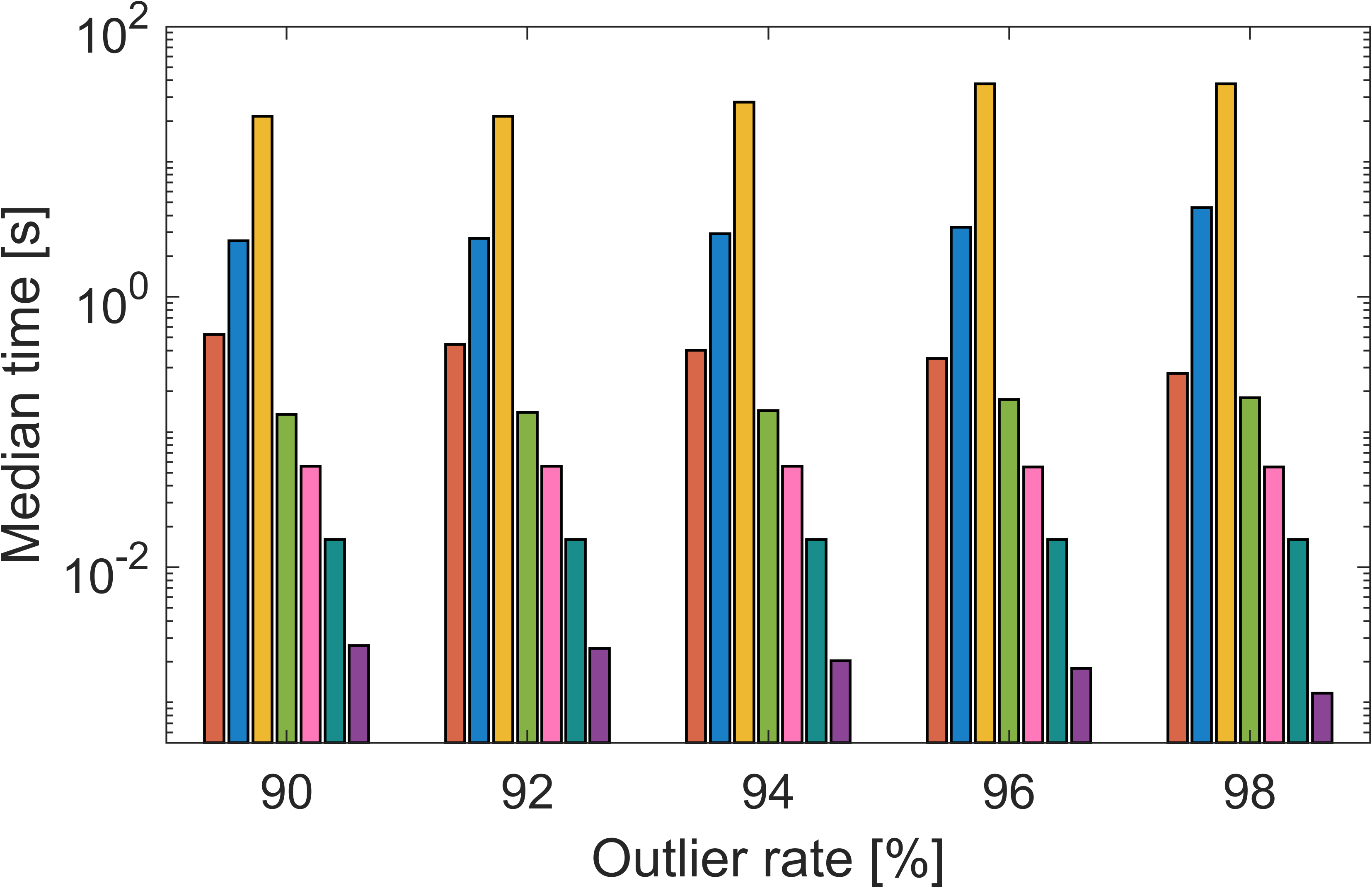}}
    \end{minipage}
  & \begin{minipage}[b]{0.3\textwidth}
    \centering
    \raisebox{-.1\height}{\includegraphics[width=\linewidth]{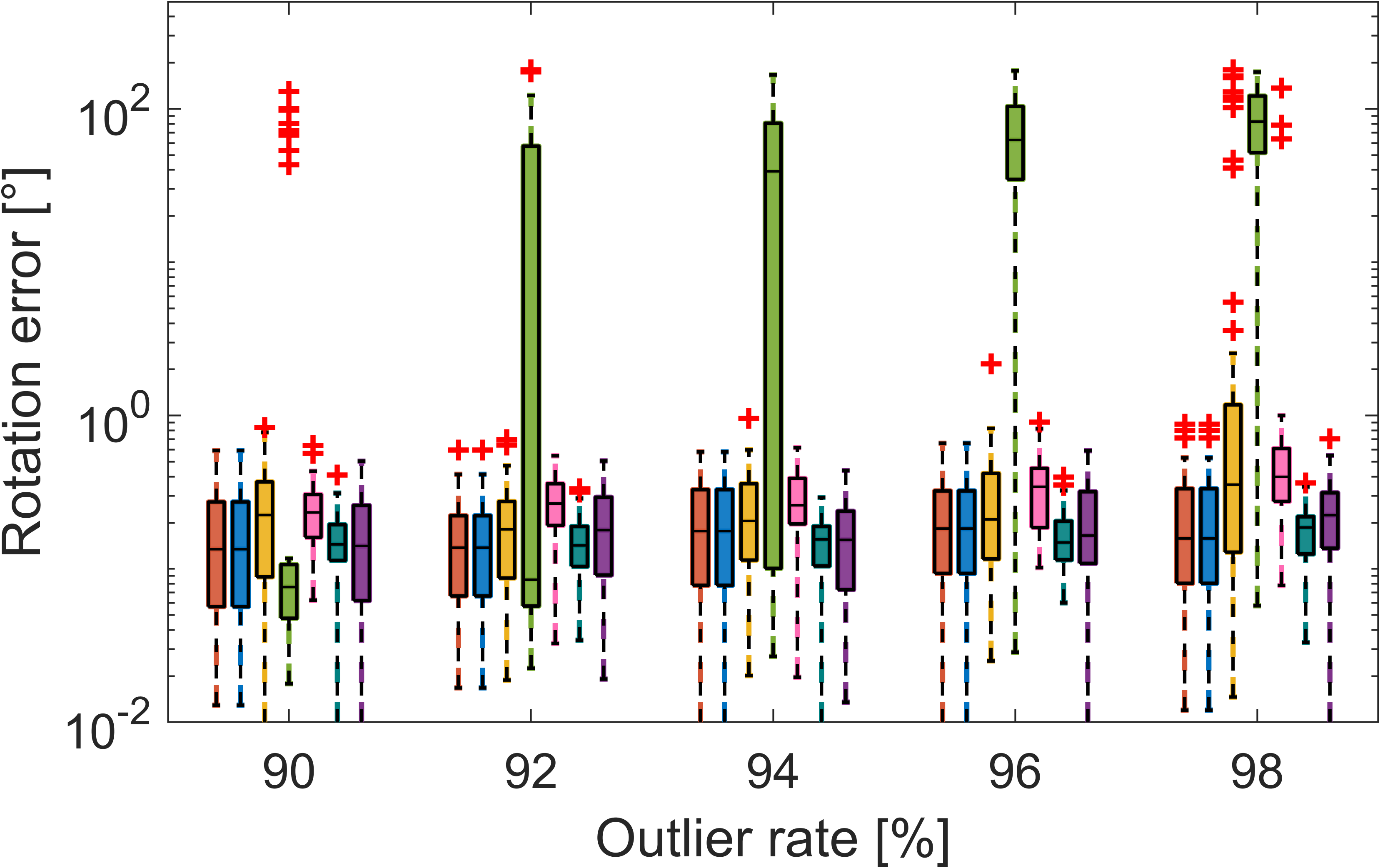}}
    \end{minipage}
  & \begin{minipage}[b]{0.3\textwidth}
    \centering
    \raisebox{-.1\height}{\includegraphics[width=\linewidth]{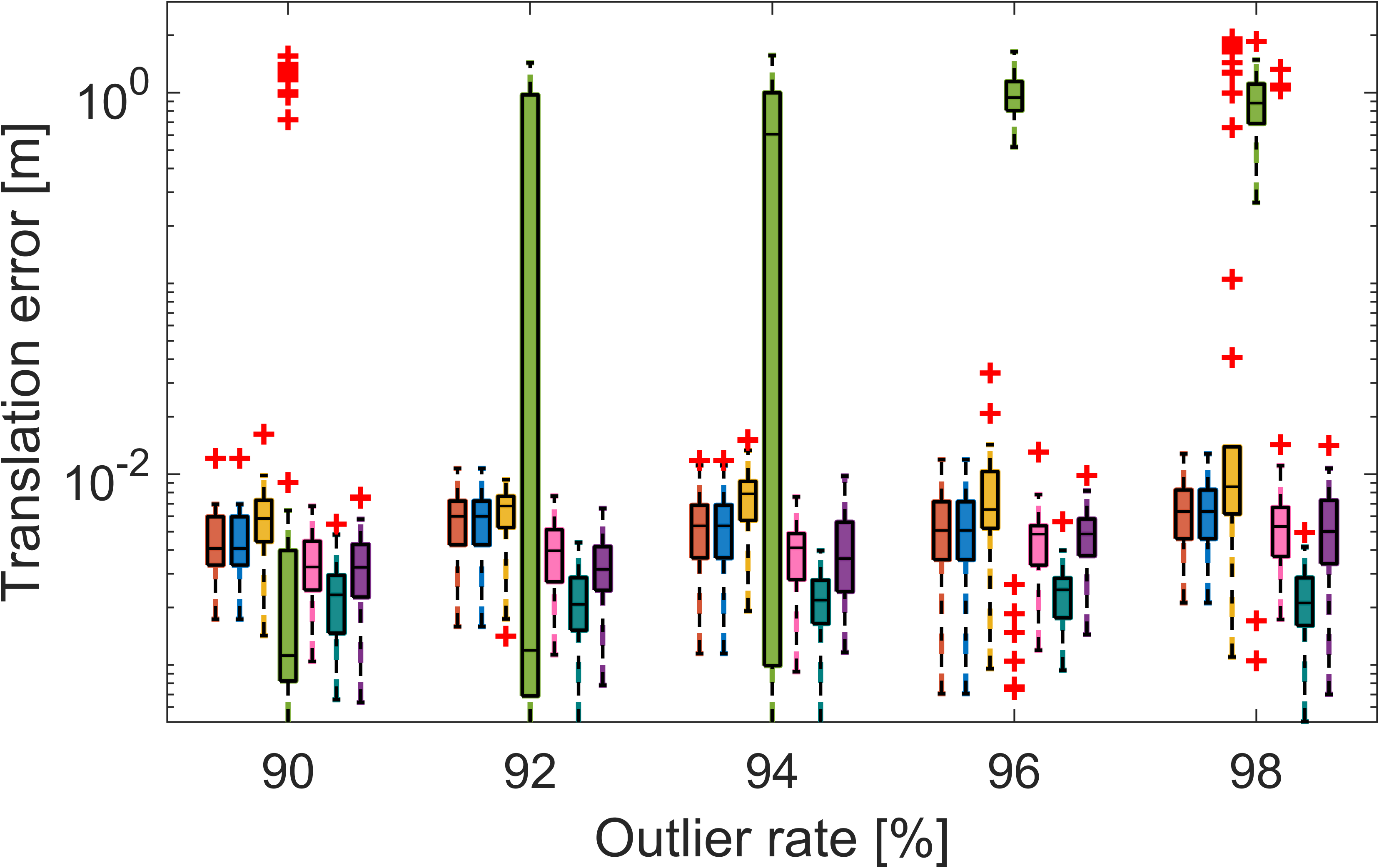}}
    \end{minipage}    
 \\ \footnotesize(a) Median running time 
  & \footnotesize(b) Rotation error 
  & \footnotesize(c) Translation error 
 \end{tabular}
 \caption{Controlled experiments on the outlier rate. (a) Median running time, (b) Rotation error, and (c) Translation error.}
 \label{synthetic_outlier}
\end{figure*}

\subsection{Experimental Setting}
In this section, we compare the performance of the proposed method with SOTA correspondence-based approaches by utilizing both synthetic and real-world data. Besides, we compare our extended SPCR method with different correspondence-free approaches. All experiments are implemented in a PC with an AMD 5600x CPU and 32GB RAM.

\subsubsection{Compared methods}
In the correspondences-based registration experiments, our method is compared with the 4-DOF as well as 6-DOF registration methods, including:
\begin{itemize}
    \item FMP+BnB\cite{cai2019practical}: A joint 4-DOF BnB method using fast match pruning (FMP) as the preprocessing step, programmed in C++.
    \item BnB\cite{cai2019practical}: A joint 4-DOF BnB method without fast match pruning.
    \item RANSAC\cite{fischler1981random}: A typical consensus maximization approach, which is customized for 4-DOF registration and programmed in C++. The maximum number of iterations is set to $10^7$. 
    \item FGR\cite{zhou2016fast}: A fast M-estimation method, which is customized for 4-DOF registration and programmed in C++. The annealing rate is set to $1.1$.
    \item GTA\cite{albarelli2010game}: A 6-DOF outlier removal method based on the game-theoretic framework, programmed in C++.
    \item GROR\cite{yan2022new}: A 6-DOF fast outlier removal method based on the reliability of the correspondence graph, programmed in C++.
\end{itemize}

In addition, our extended method is compared with the following global and local methods for the SPCR problem:
\begin{itemize}
    \item GO-ICP\cite{7368945}: A globally optimal approach that combines global BnB and local ICP, programmed in C++.
    \item GO-ICPT\cite{7368945}: A variant of GO-ICP with a specified trimming percentage for outlier removal.
    \item ICP\cite{icp}: A classic rigid registration algorithm that finds the optimal transformation by solving the least squares problem during each iteration, implemented in MATLAB Toolbox\footnote{\url{https://www.mathworks.com/help/vision/ref/pcregistericp.html}}.
    \item CPD\cite{CPD}: A probabilistic method that transfers the registration as a probability density estimation problem, programmed in C language.
    \item Gmmreg\cite{jian}: A probabilistic registration method that represents the input point clouds as GMMs and aligns them, programmed in C language.
\end{itemize}
In both correspondence-based and SPCR experiments, the proposed method is consistently referred to as \textbf{Ours}, which is implemented in Matlab 2022b. Code will be available at \url{https://github.com/Xinyi-tum/4DOF-Registration}.

\subsubsection{Evaluation metrics}
Following \cite{cai2019practical,yan2022new}, we employ running time $T$, rotation error $RE$, translation error $TE$, and success rate $SR$ to evaluate the registration performance. The calculations of $RE$ and $TE$ are shown as below:
\begin{align}
    RE &= \arccos \left(\frac{\mathrm{Tr}(\bm{R}_{gt}^{\mathrm{T}}\bm{R}^*)-1}{2}\right) \\
    TE &= \left\|\bm{t}_{gt} - \bm{t}^*\right\|        
\end{align}
where $\bm{R}_{gt}$ and $\bm{t}_{gt}$ are ground truth rotation and translation, $\bm{R}^*$ and $\bm{t}^*$ are estimated results, and $\mathrm{Tr}(\cdot)$ is the trace of a matrix. The point clouds are successfully aligned when $RE$ and $TE$ are within the predefined thresholds.

\subsection{Synthetic Data Experiments}
\begin{figure*}
 \centering
 \begin{tabular}{c c c}
    \multicolumn{3}{c}{
    \begin{minipage}{0.5\textwidth}
    \centering
    \raisebox{-.1\height}{\includegraphics[width=\linewidth]{Pics/legend.png}}
    \end{minipage} } \\
    \begin{minipage}[b]{0.3\textwidth}
    \centering
    \raisebox{-.1\height}{\includegraphics[width=\linewidth]{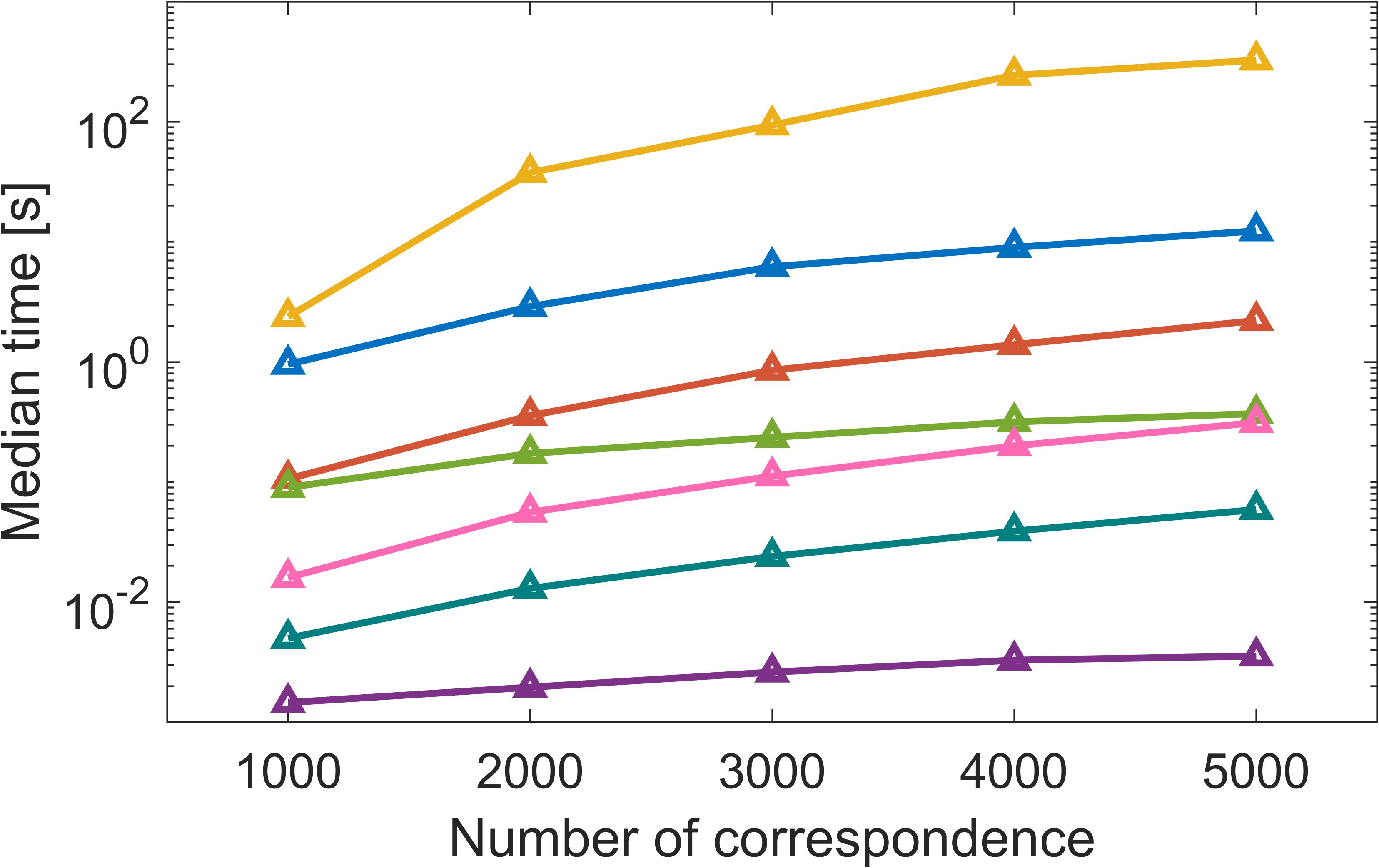}}
    \end{minipage}
  & \begin{minipage}[b]{0.3\textwidth}
    \centering
    \raisebox{-.1\height}{\includegraphics[width=\linewidth]{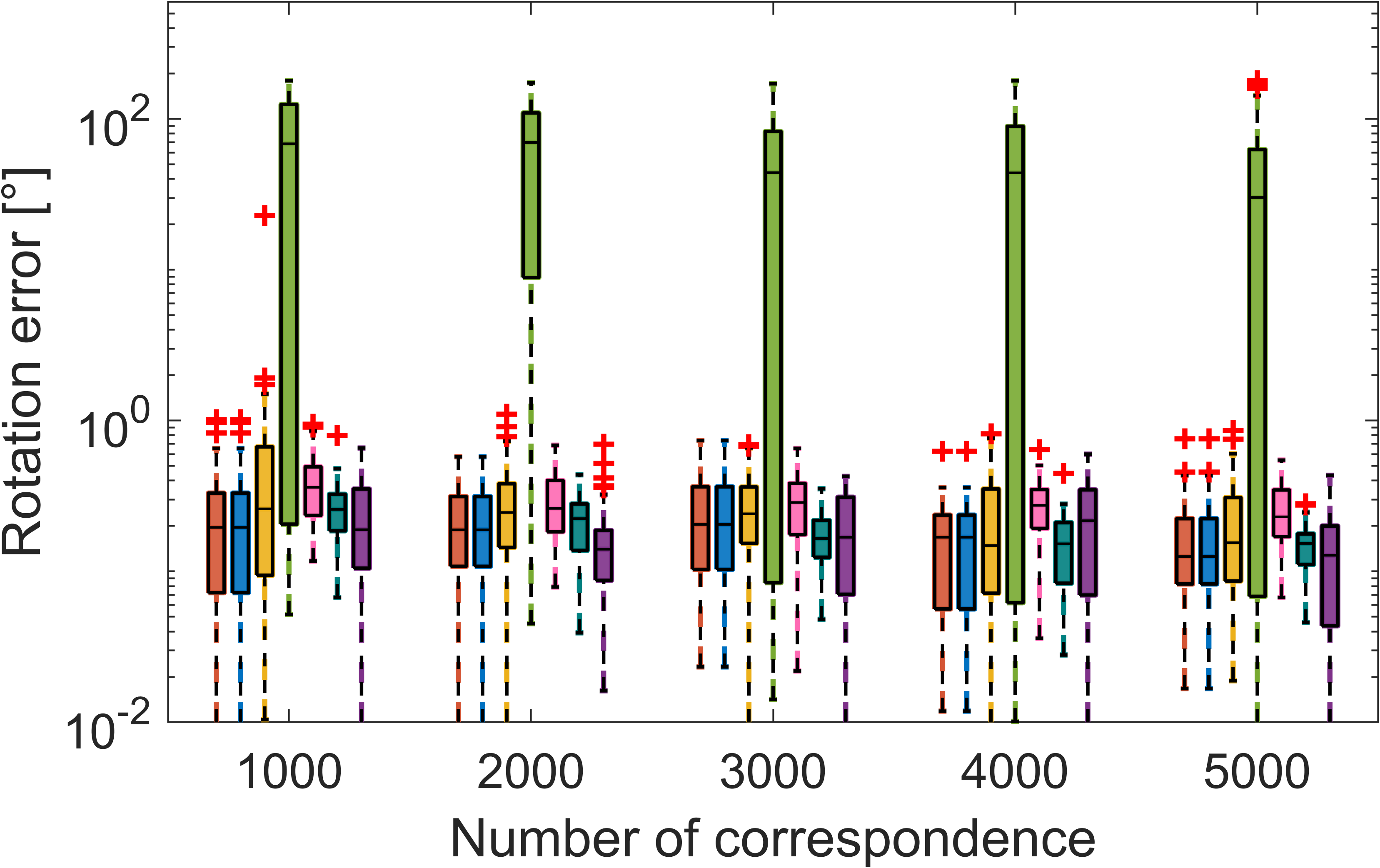}}
    \end{minipage}
  & \begin{minipage}[b]{0.3\textwidth}
    \centering
    \raisebox{-.1\height}{\includegraphics[width=\linewidth]{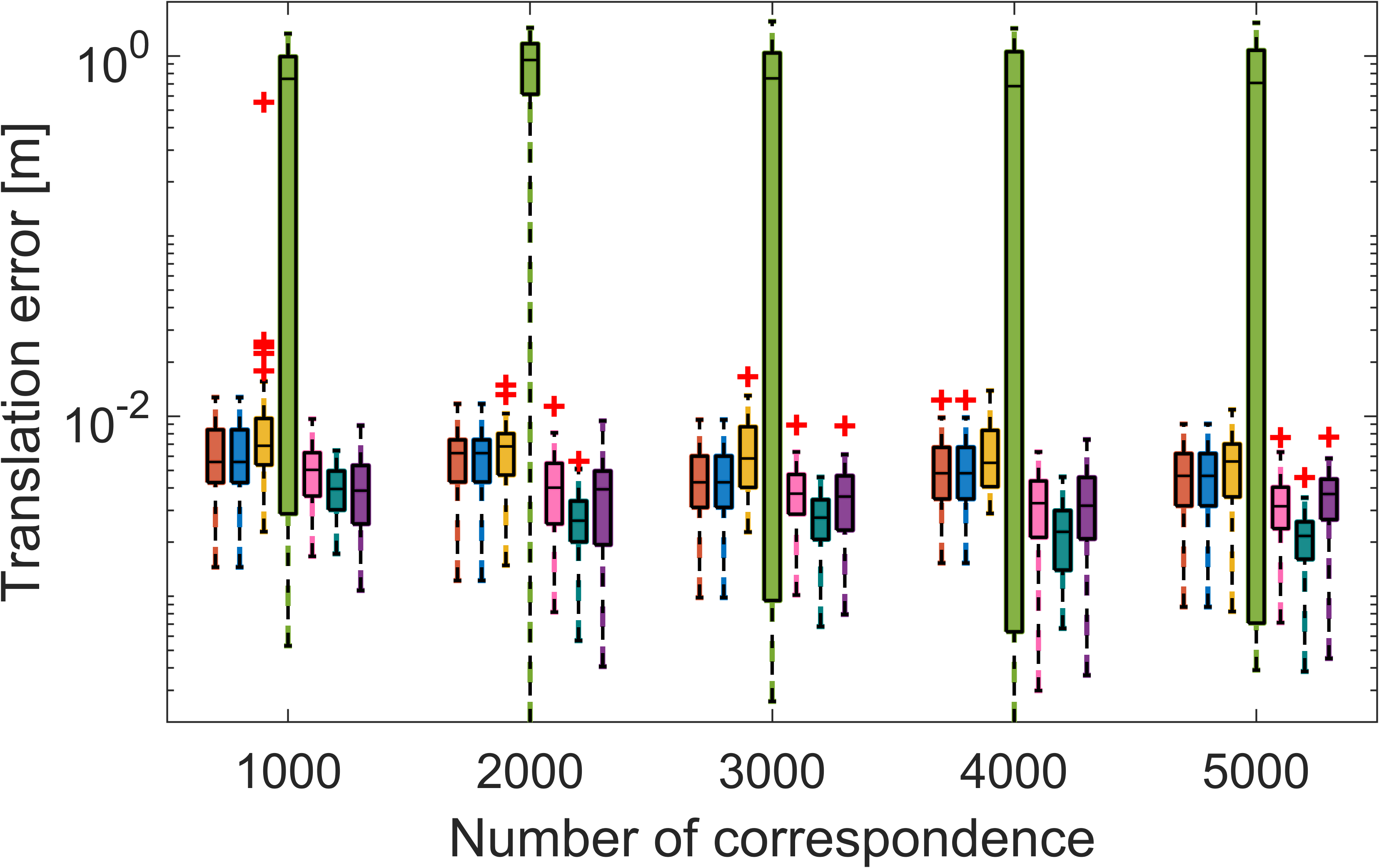}}
    \end{minipage}
 \\ \footnotesize(a) Median running time
  & \footnotesize(b) Rotation error
  & \footnotesize(c) Translation error
 \end{tabular}
 \caption{Controlled experiments on the number of correspondences. (a) Median running time, (b) Rotation error, and (c) Translation error.}
 \label{synthetic_num_point}
\end{figure*}


\begin{table*}\footnotesize
 \setlength{\tabcolsep}{7pt}
 \renewcommand{\arraystretch}{1.2}
 \caption{Controlled experiments on the number of correspondences in the extreme case. The results include median running time (s) $|$ success rate (\%) with cases satisfying $ RE \leq 1\degree$ and $TE \leq 0.01\si{\metre}$.}
 \label{high_corr}
 \centering
 \begin{tabular}{c r@{\hspace{0pt}} l r@{\hspace{0pt}} l r@{\hspace{0pt}} l r@{\hspace{0pt}} l r@{\hspace{0pt}} l r@{\hspace{0pt}} l r@{\hspace{0pt}} l}
    \hline
    \multirow{2}*{Method} 
    & \multicolumn{14}{c}{Number of correspondence ($\times10^3$)}\\
    \cline{2-15}
    {} & \multicolumn{2}{c}{10} & \multicolumn{2}{c}{20} & \multicolumn{2}{c}{50} & \multicolumn{2}{c}{100} & \multicolumn{2}{c}{200} & \multicolumn{2}{c}{500} & \multicolumn{2}{c}{1000}\\
    \hline 
    FMP+BnB &$7.375|$ &$100$ &$26.73|$ &$100$ &$165.9|$ &$100$ &$647.7|$ &$100$ &\multicolumn{2}{c}{$>1800$s} &\multicolumn{2}{c}{-} &\multicolumn{2}{c}{-}\\
    BnB &$32.20|$ &$100$ &$83.85|$ &$100$ &$374.8|$ &$100$ &$1169|$ &$100$ &\multicolumn{2}{c}{$>1800$s} &\multicolumn{2}{c}{-} &\multicolumn{2}{c}{-}\\
    RANSAC &$1474|$ &$94.0$ &\multicolumn{2}{c}{$>1800$s} &\multicolumn{2}{c}{-} &\multicolumn{2}{c}{-} &\multicolumn{2}{c}{-} &\multicolumn{2}{c}{-} &\multicolumn{2}{c}{-}\\
    FGR &$0.725|$ &$50.0$ &$1.875|$ &$56.0$ &$5.613|$ &$48.0$ &$7.540|$ &$72.0$ &$20.66|$ &$68.0$ &$44.94|$ &$62.0$ &$100.5|$ &$62.0$ \\
    GTA &$1.287|$ &$96.0$ &$4.232|$ &$0.00$ &\multicolumn{2}{c}{Out of memory} &\multicolumn{2}{c}{-} &\multicolumn{2}{c}{-} &\multicolumn{2}{c}{-} &\multicolumn{2}{c}{-}\\
    GROR &$0.217|$ &$100$ &$0.835|$ &$100$ &$4.866|$ &$100$ &$19.28|$ &$100$ &$77.04|$ &$100$ &$494.3|$ &$100$ &\multicolumn{2}{c}{$>1800$s} \\
    Ours &$0.009|$ &$100$ &$0.019|$ &$100$ &$0.039|$ &$100$ &$0.097|$ &$100$ &$0.241|$ &$100$ &$1.064|$ &$100$ &$1.905|$ &$100$ \\
    \hline 
 \end{tabular}
\end{table*}

\begin{figure*}
 \centering
 \begin{tabular}{c c c}
    \multicolumn{3}{c}{
    \begin{minipage}{0.6\textwidth}
    \centering
    \raisebox{-.1\height}{\includegraphics[width=\linewidth]{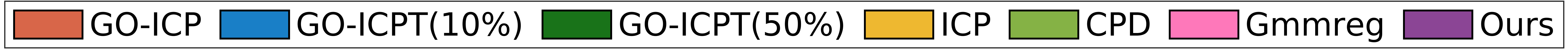}}
    \end{minipage} } \\ 
    \begin{minipage}[b]{0.3\textwidth}
    \centering
    \raisebox{-.1\height}{\includegraphics[width=\linewidth]{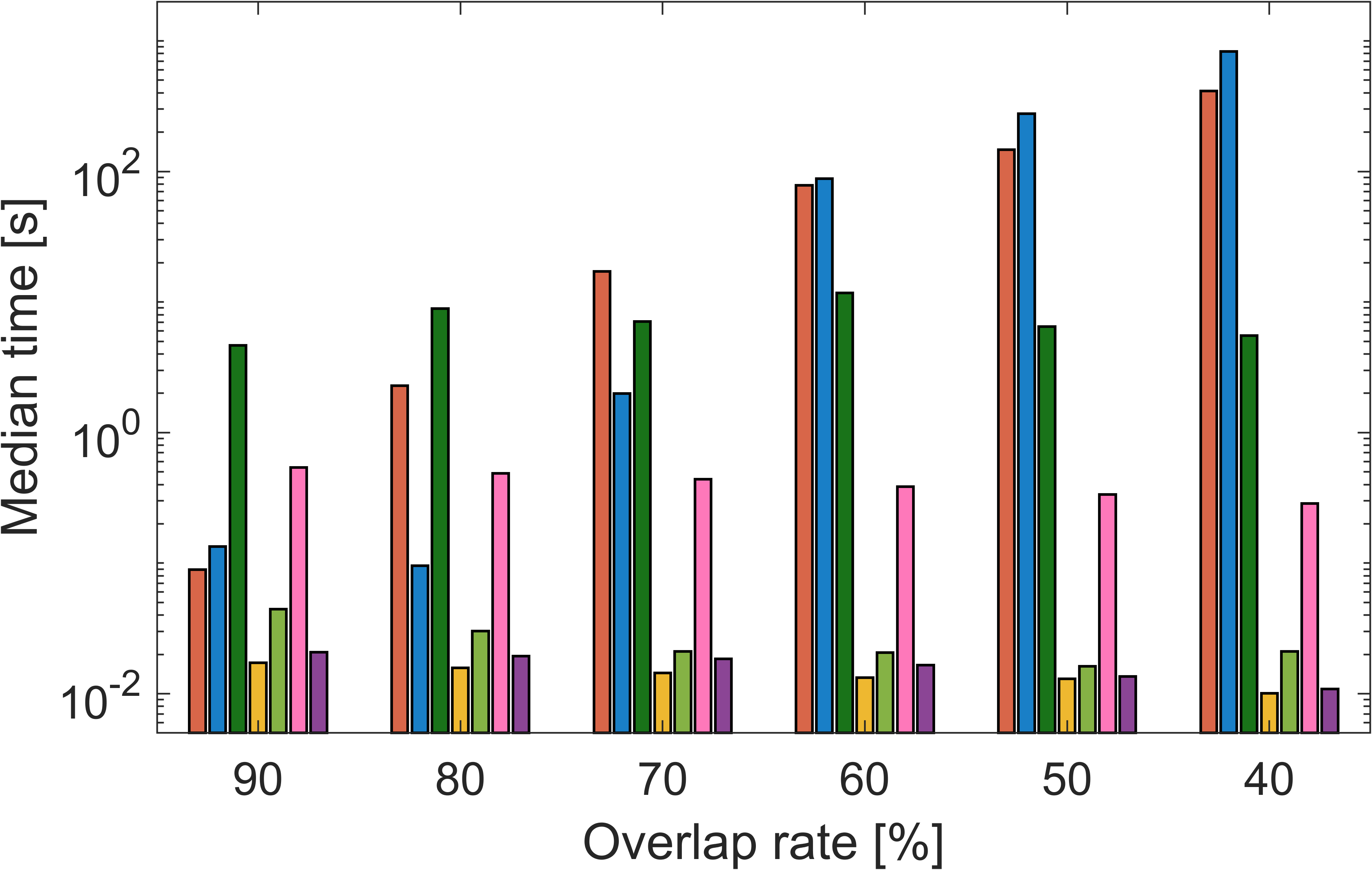}}
    \end{minipage}
  &  \begin{minipage}[b]{0.3\textwidth}
    \centering
    \raisebox{-.1\height}{\includegraphics[width=\linewidth]{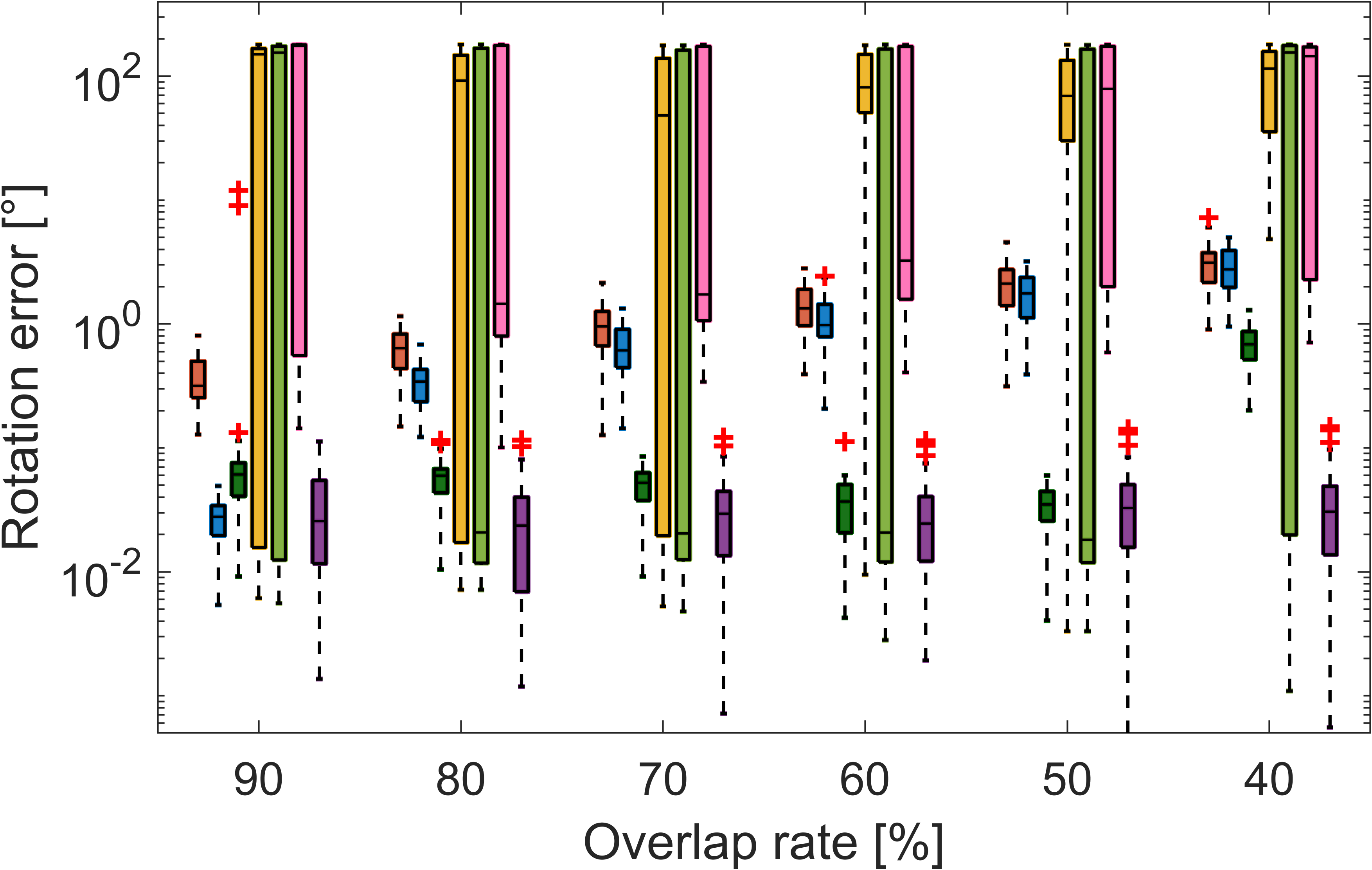}}
    \end{minipage}
  & \begin{minipage}[b]{0.3\textwidth}
    \centering
    \raisebox{-.1\height}{\includegraphics[width=\linewidth]{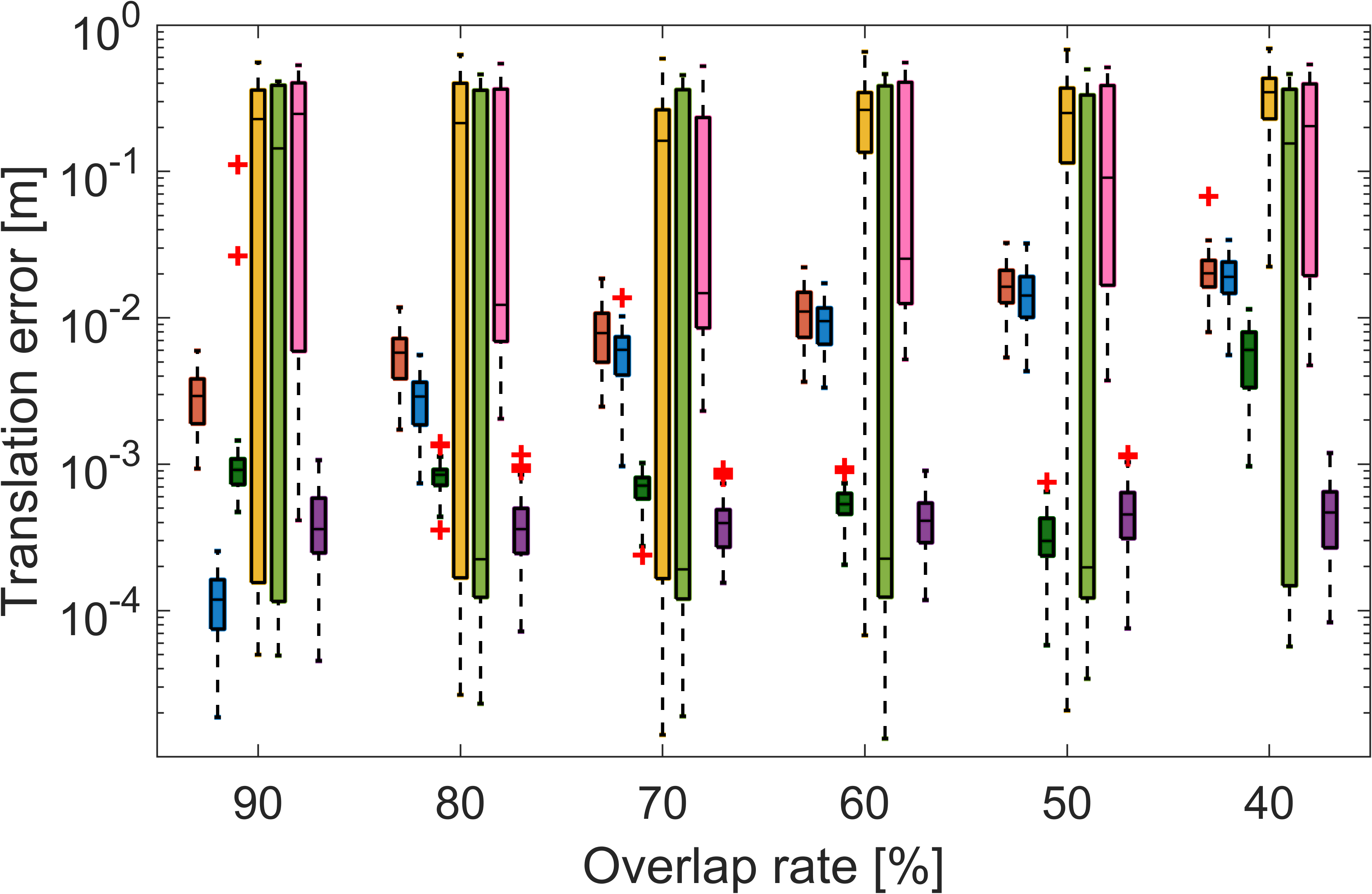}}
    \end{minipage}
 \\ \footnotesize(a) Median running time
  & \footnotesize(b) Rotation error
  & \footnotesize(c) Translation error
 \end{tabular}
 \caption{Controlled SPCR experiments on the overlap rate. (a) Median running time, (b) Rotation error, and (c) Translation error.}
 \label{synthetic_corr_free}
\end{figure*}

To verify the theoretical performance of the proposed method, we first conduct experiments using synthetic data. The source point cloud is randomly generated by creating \(N\) points distributed in the cube $[-1,1]^3$. Then the source point cloud is transformed by a random rotation, whose angle is within \([-\pi,\pi]\), and a random translation in \([-1,1]^3\) to generate the target point cloud. Due to the limitations of FMP+BnB\cite{cai2019practical}, the rotation axis is fixed as $[0,0,1]^{\mathrm{T}}$. Therefore, the gravity direction utilized in the proposed method can be set as $[0,0,-1]^T$. Subsequently, a subset of points in the target point cloud is substituted with randomly generated points in \([-1,1]^3\), imitating outliers. The outlier rate, denoted by $\eta$, represents the proportion of these substituted points relative to the total number of points. The noise is simulated by adding zero mean Gaussian noise to both source and target point clouds and the standard derivation is $\sigma=0.005$. Following \cite{yan2022new,10091912}, the inlier threshold for each method is determined by $\sigma$. Notably, the experiment on the robustness of the proposed method to gravity direction noise is given in Appendix B.

\subsubsection{Controlled experiments with different outlier rates}

This section presents two groups of controlled experiments designed to compare the outlier-robustness of the proposed method with FMP+BnB, BnB, RANSAC, FGR, GTA, and GROR. The first group of experiments has normal outlier rates, varying from $40\%$ to $80\%$ in increments of $10\%$. The second group of experiments has extremely high outlier rates, varying from $90\%$ to $98\%$ in increments of $2\%$. The correspondence number in both experiments is fixed as $N=2000$. The experiment is repeated $50$ times for each setting and each method. The median running time, rotation error, and translation error are given in Fig.~\ref{synthetic_outlier}.

In the first group of experiments, all compared methods are robust against up to $80\%$ outlier rate. In terms of rotation and translation accuracy, the proposed method demonstrates comparable performance to the other methods. However, Ours has the lowest running time except at $40\%$ outlier rate. Notably, the running time of our method gradually decreases as the outlier rate increases. This phenomenon can be attributed to the reason that in the first stage, a significant number of outlier correspondences are rejected. Because the BnB-based second stage is relatively time-consuming (worst-case exponential time complexity) among the three stages. Consequently, as the outlier rate increases, the input correspondences for the second stage gradually decrease and so does the total running time. This indicates the effectiveness of the proposed first stage in terms of outlier removal. FMP+BnB has a similar phenomenon as Ours and is more efficient than pure BnB since it also contains an efficient preprocessing step for outlier removal. Nonetheless, the SOTA FMP+BnB is approximately two orders of magnitude slower than Ours. The running time of RANSAC increases exponentially with the outlier rate, making it the slowest method starting from $60\%$ outlier rate. FGR and GTA perform well under regular outlier rates. Another SOTA GROR is relatively efficient, but is about $3$ times slower than Ours at $80\%$ outlier rate. 

In the second group of experiments, only FMP+BnB, BnB, GROR, and Ours are robust against up to $98\%$ outlier rate, while maintaining comparable registration accuracy. RANSAC only resists $94\%$ outlier rate and FGR starts breaking at $90\%$ outlier rate. Due to its high number of iterations ($10^7$), RANSAC is more robust than FGR, but is the most time-consuming method. GTA performs relatively better, but is only robust to $96\%$ outlier rate. On the other hand, Ours outperforms all methods in terms of efficiency. For instance, Ours is approximately $3\times 10^4$ times faster than RANSAC, $200$ times faster than FMP+BnB, and $10$ times faster than GROR at $98\%$ outlier rate. In general, Ours stands out as the fastest method and one of the most robust methods.

\subsubsection{Controlled experiments with different correspondences numbers}

This section presents two groups of controlled experiments aimed at comparing the registration efficiency. The correspondence number in the first group of experiment varies from $1000$ to $5000$. The outlier rate is fixed as $95\%$, and each experiment is repeated $50$ times for each setting. The experimental results are plotted in Fig.~\ref{synthetic_num_point}. Compared to other methods, the efficiency of the proposed method is less affected by the correspondence number. This implies that, with an increase in the number of correspondence, the efficiency advantage of Ours will become more prominent. Similarly, the running time of FGR exhibits a slow increase as the number of correspondences grows, resulting in only a marginal time difference between GTA and FGR at $N=5000$. However, FGR often produces unsatisfactory registration results when outlier rates are high. RANSAC is the method most sensitive to the number of correspondences and consistently exhibits the longest running time among all the methods. Particularly, as the correspondences number increases from $1000$ to $5000$, the efficiency of Ours increases from being roughly $3$ times faster than the second fastest method, GROR, to being $16$ times faster. In the case of FMP+BnB, this improvement ranges from being $70$ times faster to $613$ times faster. 

To further explore the scalability of Ours, the second group of experiments has extremely high correspondence numbers, varying from $10k$ to $1000k$, where $k$ represents $10^3$. This experiment also has 50 trials for each setting, and the outlier rates remain at $95\%$. The results, including median running times and success rates, are presented in Table~\ref{high_corr}. The thresholds for successful registration are $RE\leq1\degree$ and $TE\leq0.01\si{m}$. Notably, if the running time of a method exceeds $1800$s, the results for that method are not reported. As can be seen from the results of success rate, only FMP+BnB, BnB, GROR, and Ours can successfully register all point cloud pairs even with a large scale of correspondences. However, both FMP+BnB and BnB surpass $1800$s when the number of points reaches $200k$. Although FGR demonstrates better efficiency, it exhibits low success rate. As $N$ increases to $50k$, GTA is unable to even operate properly due to the memory problem. Furthermore, the running time of GROR grows faster than that of FGR as the correspondence number increases. Therefore, GROR becomes slower than FGR starting from $N=100k$ and it exceeds the time limit of $1800$s at $N=1000k$. However, thanks to the proposed transformation strategy, the results indicate that the proposed method has the best scalability among all methods. It can successfully align $100k$ correspondences in $0.1$s and register $1000k$ correspondences in $2$s. 

\begin{figure}
 \centering
 \begin{tabular}{c c c c c c}
    \multicolumn{3}{c}{
    \begin{minipage}[b]{0.4\columnwidth}
    \centering
    \raisebox{-.1\height}{\includegraphics[width=\linewidth]{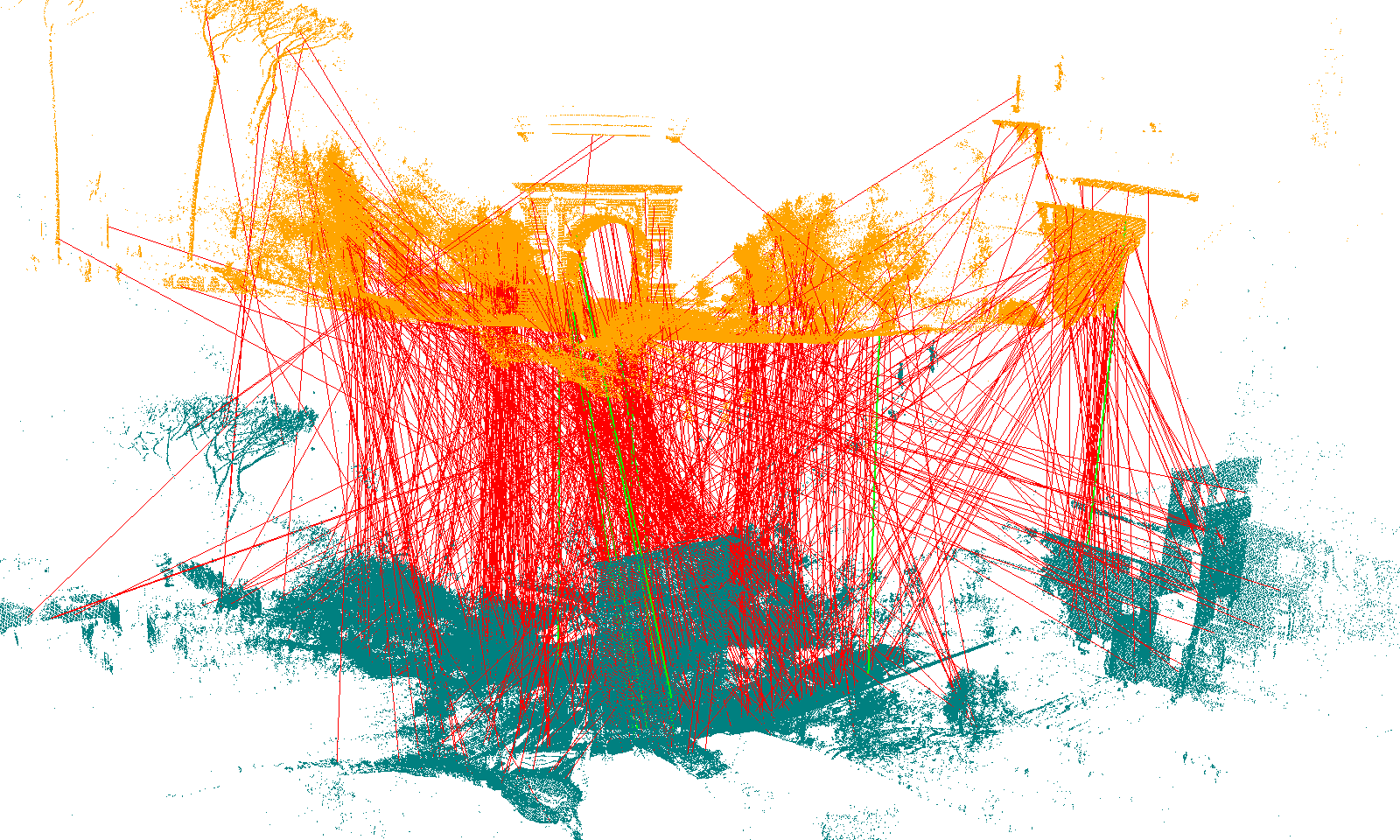}}
    \end{minipage}    }
 &  \multicolumn{3}{c}{
    \begin{minipage}[b]{0.4\columnwidth}
    \centering
    \raisebox{-.1\height}{\includegraphics[width=\linewidth]{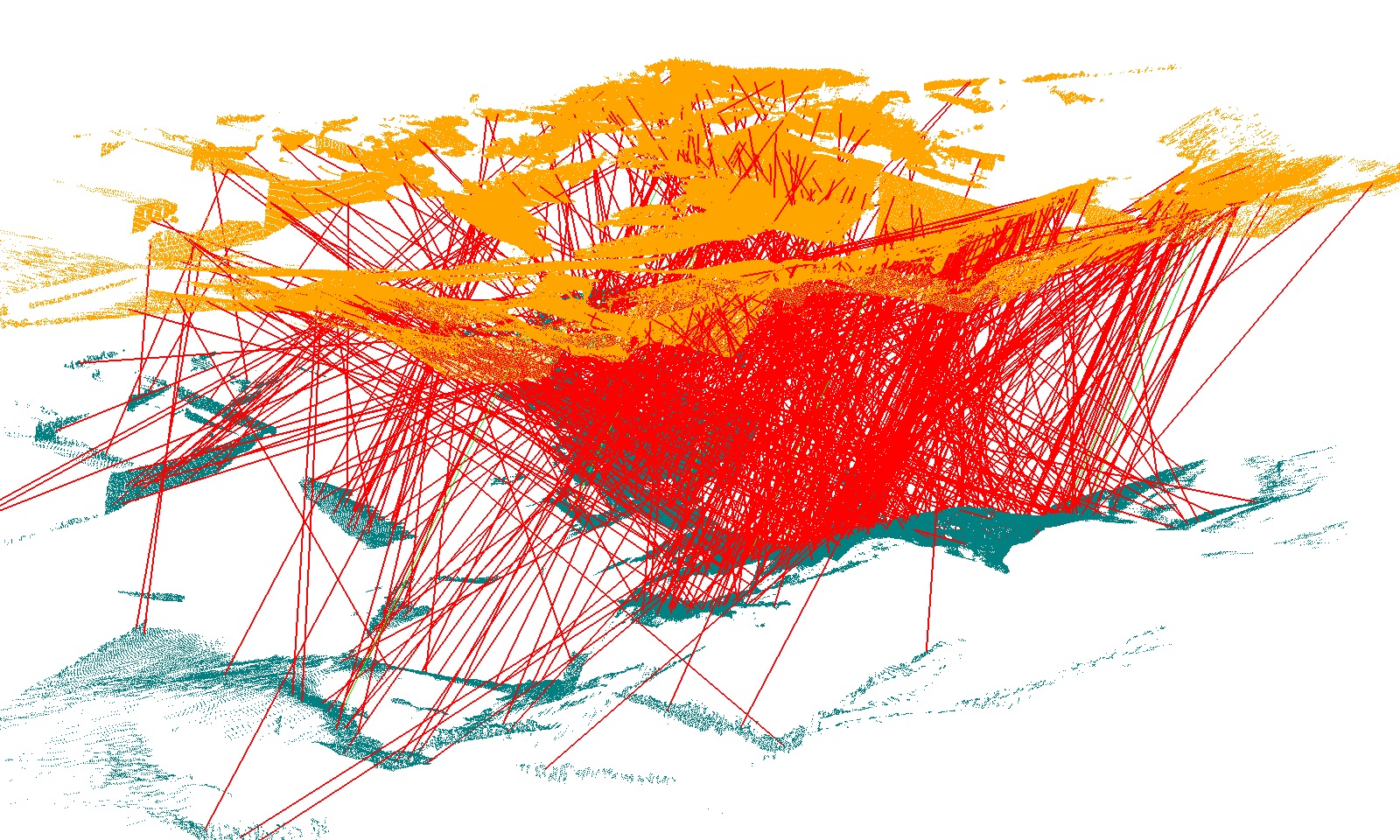}}
    \end{minipage}    }
 \\ \multicolumn{3}{c}{\footnotesize(a) \textit{Arch} s1-s2}
 &  \multicolumn{3}{c}{\footnotesize(b) \textit{Courtyard} s1-s2}
 \\ \multicolumn{3}{c}{\footnotesize$N = 12617$}
 &  \multicolumn{3}{c}{\footnotesize$N = 23572$}
 \\ \multicolumn{3}{c}{\footnotesize$\eta = 98.45\%$}
 &  \multicolumn{3}{c}{\footnotesize$\eta = 93.69\%$}
 \\ \multicolumn{2}{c}{
    \begin{minipage}[b]{0.31\columnwidth}
    \centering
    \raisebox{-.1\height}{\includegraphics[width=\linewidth]{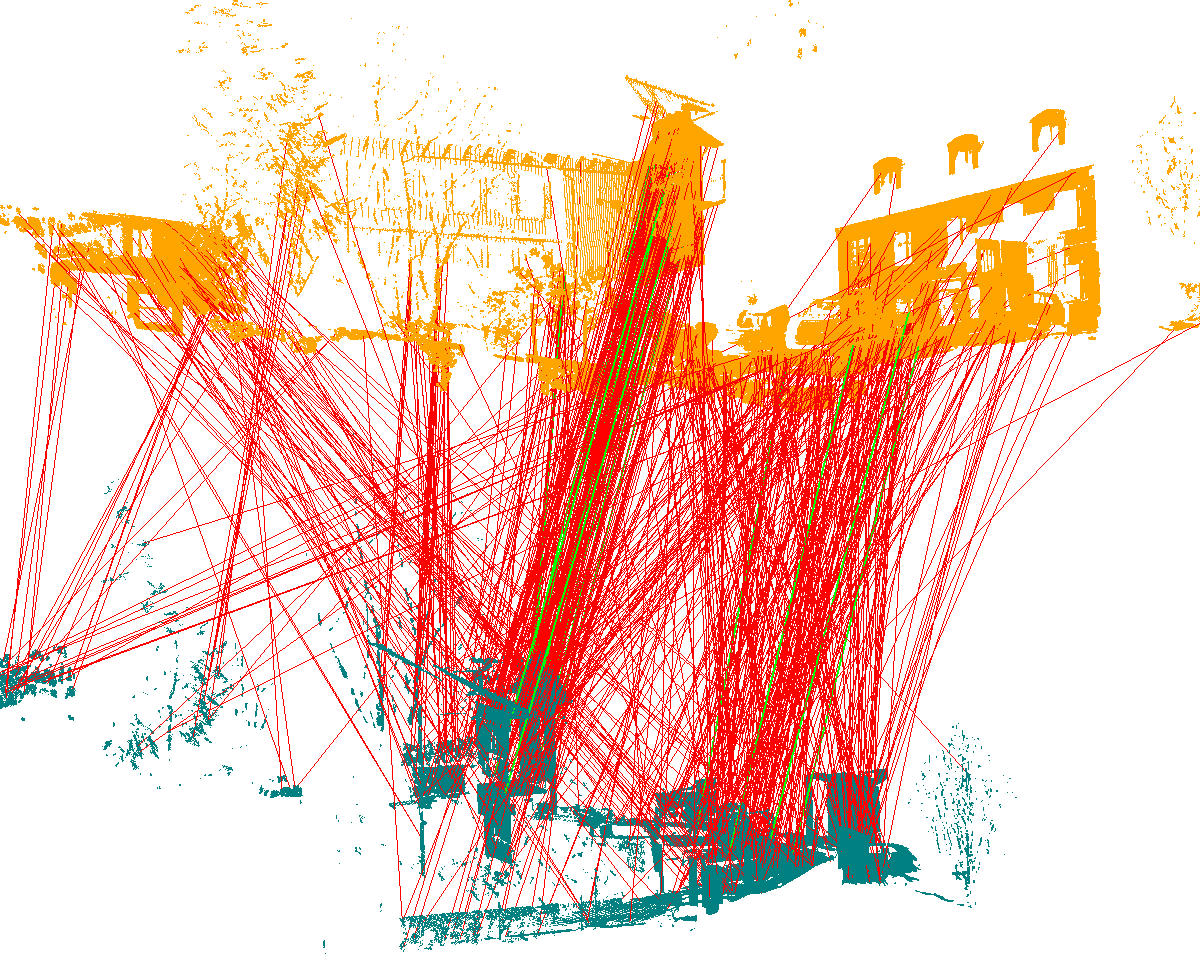}}
    \end{minipage}    }
 &  \multicolumn{2}{c}{
    \begin{minipage}[b]{0.31\columnwidth}
    \centering
    \raisebox{-.1\height}{\includegraphics[width=\linewidth]{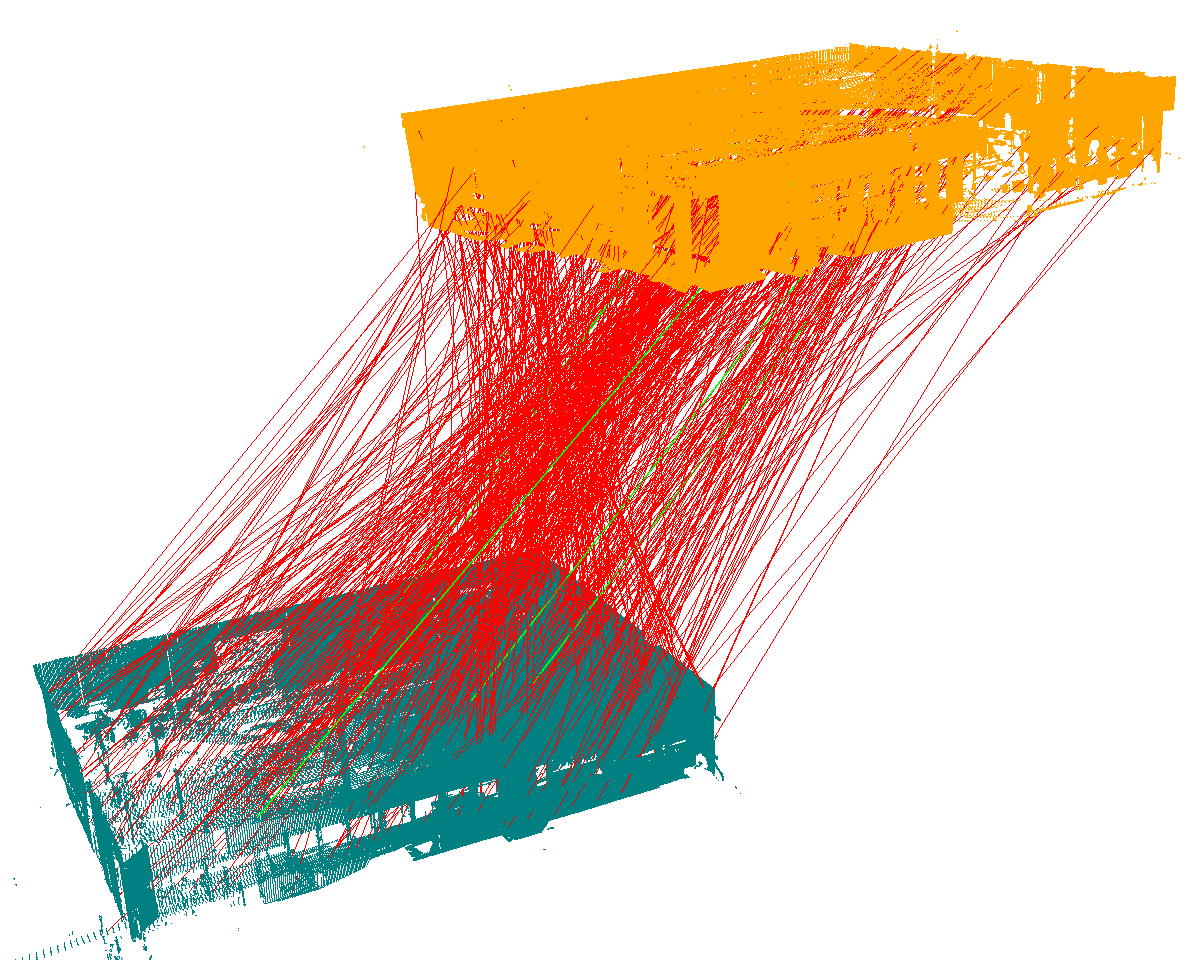}}
    \end{minipage}    }
 &  \multicolumn{2}{c}{
    \begin{minipage}[b]{0.31\columnwidth}
    \centering
    \raisebox{-.1\height}{\includegraphics[width=\linewidth]{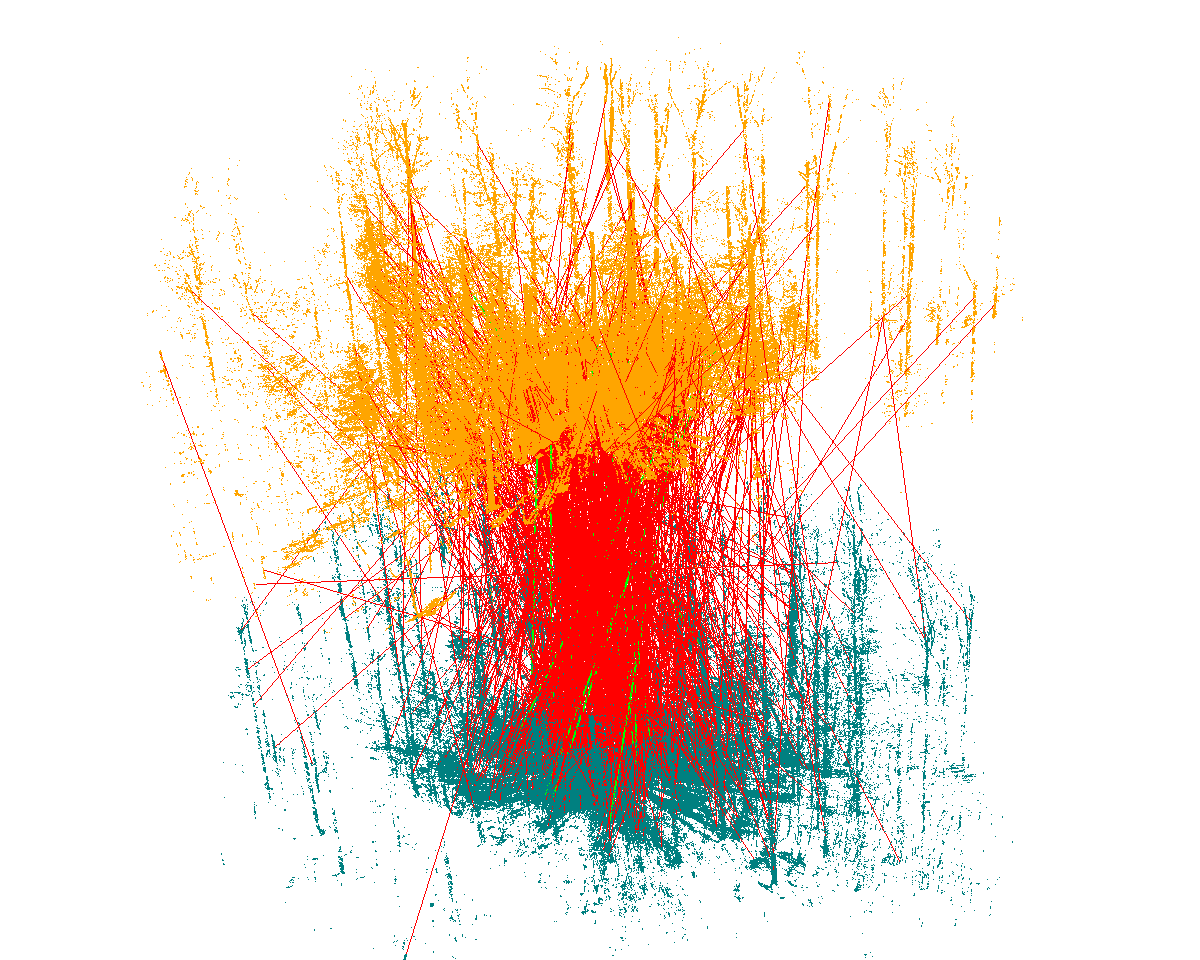}}
    \end{minipage}    }
 \\ \multicolumn{2}{c}{\footnotesize(c) \textit{Facade} s1-s2}
 &  \multicolumn{2}{c}{\footnotesize(d) \textit{Office} s1-s2}
 &  \multicolumn{2}{c}{\footnotesize(e) \textit{Trees} s1-s2}
 \\ \multicolumn{2}{c}{\footnotesize$N = 1901$}
 &  \multicolumn{2}{c}{\footnotesize$N = 1437$}
 &  \multicolumn{2}{c}{\footnotesize$N = 9543$}
 \\ \multicolumn{2}{c}{\footnotesize$\eta = 97.16\%$}
 &  \multicolumn{2}{c}{\footnotesize$\eta = 99.51\%$}
 &  \multicolumn{2}{c}{\footnotesize$\eta = 99.41\%$}
     
 \end{tabular}
 \caption{Examples of scan pairs and initial correspondences for each scene from ETH dataset. The number of correspondences and outlier rate are denoted by $N$ and $\eta$, respectively.}
 \label{ETH_visualization}
\end{figure}

 \begin{figure*}
 \centering
 \begin{tabular}{ b{0.4cm}<{\centering} | c c c }
    \multicolumn{4}{c}{
    \begin{minipage}{0.5\textwidth}
    \centering
    \raisebox{-.1\height}{\includegraphics[width=\linewidth]{Pics/legend.png}}
    \end{minipage} }
 \\ 
    \hline
    \centering
    \raisebox{0.6cm}{\rotatebox{90}{\footnotesize(a) \textit{Arch}}}
  & \begin{minipage}[b]{0.25\textwidth}
    \centering
    \vspace{4pt}\raisebox{-.1\height}{\includegraphics[width=\linewidth]{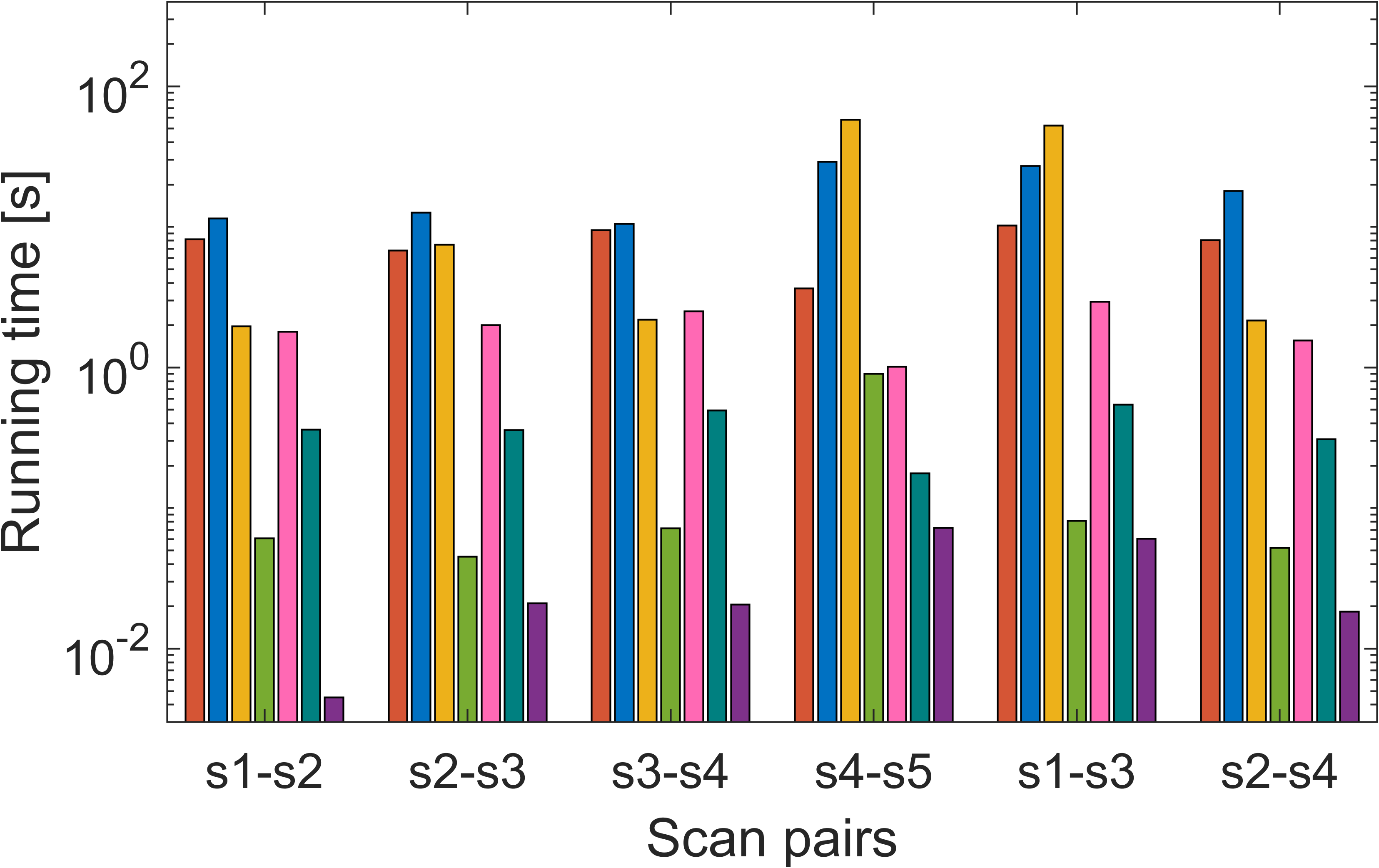}}
    \end{minipage}
  & \begin{minipage}[b]{0.25\textwidth}
    \centering
    \vspace{4pt}\raisebox{-.1\height}{\includegraphics[width=\linewidth]{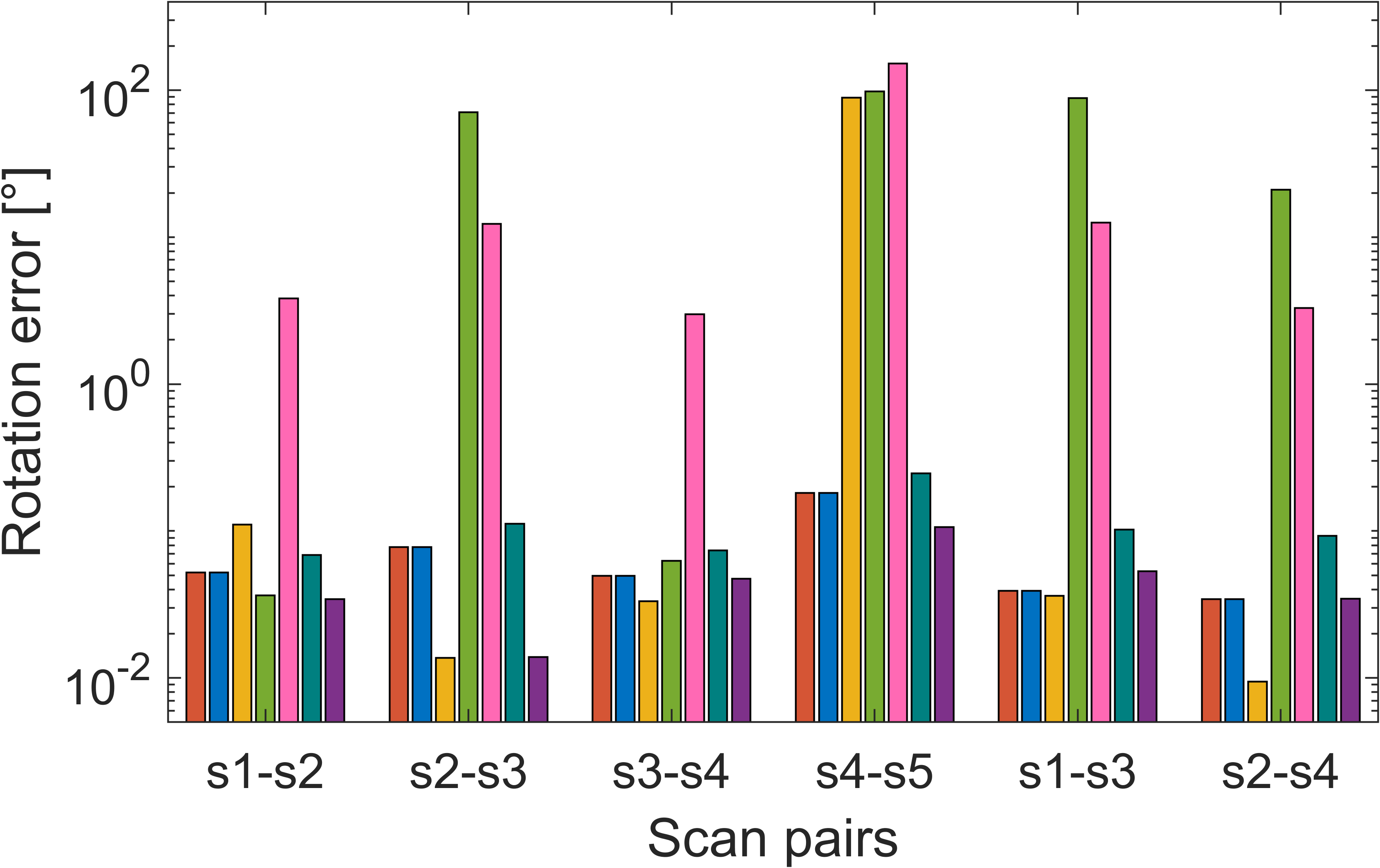}}
    \end{minipage}
  & \begin{minipage}[b]{0.25\textwidth}
    \centering
    \vspace{4pt}\raisebox{-.1\height}{\includegraphics[width=\linewidth]{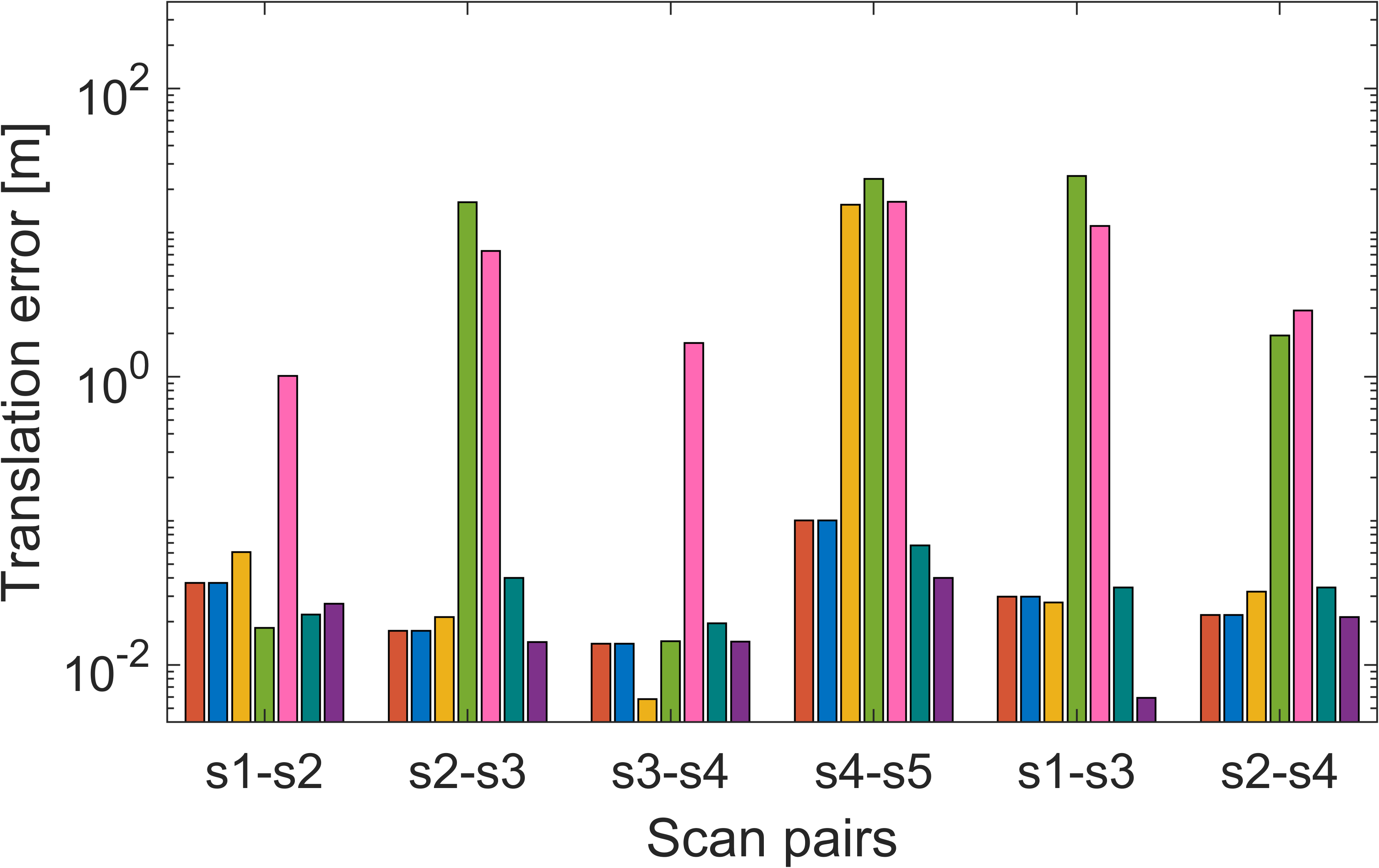}}
    \end{minipage}
 \\ 
    \hline
    \raisebox{0.5cm}{\rotatebox{90}{\footnotesize(b) \textit{Courtyard}}}
  & \begin{minipage}[b]{0.25\textwidth}
    \centering
    \vspace{4pt}\raisebox{-.1\height}{\includegraphics[width=\linewidth]{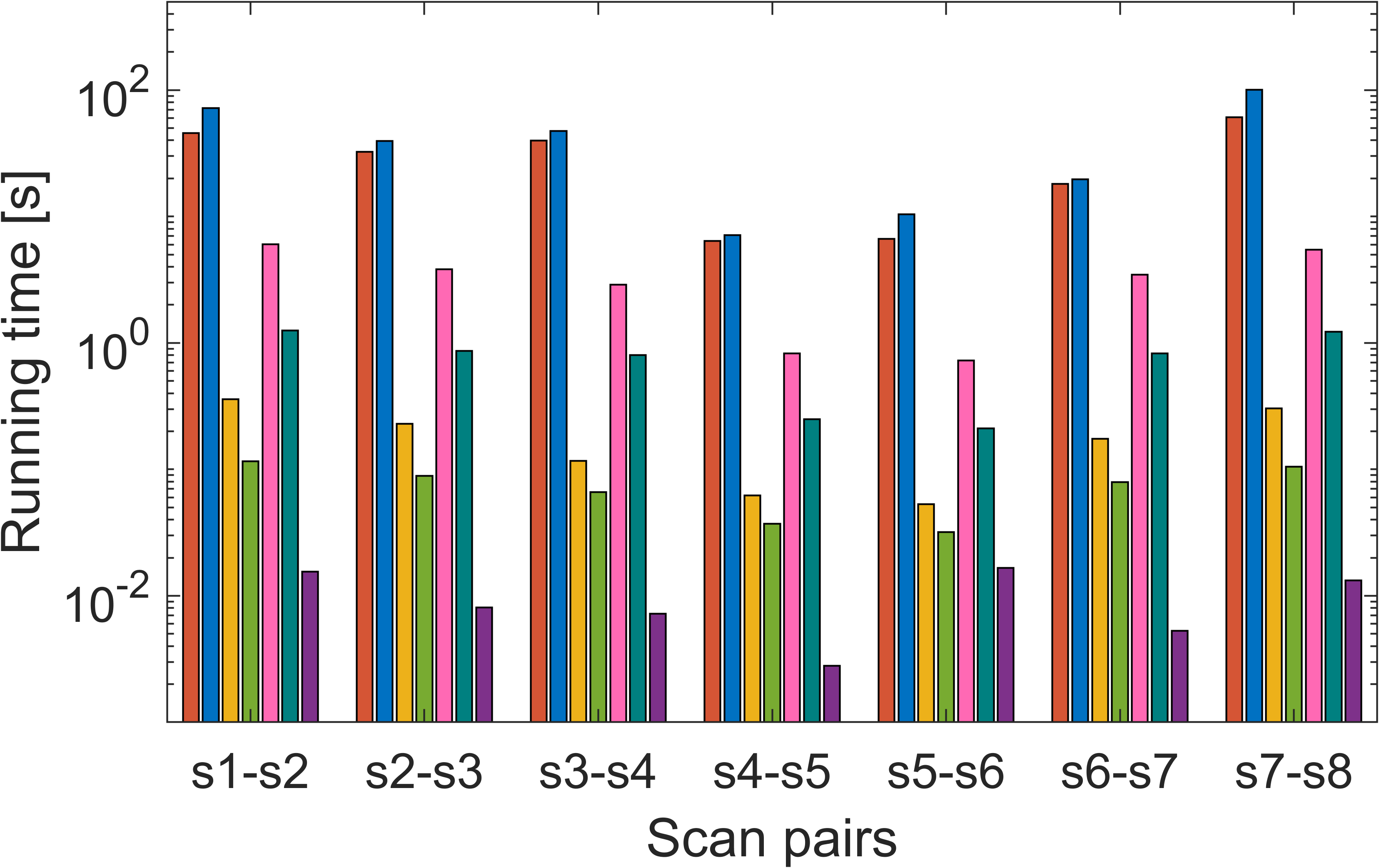}}
    \end{minipage}
  & \begin{minipage}[b]{0.25\textwidth}
    \centering
    \vspace{4pt}\raisebox{-.1\height}{\includegraphics[width=\linewidth]{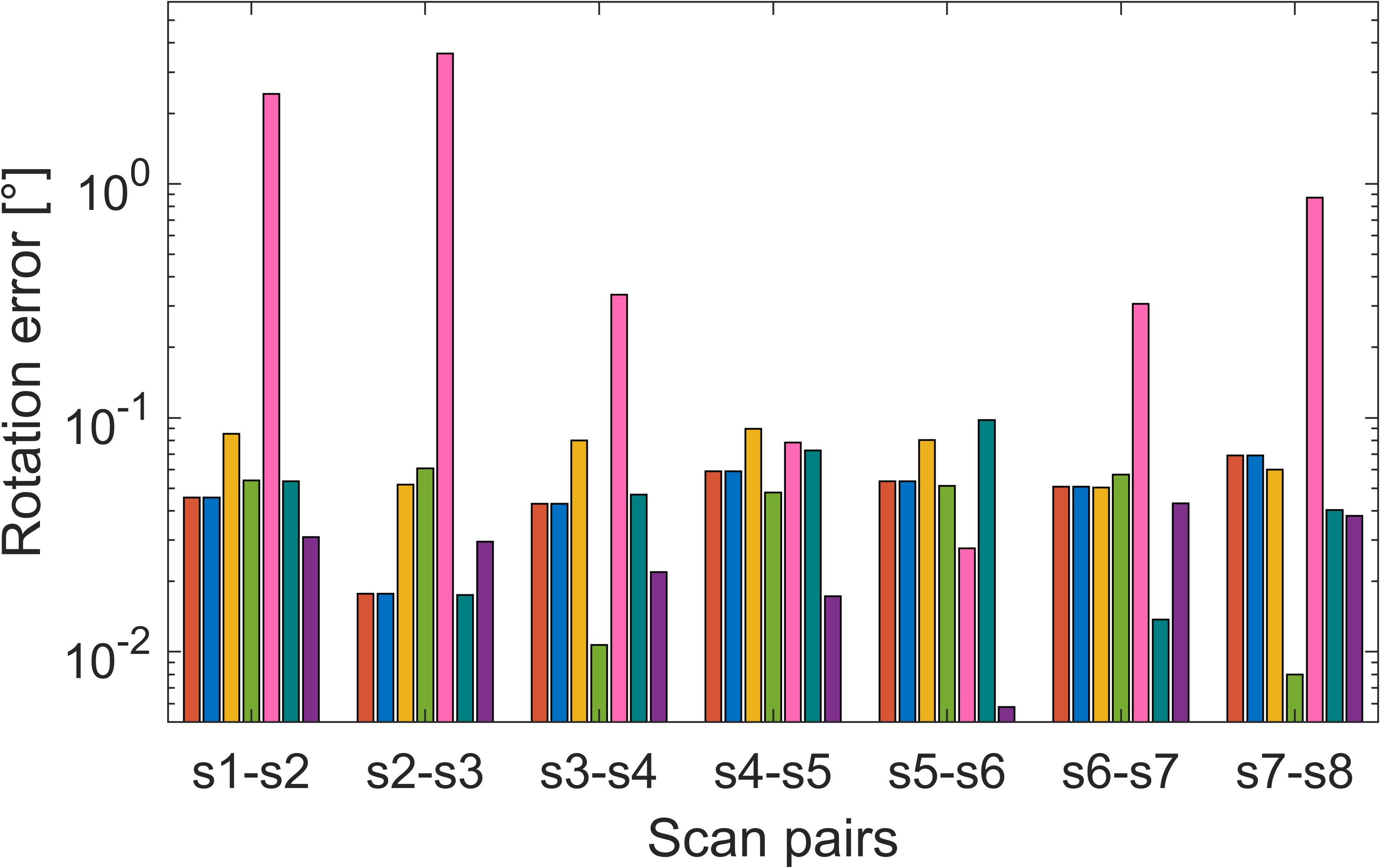}}
    \end{minipage}
  & \begin{minipage}[b]{0.25\textwidth}
    \centering
    \vspace{4pt}\raisebox{-.1\height}{\includegraphics[width=\linewidth]{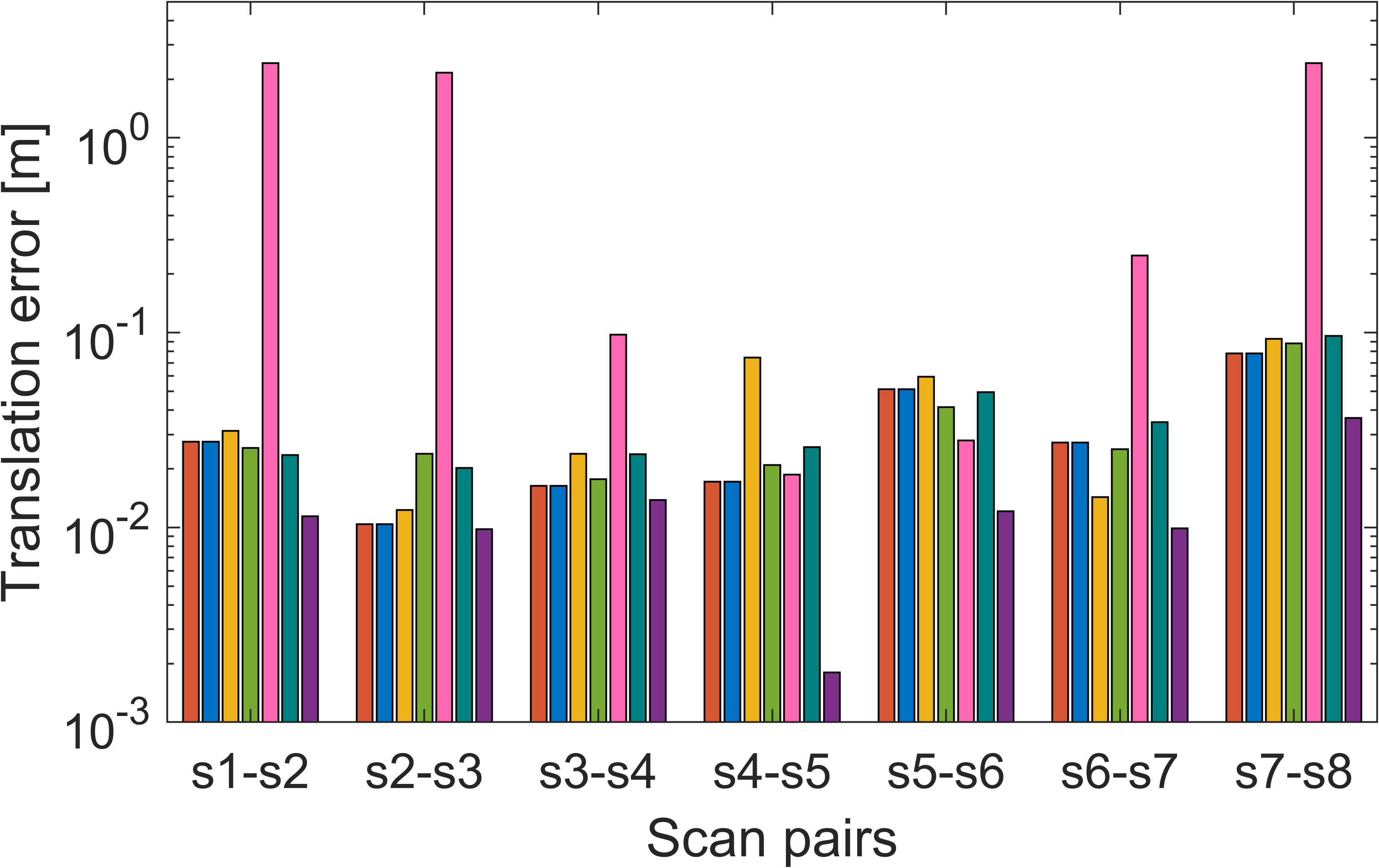}}
    \end{minipage}
 \\ 
    \hline
    \raisebox{0.6cm}{\rotatebox{90}{\footnotesize(c) \textit{Facade}}}
  & \begin{minipage}[b]{0.25\textwidth}
    \centering
    \vspace{4pt}\raisebox{-.1\height}{\includegraphics[width=\linewidth]{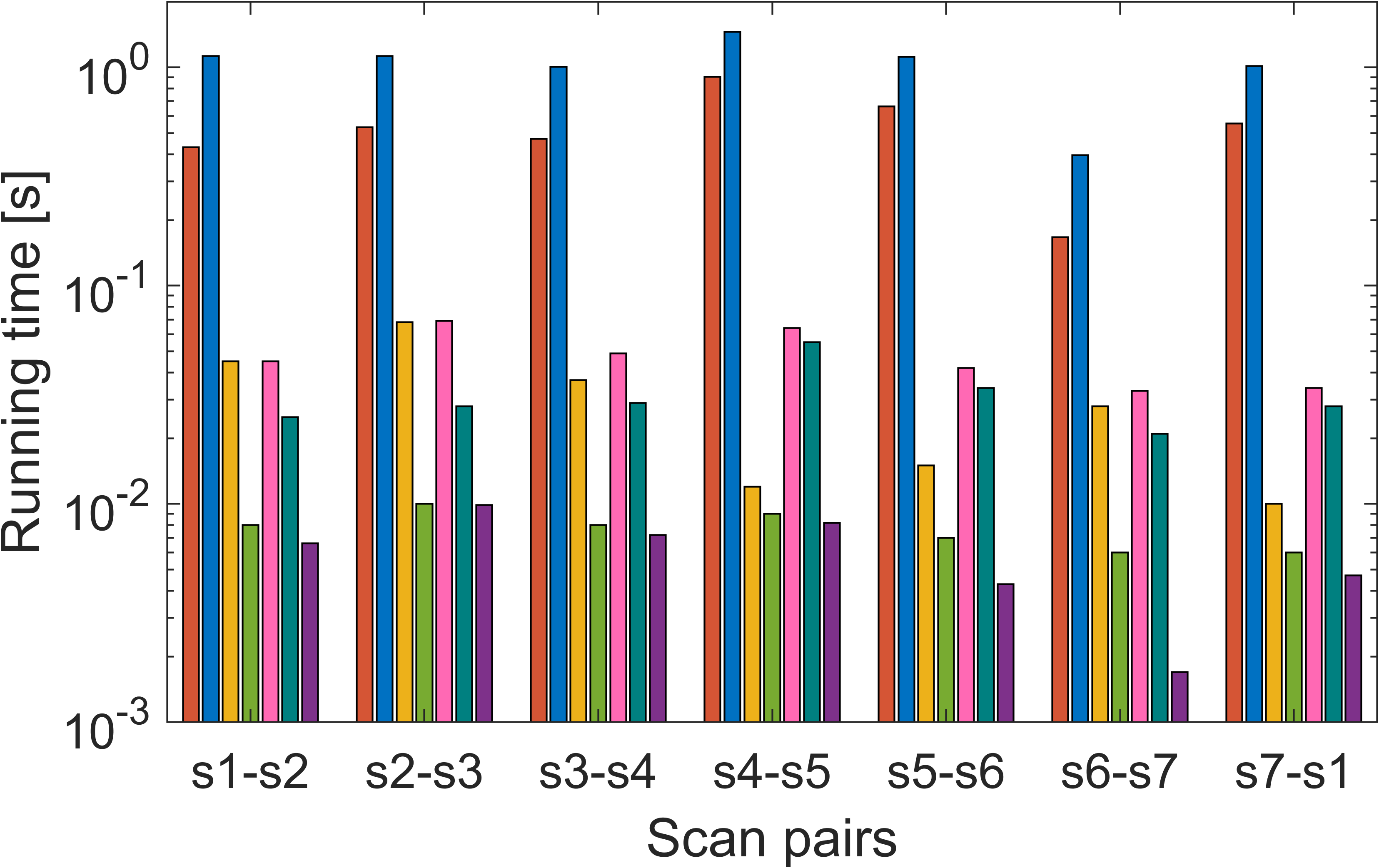}}
    \end{minipage}
  & \begin{minipage}[b]{0.25\textwidth}
    \centering
    \vspace{4pt}\raisebox{-.1\height}{\includegraphics[width=\linewidth]{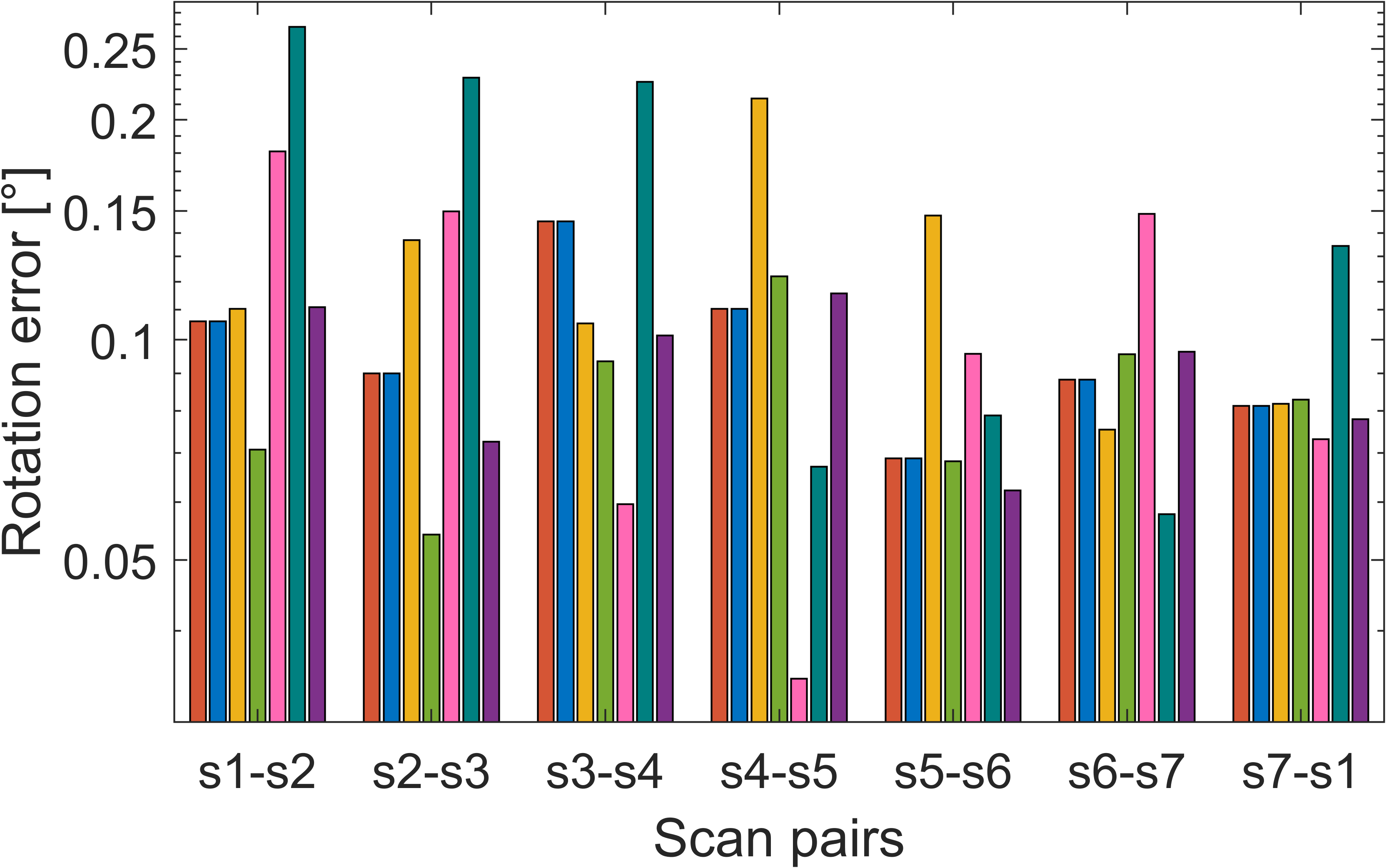}}
    \end{minipage}
  & \begin{minipage}[b]{0.25\textwidth}
    \centering
    \vspace{4pt}\raisebox{-.1\height}{\includegraphics[width=\linewidth]{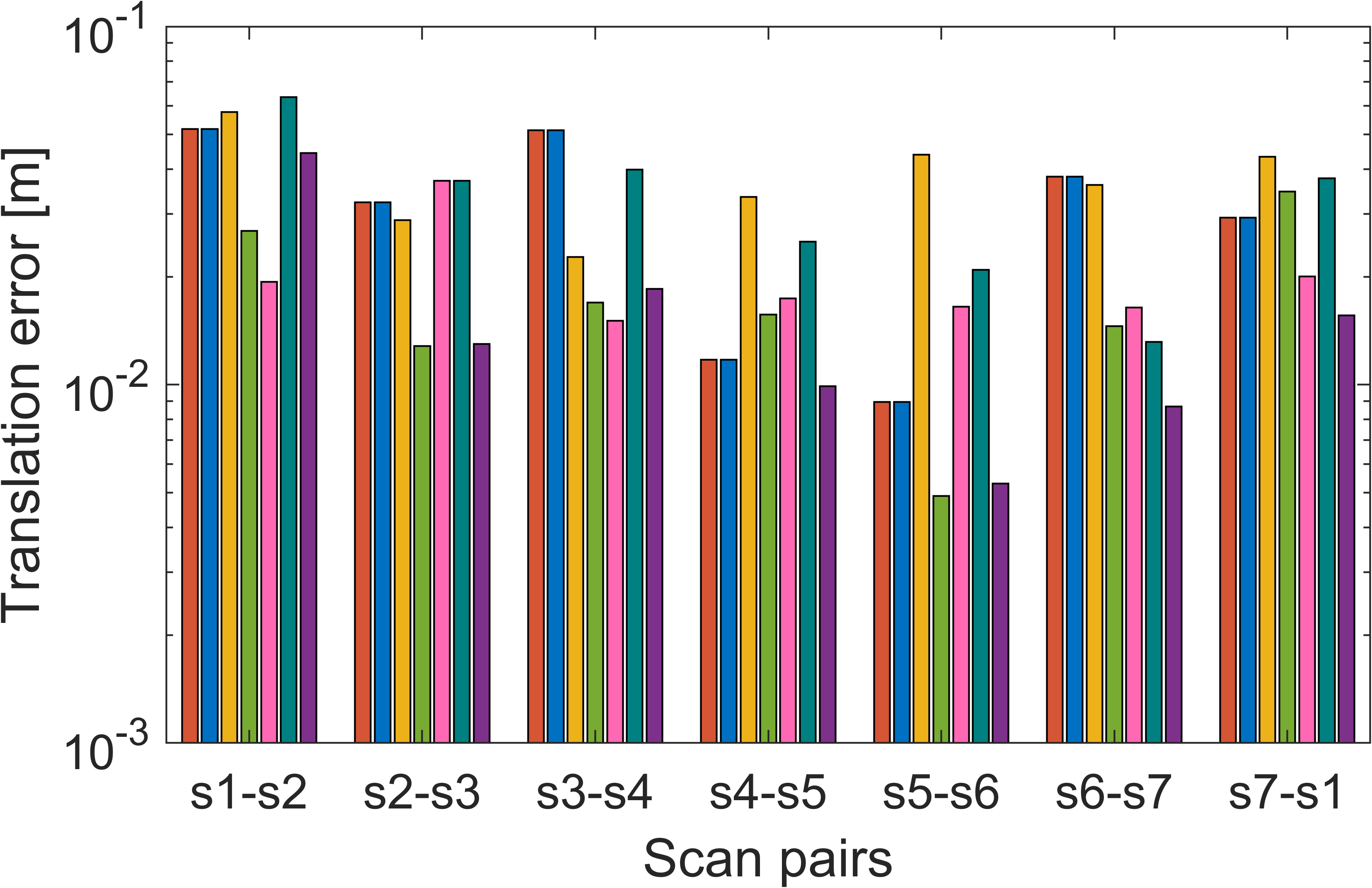}}
    \end{minipage}
 \\ 
    \hline
    \raisebox{0.6cm}{\rotatebox{90}{\footnotesize(d) \textit{Office}}}
  & \begin{minipage}[b]{0.25\textwidth}
    \centering
    \vspace{4pt}\raisebox{-.1\height}{\includegraphics[width=\linewidth]{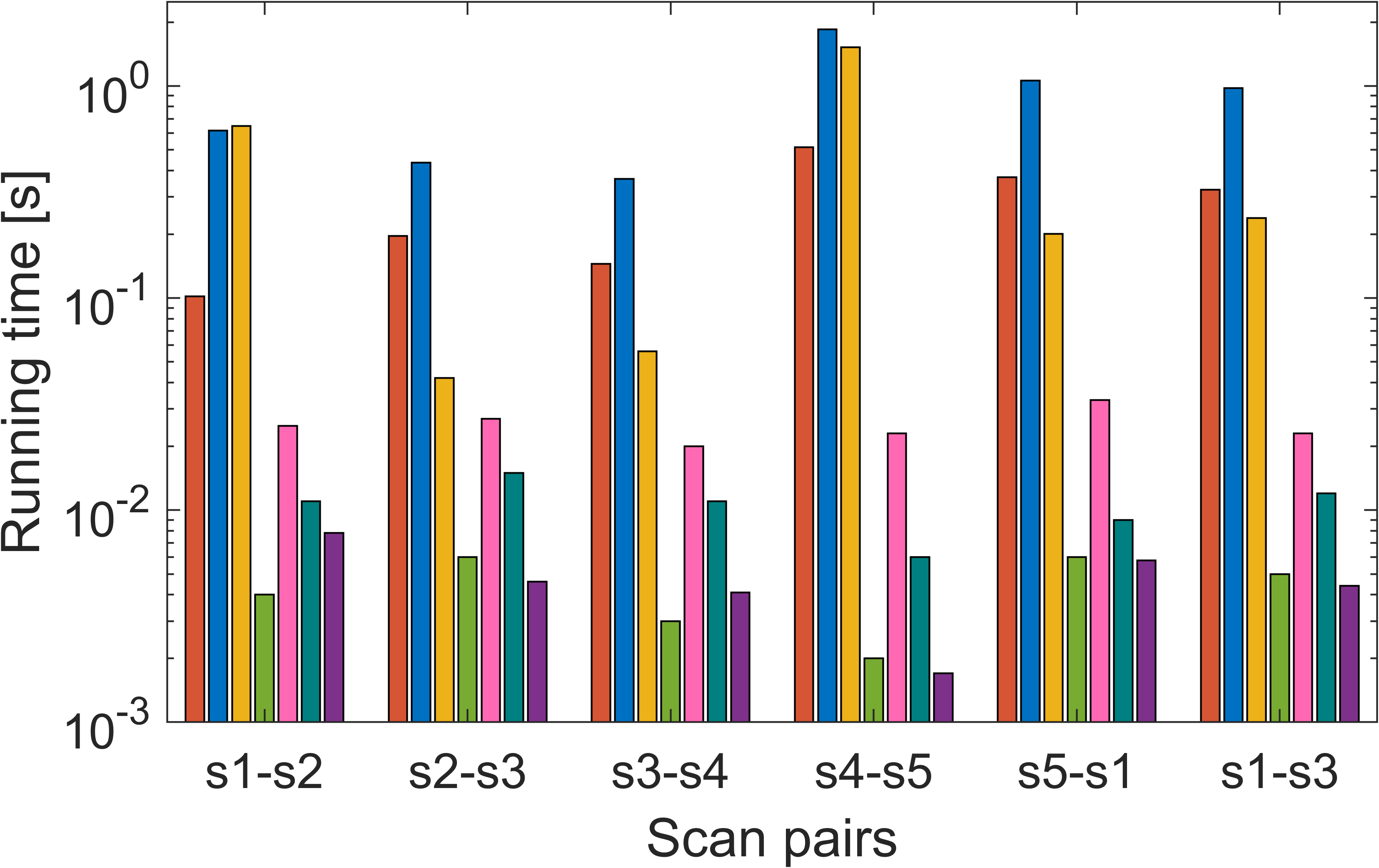}}
    \end{minipage}
  & \begin{minipage}[b]{0.25\textwidth}
    \centering
    \vspace{4pt}\raisebox{-.1\height}{\includegraphics[width=\linewidth]{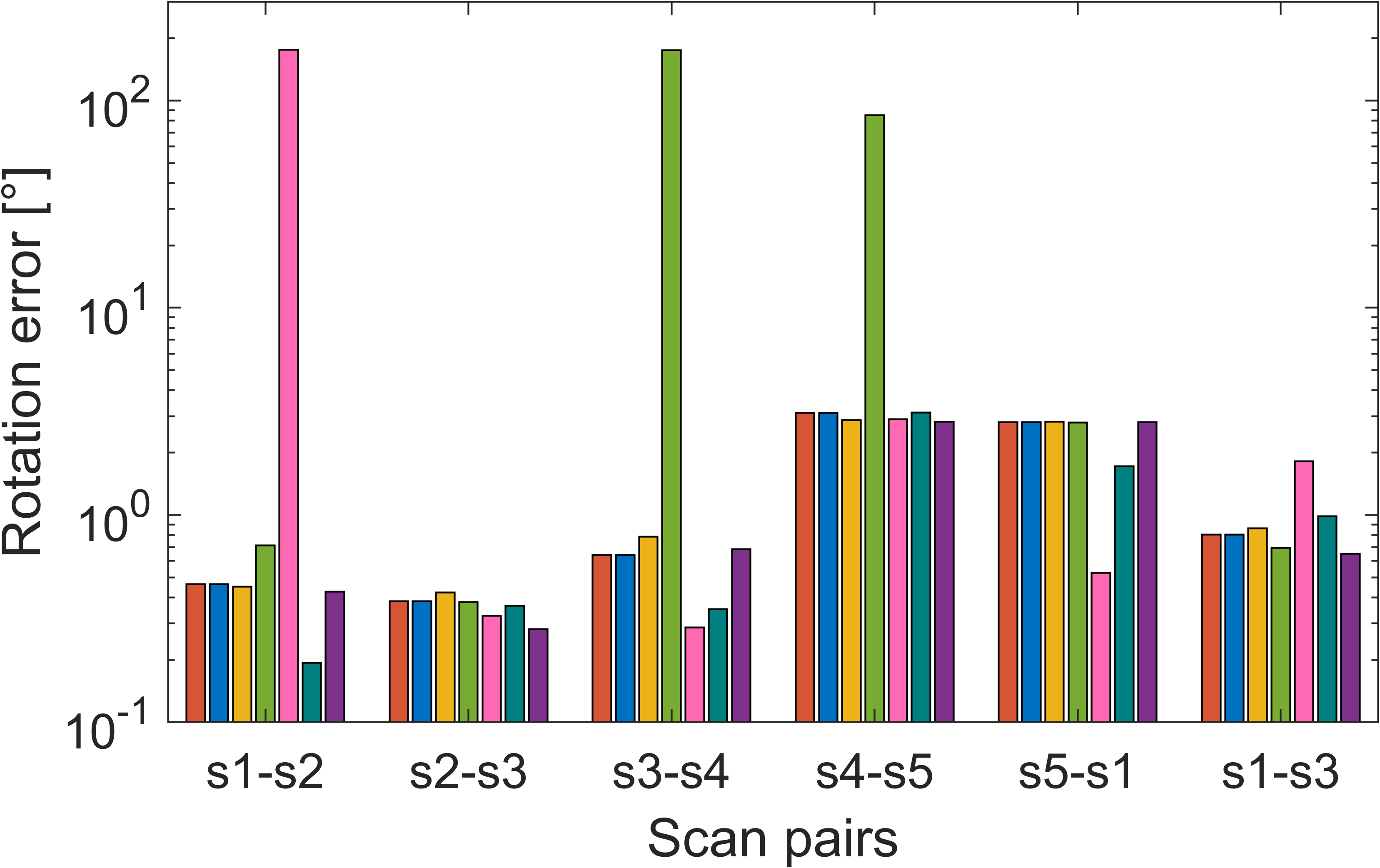}}
    \end{minipage}
  & \begin{minipage}[b]{0.25\textwidth}
    \centering
    \vspace{4pt}\raisebox{-.1\height}{\includegraphics[width=\linewidth]{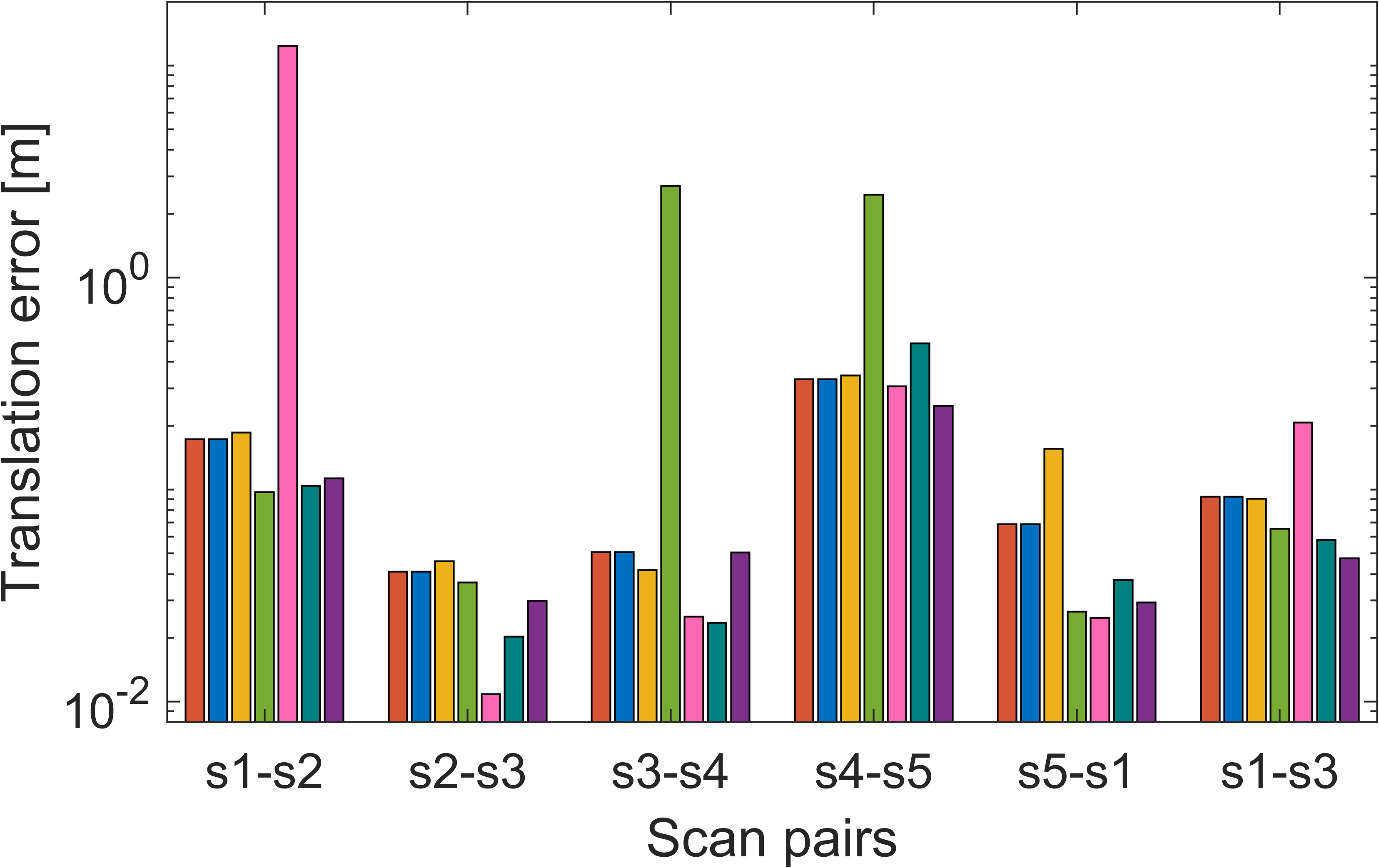}}
    \end{minipage}
 \\ 
    \hline
    \raisebox{0.6cm}{\rotatebox{90}{\footnotesize(e) \textit{Trees}}}
  & \begin{minipage}[b]{0.25\textwidth}
    \centering
    \vspace{4pt}\raisebox{-.1\height}{\includegraphics[width=\linewidth]{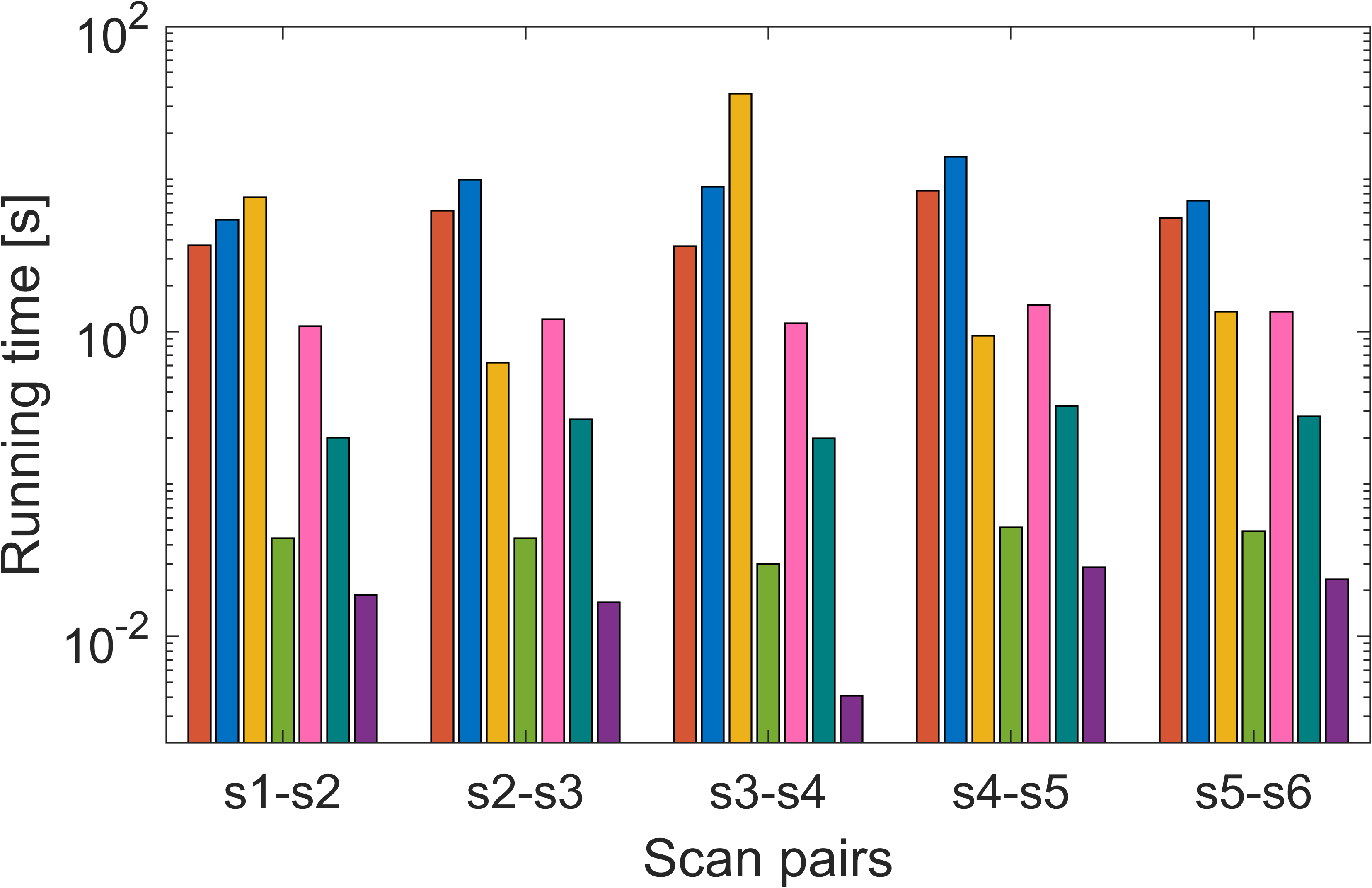}}
    \end{minipage}
  & \begin{minipage}[b]{0.25\textwidth}
    \centering
    \vspace{4pt}\raisebox{-.1\height}{\includegraphics[width=\linewidth]{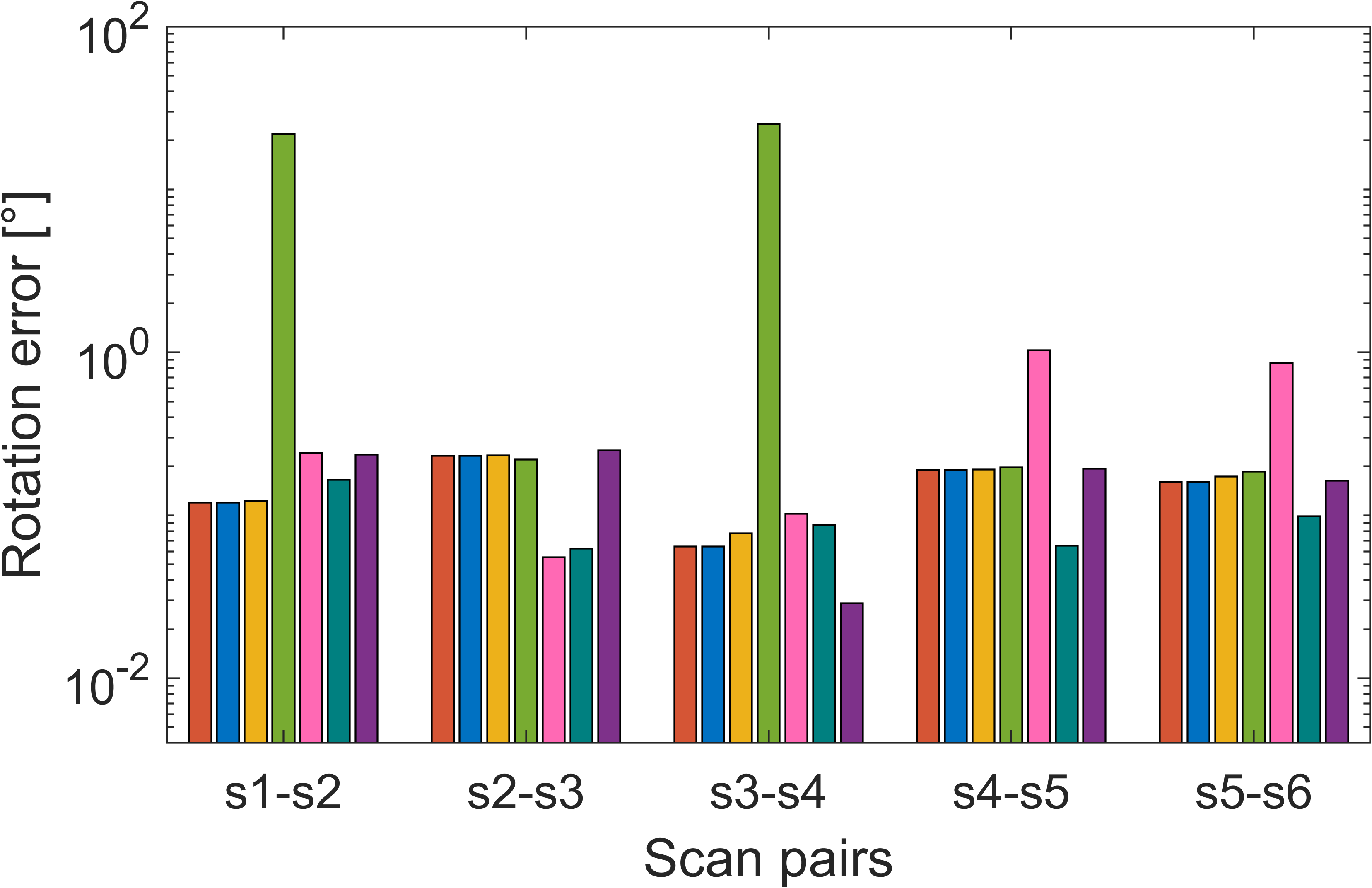}}
    \end{minipage}
  & \begin{minipage}[b]{0.25\textwidth}
    \centering
    \vspace{4pt}\raisebox{-.1\height}{\includegraphics[width=\linewidth]{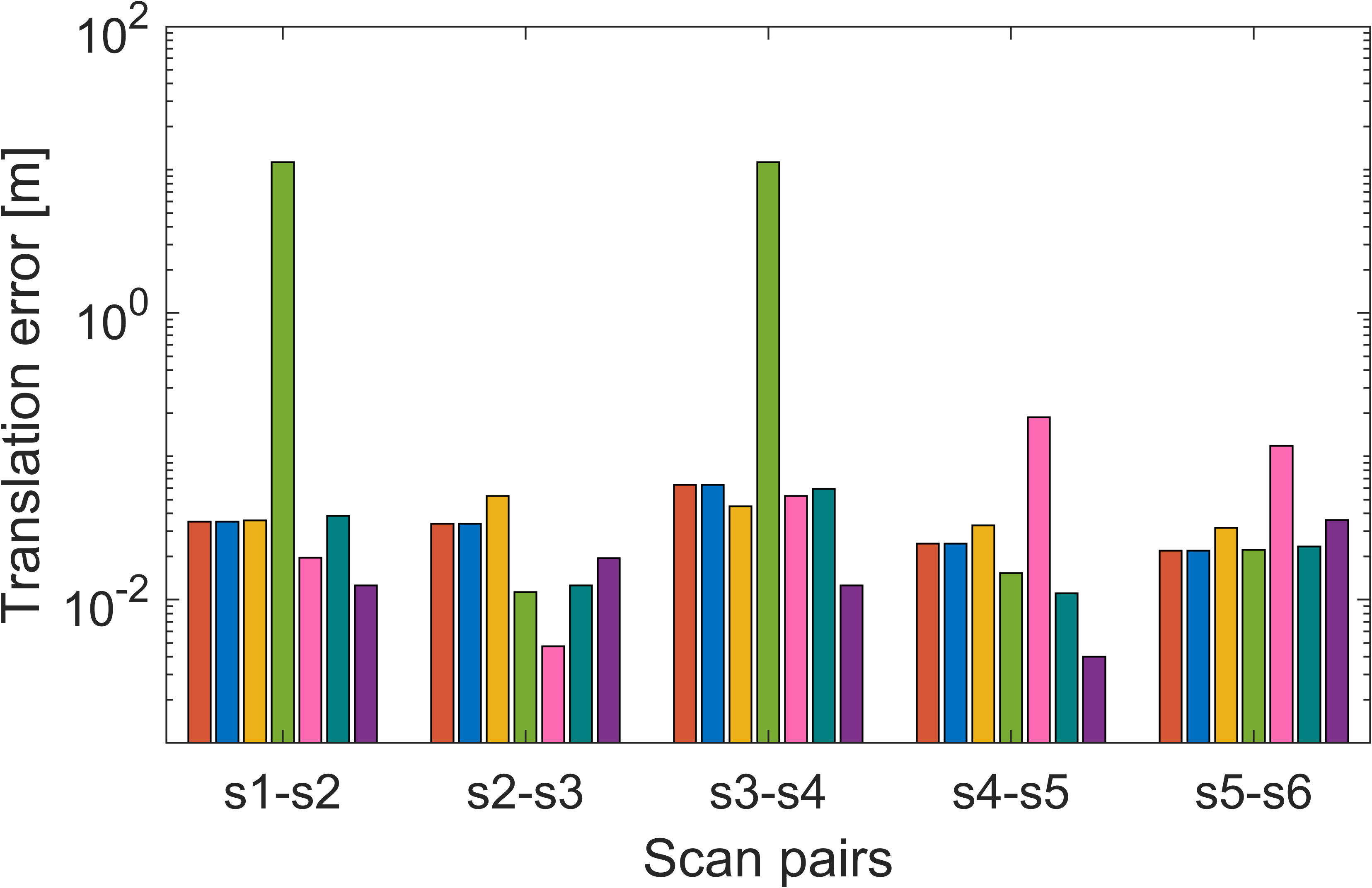}}
    \end{minipage}
 \\ \hline
 \end{tabular}
 \caption{Running times, rotation errors, and translation errors of all registration methods on the ETH dataset. (a)-(e) show the registration results for \textit{Arch}, \textit{Courtyard}, \textit{Facade}, \textit{Office}, \textit{Trees}.}
 \label{ETH_dataset}
 \end{figure*}

\subsubsection{Correspondence-free experiments}

In this section, we perform challenging correspondence-free experiments to compare our extended method with global Go-ICP, Go-ICPT, and local ICP, CPD, and GMMReg, using the \textit{Bunny} dataset from the Stanford 3-D Scanning Repository\cite{curless1996volumetric}. Due to the limitations of GO-ICP\cite{7368945}, the point cloud is pre-normalized to fit within the cube $[-1,1]^3$. The \textit{Bunny} dataset is initially down-sampled to $M=234$ points, serving as the source point cloud. Random rotations and translations are then applied to generate the target point clouds. Additionally, a specific proportion of points is randomly removed from the target point cloud to simulate partial overlap between the source and target point clouds, with the overlap rate denoted by $\rho$ ranging from $90\%$ to $40\%$. Zero-mean Gaussian noise with $\sigma=0.001$ is added to the target point cloud. The experiments are repeated $30$ times at each overlap rate, and the results are presented in Fig.~\ref{synthetic_corr_free}. Notably, the time costs for constructing distance transform (DT)\cite{7368945} in Go-ICP and Go-ICPT are not recorded, but averaged about $12$s.

It can be observed from the rotation and translation errors that local ICP tends to converge to local minima even at $\rho=90\%$. Due to the adoption of GMM, the local CPD and Gmmreg are more robust to partial overlap than ICP, as evidenced by their better error distributions than those of ICP. Notably, global Go-ICP and Go-ICPT are much more robust than these local methods. Among them, GO-ICPT ($50\%$) demonstrates the best performance, indicating that Go-ICP is strongly sensitive to the pre-set trimming percentage. However, as the overlap rate decreases, the time costs of GO-ICP and GO-ICPT increase significantly compared to local methods. In contrast, the proposed global method attains commendable robustness and accuracy even at $\rho=40\%$. On the other hand, it is faster than the existing global method (i.e., GO-ICP) and even has comparable efficiency to local methods. This shows the feasibility of our proposed method in addressing correspondence-free registration problems. Meanwhile, the correspondence-free experiment on real-world datasets is also provided in Appendix B.
\begin{table*}
\footnotesize
 \setlength{\tabcolsep}{2pt}
 \renewcommand{\arraystretch}{1.2}
 \caption{Quantitative registration results on three real-world datasets. The results include average running time (s) $|$ average rotation error ($\degree$) $|$ average translation error ($\SI{}{m}$). Bolded fonts indicate the best results.}
 \label{average}
 \centering
 \begin{tabular}{c r@{\hspace{0pt}} c@{\hspace{0pt}} l r@{\hspace{0pt}} c@{\hspace{0pt}} l r@{\hspace{0pt}} c@{\hspace{0pt}} l r@{\hspace{0pt}} c@{\hspace{0pt}} l r@{\hspace{0pt}} c@{\hspace{0pt}} l r@{\hspace{0pt}} c@{\hspace{0pt}} l r@{\hspace{0pt}} c@{\hspace{0pt}} l}
    \hline
    \multirow{2}*{Method} 
    & \multicolumn{21}{c}{Dataset}\\
    \cline{2-22}
    {} & \multicolumn{3}{c}{ETH-\textit{Arch}} & \multicolumn{3}{c}{ETH-\textit{Courtyard}} & \multicolumn{3}{c}{ETH-\textit{Facade}} & \multicolumn{3}{c}{ETH-\textit{Office}} & \multicolumn{3}{c}{ETH-\textit{Trees}} & \multicolumn{3}{c}{WHU-TLS} & \multicolumn{3}{c}{A9}\\
    \hline 
    FMP+BnB &$7.722|$ &$0.072$ &$|0.037$ &$30.02|$ &$0.048$ &$|0.033$ &$0.532|$ &$0.099$ &$|0.032$ &$0.276|$ &$1.368$ &$|0.126$ &$5.485|$ &$0.153$ &$|0.036$ &$10.45|$ &$0.245$ &$|12.51$ &$0.583|$ &$0.624$ &$|0.363$ \\
    BnB &$18.12|$ &$0.072$ &$|0.037$ &$42.33|$ &$0.048$ &$|0.033$ &$1.035|$ &$0.099$ &$|0.032$ &$0.885|$ &$1.368$ &$|0.126$ &$9.099|$ &$0.153$ &$|0.036$ &$22.02|$ &$0.245$ &$|12.51$ &$11.70|$ &$0.624$ &$|0.363$ \\
    RANSAC &$20.71|$ &$14.87$ &$|3.139$ &$0.185|$ &$0.071$ &$|0.044$ &$0.031|$ &$0.124$ &$|0.038$ &$0.451|$ &$1.369$ &$|0.144$ &$9.327|$ &$0.159$ &$|0.040$ &$5.121|$ &$0.214$ &$|12.52$ &$13.55|$ &$7.359$ &$|2.354$ \\
    FGR &$0.202|$ &$46.31$ &$|11.07$ &$0.075|$ &$0.041$ &$|0.035$ &$0.008|$ &$\textbf{0.084}$ &$|0.018$ &$\textbf{0.004}|$ &$44.16$ &$|0.899$ &$0.044|$ &$9.533$ &$|4.532$ &$\textbf{0.041}|$ &$80.66$ &$|37.96$ &$\textbf{0.008}|$ &$17.75$ &$|8.012$ \\
    GTA &$1.966|$ &$31.12$ &$|6.747$ &$3.320|$ &$1.093$ &$|1.056$ &$0.048|$ &$0.106$ &$|0.020$ &$0.025|$ &$30.22$ &$|2.155$ &$1.253|$ &$0.458$ &$|0.077$ &$1.499|$ &$3.161$ &$|15.09$ &$0.078|$ &$68.28$ &$|10.47$ \\
    GROR &$0.374|$ &$0.116$ &$|0.036$ &$0.777|$ &$0.049$ &$|0.039$ &$0.031|$ &$0.151$ &$|0.034$ &$0.011|$ &$\textbf{1.122}$ &$|0.122$ &$0.253|$ &$\textbf{0.096}$ &$|0.029$ &$0.318|$ &$0.795$ &$|12.72$ &$0.040|$ &$1.652$ &$|0.525$ \\
    Ours &$\textbf{0.033}|$ &$\textbf{0.048}$ &$|\textbf{0.020}$ &$\textbf{0.010}|$ &$\textbf{0.027}$ &$|\textbf{0.014}$ &$\textbf{0.006}|$ &$0.091$ &$|\textbf{0.016}$ &$0.005|$ &$1.278$ &$|\textbf{0.086}$ &$\textbf{0.018}|$ &$0.174$ &$|\textbf{0.017}$ &$0.053|$ &$\textbf{0.112}$ &$|\textbf{0.165}$ &$0.055|$ &$\textbf{0.394}$ &$|\textbf{0.361}$ \\
    \hline 
 \end{tabular}
\end{table*}

\subsection{Real-World Data Experiments}

\begin{figure}
 \centering
 \begin{tabular}{c}
    \begin{minipage}{0.9\columnwidth}
    \centering
    \raisebox{-.1\height}{\includegraphics[width=\columnwidth]{Pics/legend.png}}
    \end{minipage}
 \\ \begin{minipage}[b]{0.98\columnwidth}
    \centering
    \raisebox{-.1\height}{\includegraphics[width=\columnwidth]{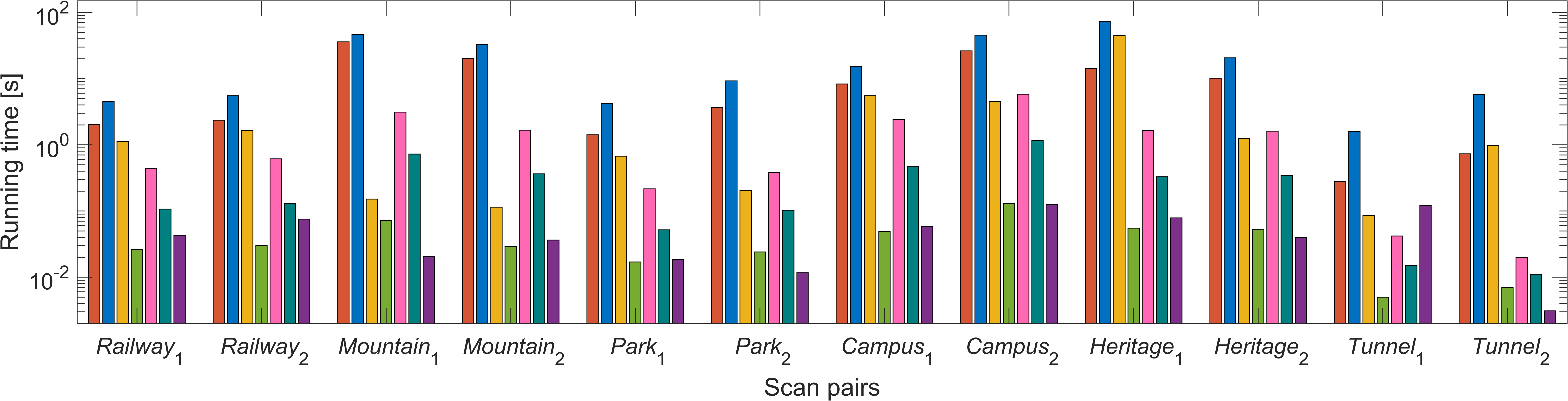}}
    \end{minipage}
 \\ \footnotesize(a) Running time
 \\ 
    \begin{minipage}[b]{0.98\columnwidth}
    \centering
    \raisebox{-.1\height}{\includegraphics[width=\columnwidth]{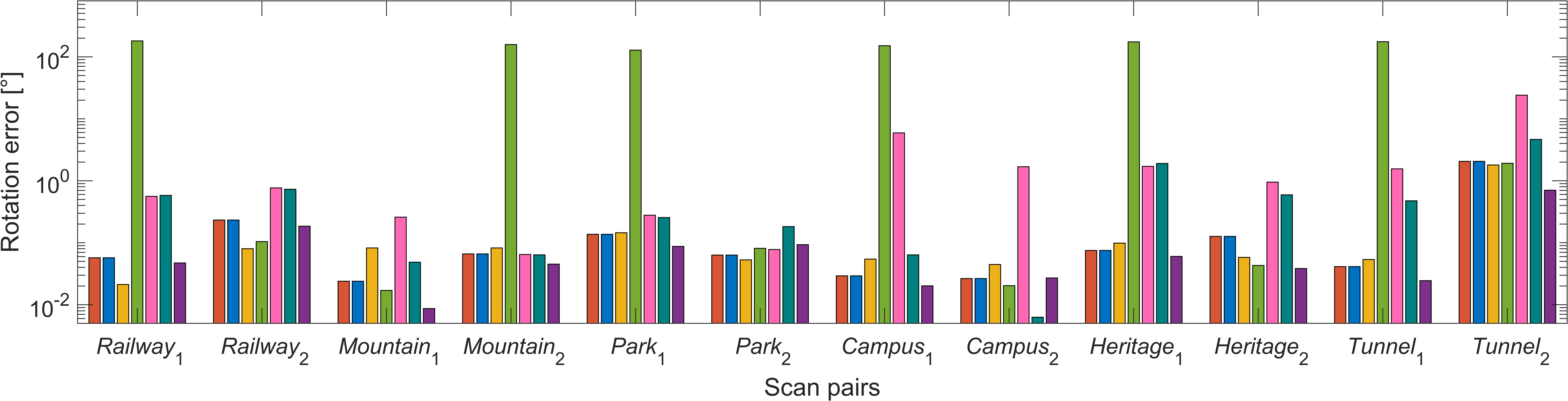}}
    \end{minipage}
 \\ \footnotesize(b) Rotation error
 \\ 
    \begin{minipage}[b]{0.98\columnwidth}
    \centering
    \raisebox{-.1\height}{\includegraphics[width=\columnwidth]{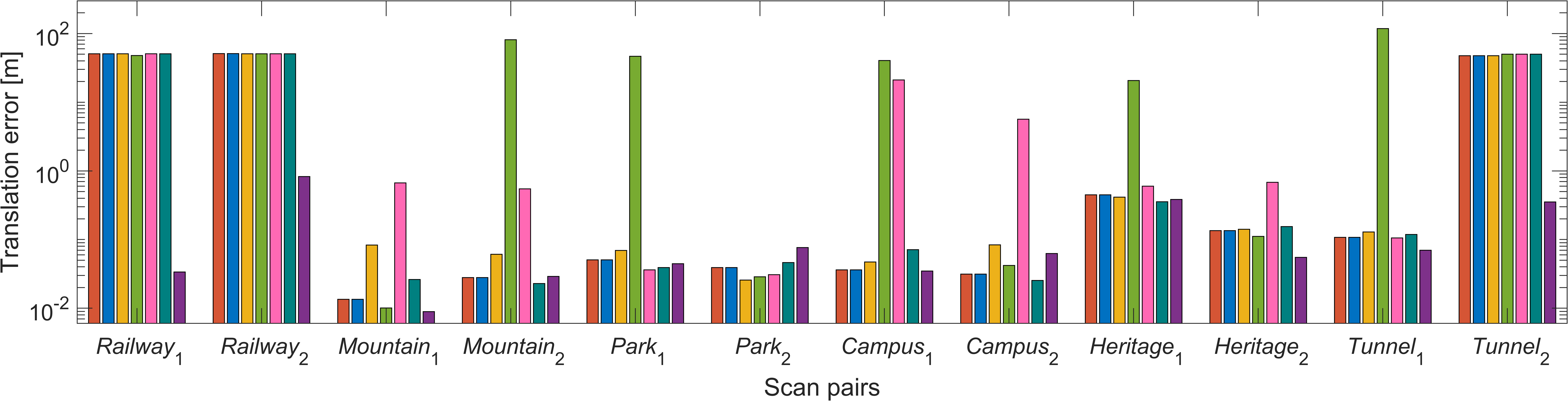}}
    \end{minipage}
 \\ \footnotesize(c) Translation error
 \end{tabular}
 \caption{(a) Running time, (b) Rotation error, and (c) Translation error of all registration methods on the WHU-TLS dataset.}
 \label{WHU_dataset}
\end{figure}

\begin{figure*}\footnotesize
 \centering
 \begin{tabular}{c c c c}
    \begin{minipage}[b]{0.23\textwidth}
    \centering
    \raisebox{-.1\height}{\includegraphics[width=\linewidth]{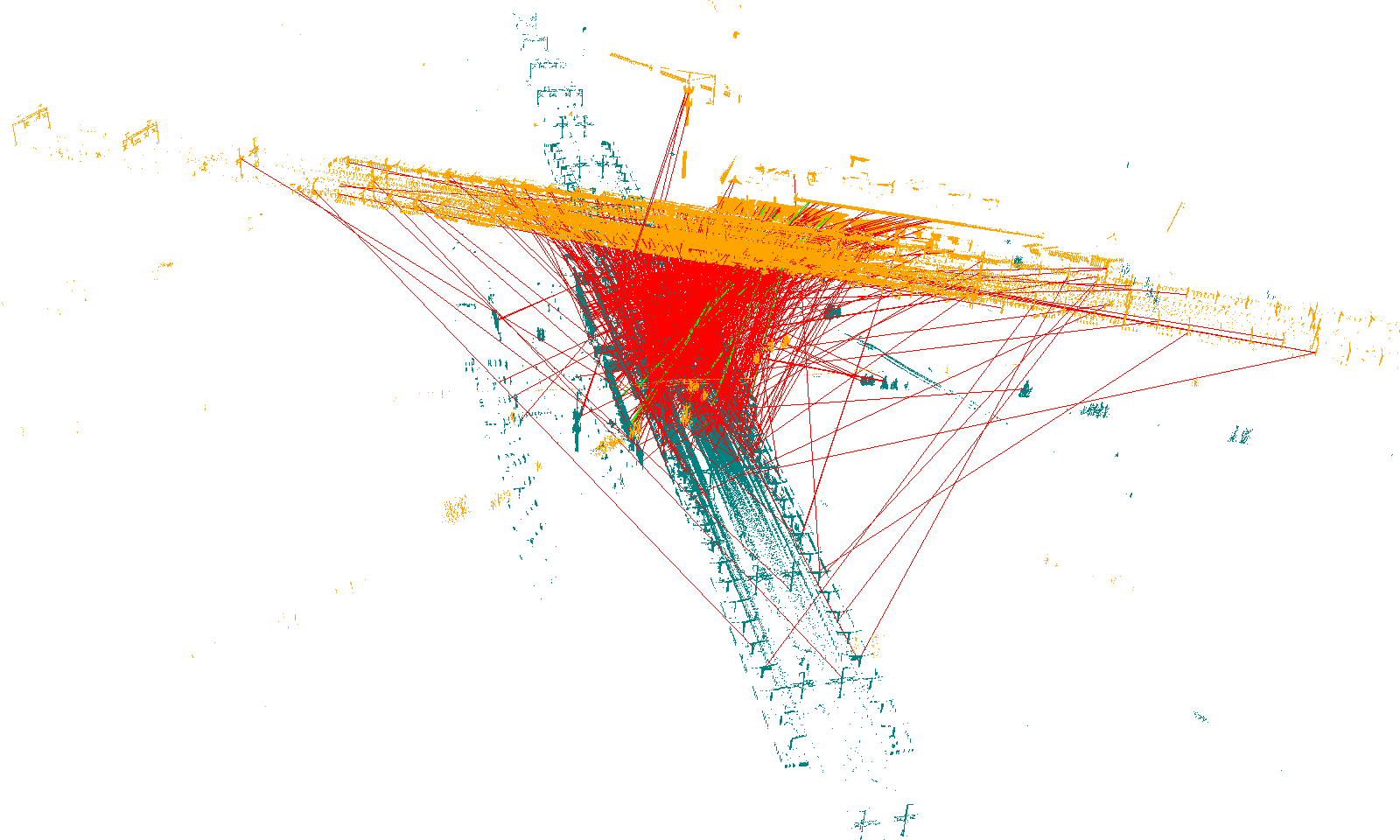}}
    \end{minipage}   
 &  \begin{minipage}[b]{0.23\textwidth}
    \centering
    \raisebox{-.1\height}{\includegraphics[width=\linewidth]{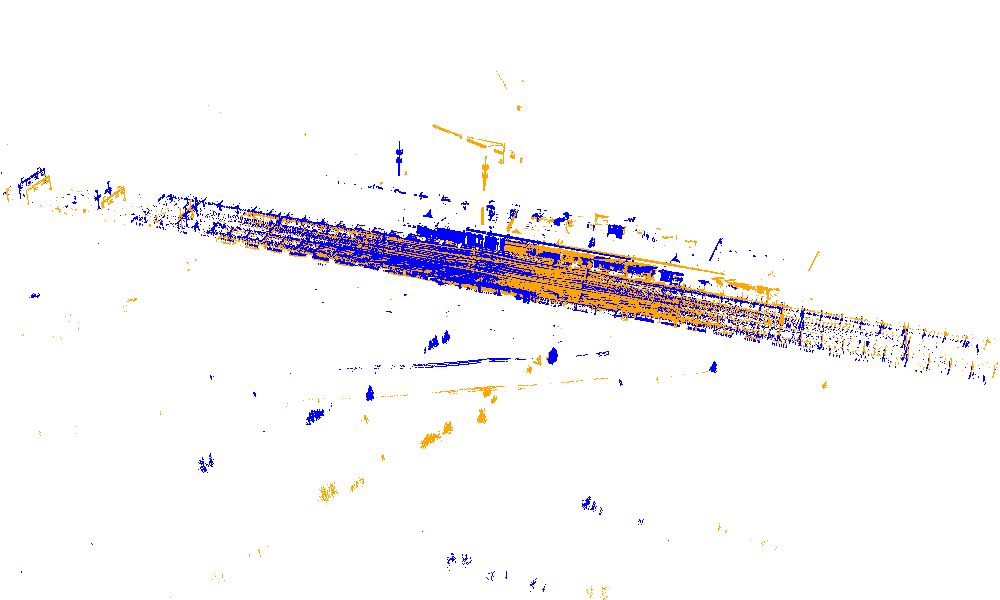}}
    \end{minipage}
 &  \begin{minipage}[b]{0.23\textwidth}
    \centering
    \raisebox{-.1\height}{\includegraphics[width=\linewidth]{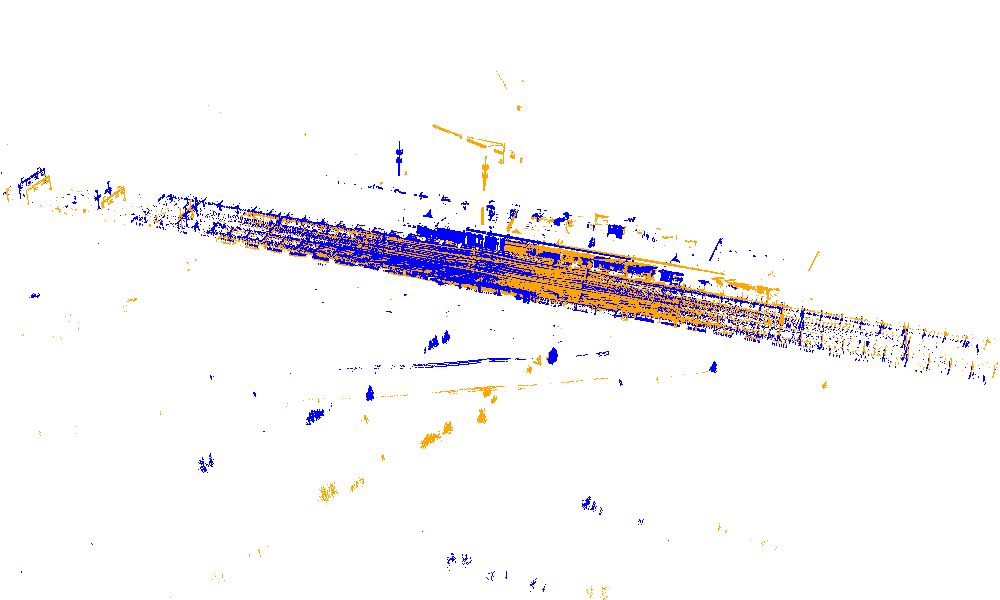}}
    \end{minipage}
 &  \begin{minipage}[b]{0.23\textwidth}
    \centering
    \raisebox{-.1\height}{\includegraphics[width=\linewidth]{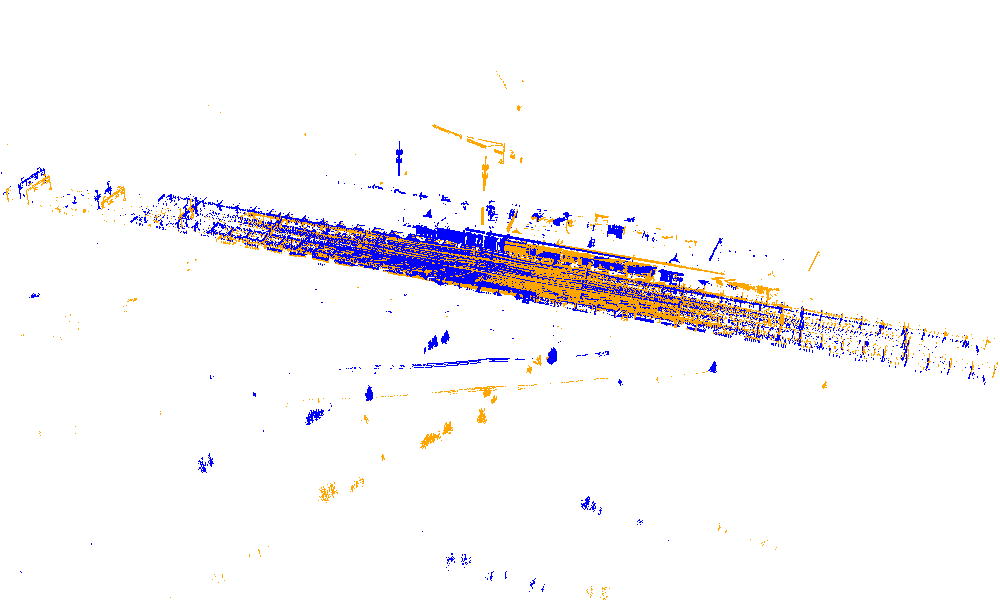}}
    \end{minipage}
 \\ \footnotesize(a) Initial
 &  \footnotesize(b) FMP+BnB
 &  \footnotesize(c) BnB
 &  \footnotesize(d) RANSAC
 \\ $N = 6013$
 & $T=\SI{2.040}{s}$
 & $T=\SI{4.553}{s}$
 & $T=\SI{1.131}{s}$
 \\ $\eta = 99.60\%$
 &  $RE=0.058\degree$
 &  $RE=0.058\degree$
 &  $RE=0.021\degree$
 \\ ~
 &  $TE=\SI{50.809}{m}$
 &  $TE=\SI{50.809}{m}$
 &  $TE=\SI{50.795}{m}$

 \\ 
    \begin{minipage}[b]{0.23\textwidth}
    \centering
    \raisebox{-.1\height}{\includegraphics[width=\linewidth]{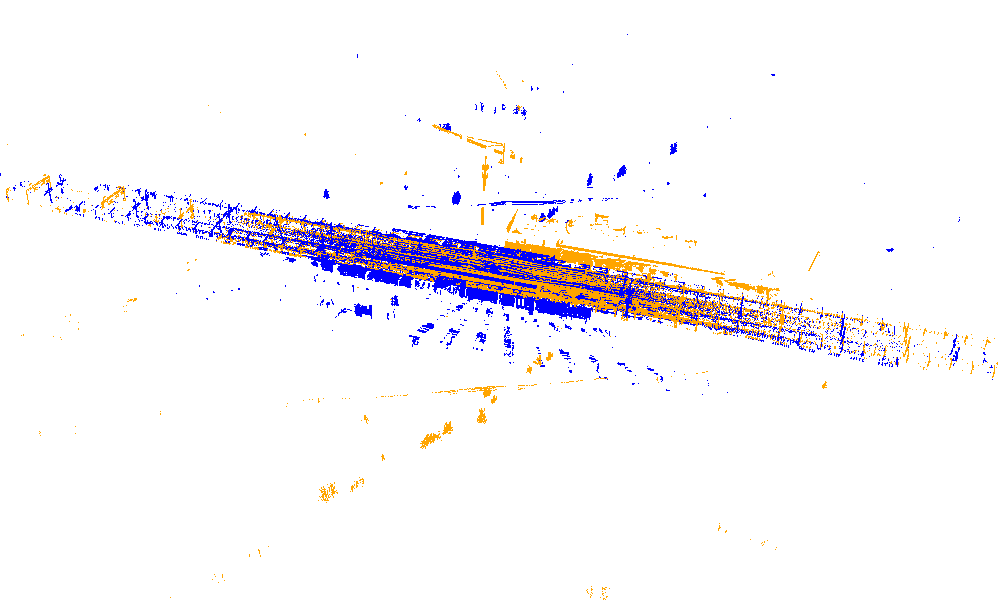}}
    \end{minipage}
 &  \begin{minipage}[b]{0.23\textwidth}
    \centering
    \raisebox{-.1\height}{\includegraphics[width=\linewidth]{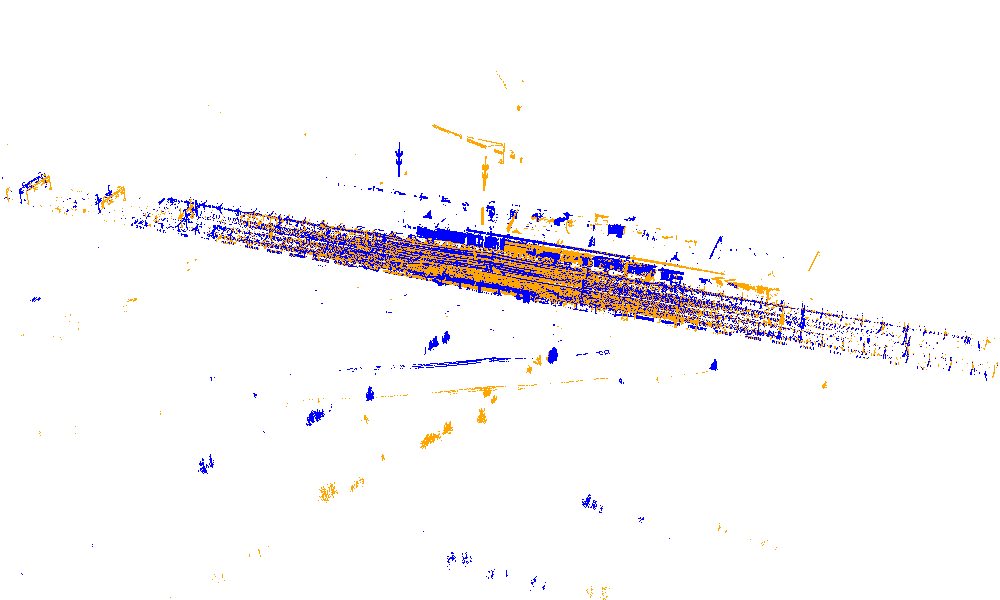}}
    \end{minipage}
 &  \begin{minipage}[b]{0.23\textwidth}
    \centering
    \raisebox{-.1\height}{\includegraphics[width=\linewidth]{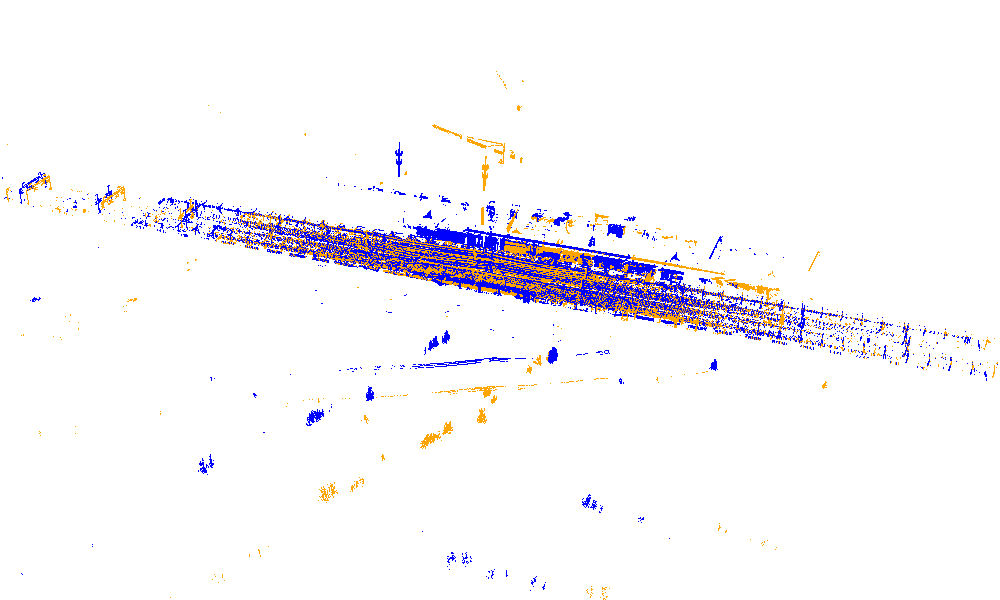}}
    \end{minipage}
 &  \begin{minipage}[b]{0.23\textwidth}
    \centering
    \raisebox{-.1\height}{\includegraphics[width=\linewidth]{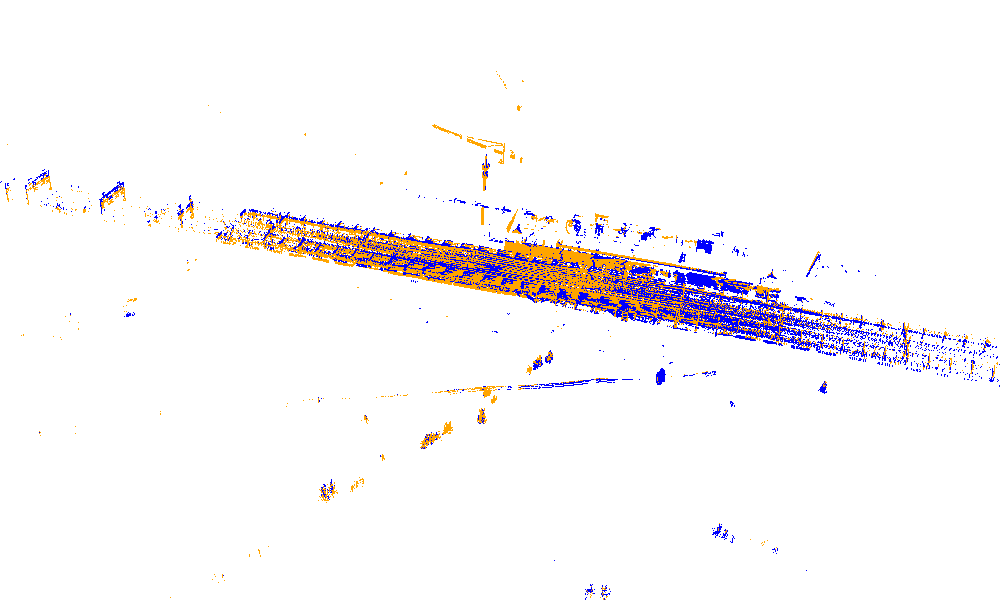}}
    \end{minipage}
 \\ \footnotesize(e) FGR
 &  \footnotesize(f) GTA
 &  \footnotesize(g) GROR
 &  \footnotesize(h) Ours
 \\ $T=\SI{0.026}{s}$
 &  $T=\SI{0.443}{s}$
 &  $T=\SI{0.107}{s}$
 &  $T=\SI{0.043}{s}$
 \\ $RE=179.943\degree$
 &  $RE=0.561\degree$
 &  $RE=0.582\degree$
 &  $RE=0.047\degree$
 \\ $TE=\SI{47.872}{m}$
 &  $TE=\SI{50.759}{m}$
 &  $TE=\SI{50.785}{m}$
 &  $TE=\SI{0.034}{m}$

 \end{tabular}
 \caption{Qualitative and quantitative registration results for scan pair \textit{Railway}$_1$ on the WHU-TLS dataset. (a) Initial, (b) FMP+BnB, (c) BnB, (d) RANSAC, (e) FGR, (f) GTA, (g) GROR, and (h) Ours. The source, target, and aligned point clouds are green, orange, and blue, respectively.}
 \label{WHU_visualization}
\end{figure*}

\begin{figure*}
 \centering
 \begin{tabular}{c c c}
    \multicolumn{3}{c}{
    \begin{minipage}{0.5\textwidth}
    \centering
    \raisebox{-.1\height}{\includegraphics[width=\linewidth]{Pics/legend.png}}
    \end{minipage} } \\
    \begin{minipage}[b]{0.3\textwidth}
    \centering
    \raisebox{-.1\height}{\includegraphics[width=\linewidth]{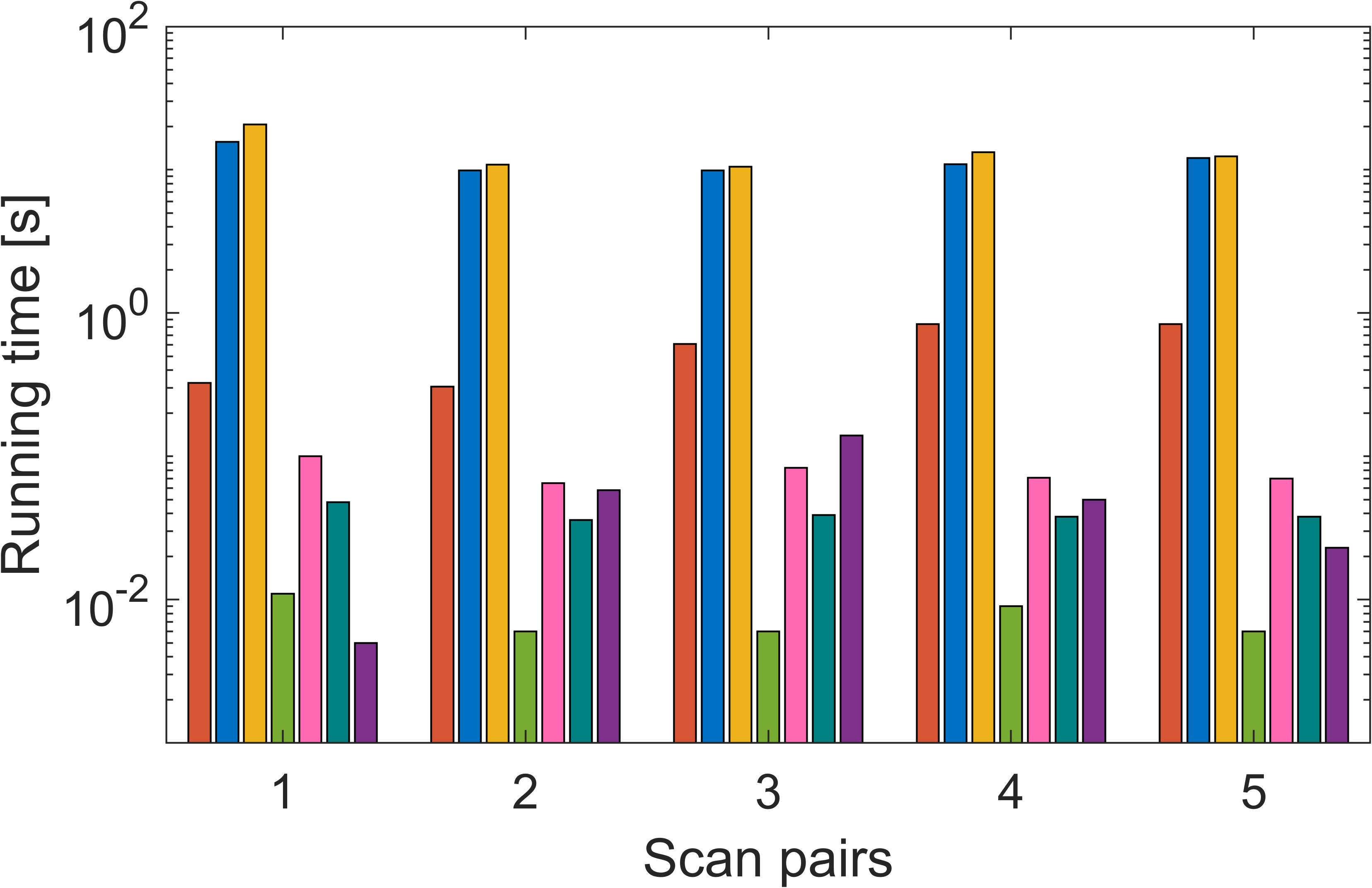}}
    \end{minipage}
  & \begin{minipage}[b]{0.3\textwidth}
    \centering
    \raisebox{-.1\height}{\includegraphics[width=\linewidth]{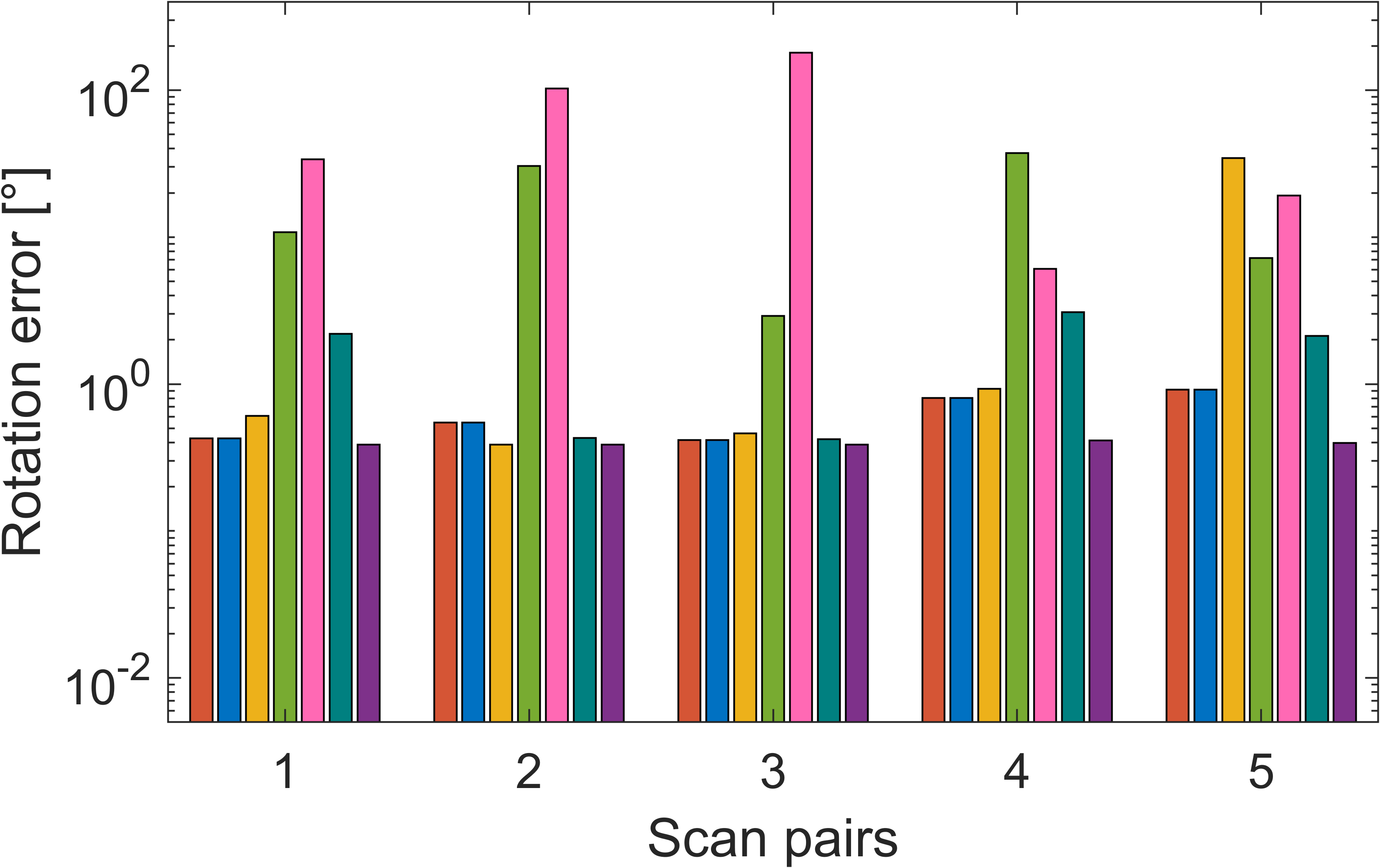}}
    \end{minipage}
  & \begin{minipage}[b]{0.3\textwidth}
    \centering
    \raisebox{-.1\height}{\includegraphics[width=\linewidth]{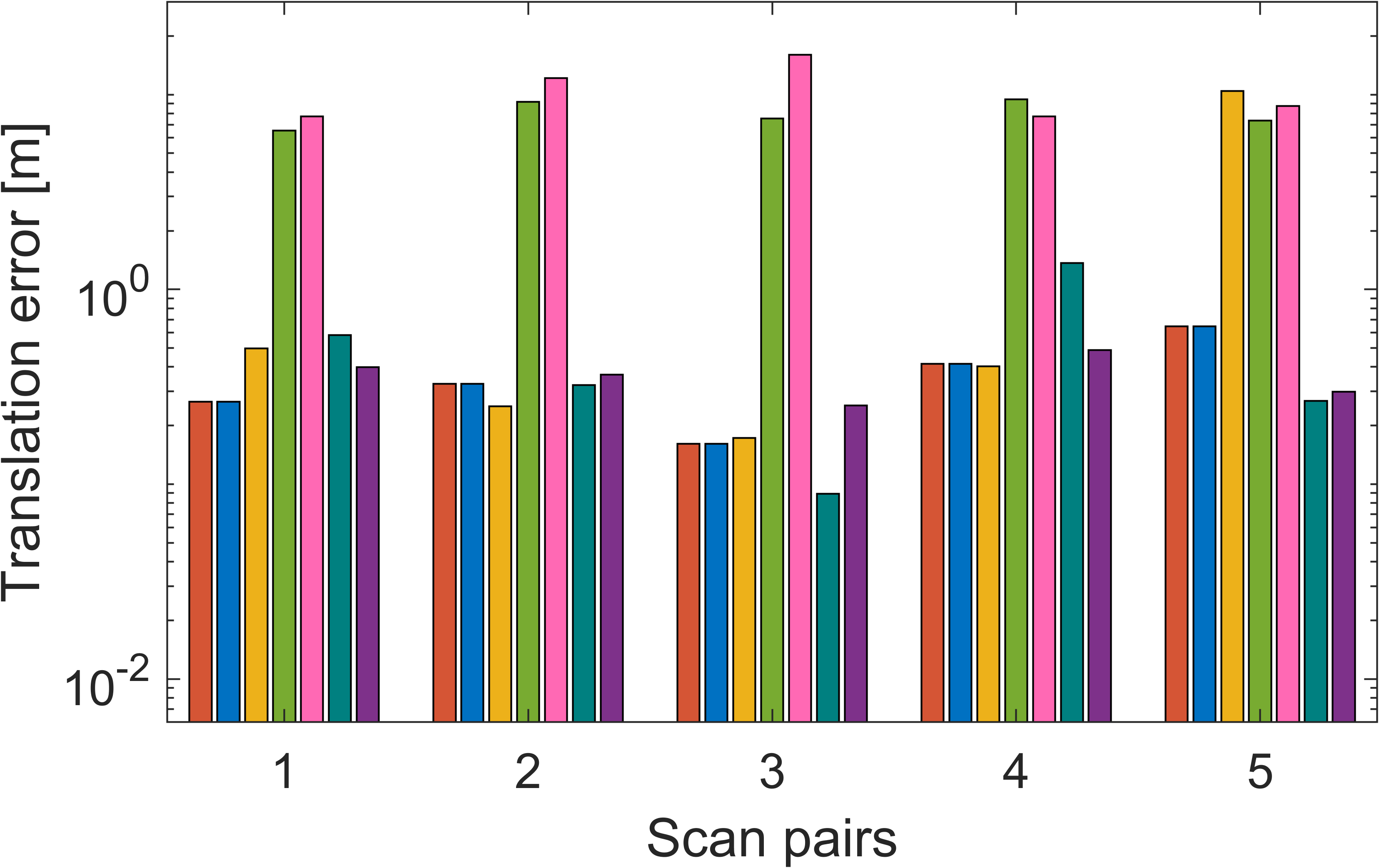}}
    \end{minipage}
 \\ \footnotesize(a) Running time
  & \footnotesize(b) Rotation error
  & \footnotesize(c) Translation error
 \end{tabular}
 \caption{(a) Running time, (b) Rotation error, and (c) Translation error of all registration methods on the A9 dataset.}
 \label{A9_dataset}
\end{figure*}

\begin{figure}\footnotesize
 \centering
 \begin{tabular}{c c}
    \begin{minipage}[b]{0.45\columnwidth}
    \centering
    \raisebox{-.1\height}{\includegraphics[width=\linewidth]{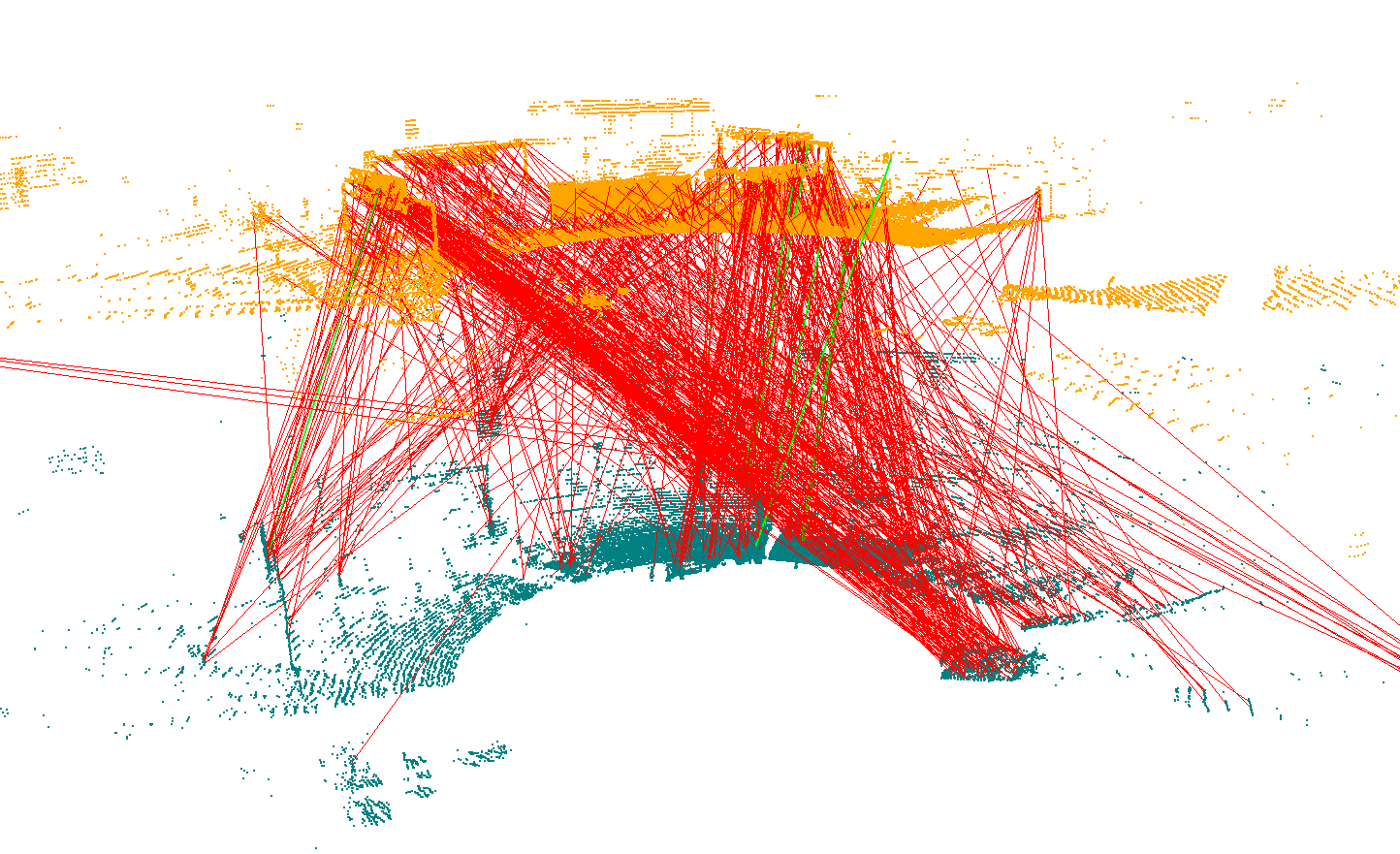}}
    \end{minipage}   
 &  \begin{minipage}[b]{0.45\columnwidth}
    \centering
    \raisebox{-.1\height}{\includegraphics[width=\linewidth]{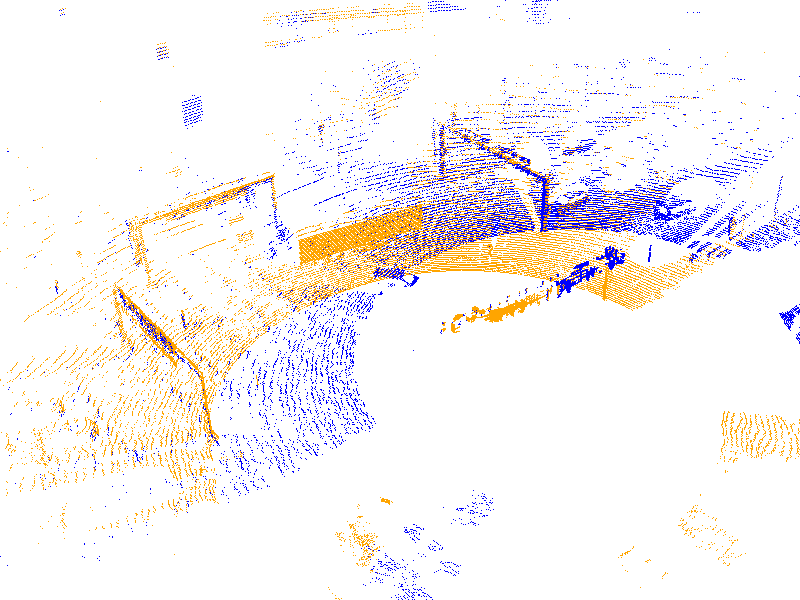}}
    \end{minipage}
 \\ \footnotesize(a) Initial
  &  \footnotesize(b) After registration
 \end{tabular}
 \caption{Qualitative registration results of scan pair $3$ from A9 dataset. (a) Initial, (b) After registration.}
 \label{A9_visualization}
\end{figure}

To evaluate the registration performance on the real-world data, we use three challenging large-scale datasets, including ETH dataset~\cite{theiler2014keypoint}, WHU-TLS dataset~\cite{dong2020registration}, and A9 dataset~\cite{cress2022a9}. Following the preparation strategy in \cite{9373914,10091912}, the first step is to downsample the original point cloud using the voxel grid algorithm\cite{rusu20113d}. The second step is to extract ISS~\cite{zhong2009intrinsic} keypoints and calculate FPFH~\cite{rusu2009fast} descriptors for each keypoint. The third step is to establish the set of putative correspondences $\mathcal{C}$ through the K-nearest neighbor search\cite{lowe2004distinctive}. 

\subsubsection{ETH dataset experiments}
The ETH dataset is a large-scale terrestrial LiDAR dataset, which contains $5$ scenes: \textit{arch}, \textit{courtyard}, \textit{facade}, \textit{office}, and \textit{trees}. In each challenging scene, several scans are captured from different positions thus suffering from low overlap, self-similar structures, and occlusion. As visualized in Fig.~\ref{ETH_visualization}, the outlier rates for various scan pairs in the ETH dataset are notably high, ranging approximately from $86\%$ to $99\%$. We employ adjacent scans as the input of pair-wise registration in each scene. Following~\cite{yan2022new}, the downsample resolution is set to $\SI{0.1}{m}$, as well as the inlier threshold. The detailed information for each scan pair is given in Appendix C, including the number of points, number of keypoints, number of correspondences, and outlier rate. Notably, the gravity direction employed for the proposed method is set to $[0,0,-1]^T$. Fig.~\ref{ETH_dataset} shows the alignment results for a total of $32$ scan pairs and Table~\ref{average} presents the average running times and average errors.

As can be seen from the results, only FMP+BnB, BnB, GROR, and Ours successfully register all scan pairs with relatively low errors. Notably, with the exception of \textit{Office}, our method demonstrates the lowest average time cost in most scenes. Besides, Ours achieves the lowest average rotation and translation errors in \textit{Arch} and \textit{courtyard}, as well as the lowest average translation errors in \textit{facade}, \textit{office}, and \textit{trees}. Although FMP+BnB and BnB perform admirably across all scenes, they are computationally expensive compared to other methods. RANSAC, owing to its non-deterministic nature, occasionally produces inaccurate results, as seen in the case of s4-s5 in \textit{Arch}.  Additionally, RANSAC becomes considerably time-consuming when dealing with outlier rates surpassing $99\%$, as evident in the registration of scan pairs s4-s5 ($\eta=99.77\%$) and s1-s3 ($\eta=99.71\%$) in \textit{Arch}. While FGR is highly efficient, it is susceptible to failure, particularly in scenes with high outlier rates. A total of $8$ scan pairs exhibit significant registration errors in the results obtained by FGR. Similar to RANSAC, GTA is efficient but also occasionally generates erroneous results due to its non-deterministic nature. This is particularly evident in \textit{Arch}, where GTA exhibits high registration errors, including a $31.12\degree$ average rotation error and a $6.747\SI{}{m}$ average translation error. GROR is the second-best approach, which has high accuracy and high efficiency. Overall, benefiting from the proposed transformation decoupling strategy, our method shows superior registration efficiency and accuracy compared to SOTA methods.   

\subsubsection{WHU-TLS dataset experiments}

The WHU-TLS dataset is another large-scale terrestrial LiDAR dataset. To ensure the generality of the registration algorithm, we randomly select two scan pairs each from the \textit{Railway}, \textit{Mountain}, \textit{Park}, \textit{Campus}, \textit{Heritage} and \textit{Tunnel} scenes for the registration experiment. Following~\cite{yan2022new}, the downsample resolution is set to $\SI{0.2}{m}$, as well as the inlier threshold. Detailed information about the selected scan pairs is given in Appendix C. The outlier rates for these chosen pairs range from approximately $89\%$ to $99\%$, with the correspondence numbers spanning from around $1k$ to $20k$. Besides, the gravity direction remains fixed at $[0,0,-1]^T$. The registration results for a total of $12$ scan pairs, including running time, rotation error, and translation error, are shown in Fig.~\ref{WHU_dataset}. In addition, the quantitative average results are presented in Table~\ref{average}.

Observing the average rotation and translation errors, it is evident that only our method successfully registers all scan pairs in the WHU-TLS dataset. Despite their robust performance in the ETH dataset, both FMP+BnB and BnB, as well as GROR, exhibit significant translation errors when registering scan pairs \textit{Railway}$_1$, \textit{Railway}$_2$, and \textit{Tunnel}$_2$. This phenomenon can be attributed to the abundance of self-similar structures that frequently appear in these railway and tunnel scenes. As shown in Fig.~\ref{WHU_visualization}, we provide qualitative and quantitative registration results for scan pair \textit{Railway}$_1$ to further illustrate this phenomenon. All methods, except Ours, converge to similar registration solutions (local minima), as their translation errors are nearly identical. This demonstrates the robustness of our method to the abundant outliers generated by practical self-similar structures. Similar to the findings in the ETH dataset, Ours, along with FMP+BnB, BnB, and GROR, achieve good registration results for scan pairs other than those mentioned above. The performance of FGR and GTA is unstable, as they occasionally yield unsatisfactory results, despite their high efficiency. Besides, RANSAC has high computational costs second only to FMP+BnB and BnB in most scenes. In contrast, the proposed method not only achieves excellent accuracy (with the lowest average rotation and translation errors) but also achieves the second-highest computational efficiency, surpassed only by FGR.

\subsubsection{A9 dataset experiments}

To validate the performance of the proposed method in autonomous driving scenarios, we employ the A9 dataset, which is gathered by multiple roadside sensors mounted on highway gantries. The purpose of this experiment is to align point clouds collected from two distinct LiDARs. We select $5$ scan pairs with a minimum temporal separation of $10$s as input for the experiment. The chosen point cloud pairs contain diverse moving objects (as can be seen in Fig.~\ref{A9_visualization}), including cars, buses, and pedestrians, resulting in the challenging registration task. Following the segmentation of ground points, we leverage non-ground points to establish putative correspondences. Nevertheless, the outlier rates for all scan pairs exceed $99\%$. The downsample resolution also is set to $\SI{0.1}{m}$, as well as the inlier threshold. Appendix C provides detailed information about the selected scan pairs. Since both LiDARs are parallel to the ground, the gravity direction employed remains $[0,0,-1]^T$. The registration results of all methods are shown in Fig.~\ref{A9_dataset} and Table~\ref{average}. Qualitative registration results of the proposed method for scan pair $3$ are provided in Fig.~\ref{A9_visualization}.

The experimental results demonstrate that the proposed method successfully aligns all scan pairs and achieves the best accuracy. FMP+BnB and BnB also deliver accurate results for all scan pairs, but are computationally expensive as well. In this dataset, RANSAC is the most time-consuming due to the extremely high outlier rates. Similar to the former two datasets, the performance of RANSAC, FGR, and GTA is unstable, resulting in high mean errors. GROR is efficient but has a relatively high average rotation error. In contrast, our method has comparable efficiency to GROR and is an order of magnitude faster than FMP+BnB in this dataset. 

\section{Conclusion}\label{Conclusion}
In this paper, we present a novel transformation decoupling strategy based on screw theory for deterministic point cloud registration with gravity prior. By leveraging screw theory to reformulate the registration problem, we successfully decouple the 4-DOF registration problem into three sub-problems with 1-DOF, 2-DOF, and 1-DOF, respectively. This transformation decoupling strategy significantly enhances registration efficiency. Furthermore, we propose an efficient and deterministic three-stage method to tackle these sub-problems, including interval stabbing, BnB, and global voting. In particular, the proposed method can be extended to efficiently solve the SPCR problem. Overall, extensive experiments on both synthetic and real-world data illustrate the SOTA performance of our method in terms of efficiency and robustness.


%




\section*{Acknowledgment}
This research was supported by the Federal Ministry for Digital and Transport, Germany as part of the Providentia++ research project (01MM19008A).

\ifCLASSOPTIONcaptionsoff
  \newpage
\fi



%
\bibliographystyle{IEEEtran}
\bibliography{Literature}

%



\end{document}